\documentclass[11pt]{amsart}
\usepackage{hyperref}
\usepackage{halloweenmath}
\usepackage{amsmath,amsfonts,amssymb,amscd,amsthm,amsbsy,amsxtra,bbm,bm, epsf,calc,comment,appendix, xcolor}
\usepackage{color}
\usepackage{datetime}
\usepackage{latexsym}
\usepackage[english]{babel}
\usepackage{enumerate}
\usepackage{graphicx}
\usepackage{epsfig}
\usepackage{dsfont}
\usepackage{tikz}

\numberwithin{equation}{section}
\newtheorem{theorem}{Theorem}[section]
\newtheorem{lemma}[theorem]{Lemma}
\newtheorem{proposition}[theorem]{Proposition}
\newtheorem{corollary}[theorem]{Corollary}
\newtheorem{example}[theorem]{Example}
\newtheorem{assumption}[theorem]{Assumption}

\theoremstyle{remark}
\newtheorem{remark}[theorem]{Remark}

\theoremstyle{definition}

\usepackage{amsthm}
\usepackage{hyperref}
\hypersetup
{colorlinks = true,
linkcolor = red, 
anchorcolor = red, 
citecolor = blue, 
filecolor = blue,
urlcolor = blue}
\usepackage{hypcap}
\usepackage{cleveref}
\usepackage{lipsum}

\usepackage{enumerate}
\usepackage{enumitem}

\def\R{{ \mathbb{R}}}
\def\X{{ \mathcal{X}}}

\def\Z{\mathcal{Z}}
\def\N{\mathbb{N}}
\def\II{{\rm I\kern-0.5exI}}
\def\III{{\rm I\kern-0.5exI\kern-0.5exI}}

\renewcommand{\c}{\mathbf{c}}
\newcommand{\RR}{\mathbb{R}}

\newcommand{\ZZ}{\mathbb{Z}}

\newcommand{\veps}{\varepsilon}

\definecolor{mygreen}{rgb}{0.1,0.75,0.2}

\DeclareMathOperator*{\argmin}{argmin}

\DeclareSymbolFont{bbold}{U}{bbold}{m}{n}
\DeclareSymbolFontAlphabet{\mathbbold}{bbold}

\newcommand{\vp}{\varphi}

\newcommand{\spt}{\textup{spt}}

\title[MOT and Adversarial multiclass classification]{The multimarginal optimal transport formulation of adversarial multiclass classification}

\author{Nicol\'as {Garc\'ia Trillos}}
\address{Department of Statistics, University of Wisconsin-Madison, 1300 University Avenue, Madison, Wisconsin 53706, USA.}
    \email{garciatrillo@wisc.edu}
\author{Matt Jacobs}
\address{Department of Mathematics, 
Purdue University, 
Mathematical Sciences Bldg, 150 N University St, West Lafayette, IN 47907}
\email{jacob225@purdue.edu}
\author{Jakwang Kim}
\address{Department of Statistics, University of Wisconsin-Madison, 1300 University Avenue, Madison, Wisconsin 53706, USA.}
\email{kim836@wisc.edu}

\date{\today}

\keywords{Adversarial learning, Multiclass classification, Optimal transport, Multimarginal optimal transport, Wasserstein barycenter, Generalized barycenter problem}

\begin{document}

\thanks{{\bf Acknowledgements:}
All authors contributed equally and their names are listed in alphabetic order. This paper is a preprint version of a paper that will appear in JMLR. Part of this work was completed while the authors were visiting the Simons Institute to participate in the program ``Geometric Methods in Optimization and Sampling" during the Fall of 2021. The authors would like to thank the institute for hospitality and support. The authors would also like to thank Camilo A. Garc\'ia Trillos, Ryan Murray, and Meyer Scetbon for enlightening conversations on the topics discussed in this work. NGT was supported by NSF-DMS grant 2005797, and together with JK would also like to thank the IFDS at UW-Madison and NSF through TRIPODS grant 2023239 for their support.}

\begin{abstract}

We study a family of adversarial multiclass classification problems and provide equivalent reformulations in terms of: 1) a family of generalized barycenter problems introduced in the paper and 2) a family of multimarginal optimal transport problems where the number of marginals is equal to the number of classes in the original classification problem. These new theoretical results reveal a rich geometric structure of adversarial learning problems in multiclass classification and extend recent results restricted to the binary classification setting. A direct computational implication of our results is that by solving either the barycenter problem and its dual, or the MOT problem and its dual, we can recover the optimal robust classification rule and the optimal adversarial strategy for the original adversarial problem. Examples with synthetic and real data illustrate our results.

\end{abstract}

\maketitle

\section{Introduction}

In this paper we study, from analytical and geometric perspectives, the problem of adversarial learning in multiclass classification. By multiclass classification we mean the task of assigning classes $\hat i$ in a set of $K$ available classes to all inputs $\hat x$ in some feature space $\X$ based on the observation of training pairs $z=(x,i)$. The adversarial component of the problem refers to the desire of producing classification rules that are \textit{robust} to data perturbations. Mathematically speaking, this means studying optimization problems of the form:
\begin{equation}
  \label{Robust problem:Intro}
  \inf_{f \in \mathcal{F}} \sup_{\widetilde{\mu} \in \mathcal{P}(\Z) }  \left\{ R(f, \widetilde{\mu}) - C(\mu, \widetilde{\mu}) \right\}.
\end{equation}
Here, $\mathcal{F}$ denotes the set of \textit{all} probabilistic multiclass classifiers —see section \ref{sec:Preliminaries}; $\mu$ denotes the observed data distribution, which in general is some probability measure on the  space $\Z = \X \times \{1, \dots, K \}$, but which for simplicity can be thought of as an empirical measure associated to a finite training data set; $C$ represents a notion of ``distance'' between data distributions; $R(f, \widetilde{\mu})$ is a risk functional relative to a data distribution $\widetilde{\mu}$ (thought of as a perturbation of $\mu$) and a choice of loss function, which in this paper will be restricted to be the $0$-$1$ loss. Problem \eqref{Robust problem:Intro} can be interpreted as a game between a \textit{learner} and an \textit{adversary}: the learner's goal is to find a classifier with small risk, while the adversary tries to find a data perturbation $\widetilde{\mu}$ that makes the risk for the learner large. The adversary has an implicit budget to perform their actions: the adversary can not choose a $\widetilde{\mu}$ that is too far away (relative to $C$) from the original data distribution $\mu$.

For a large family of functionals $C$ in \eqref{Robust problem:Intro} we show that the adversarial problem \eqref{Robust problem:Intro} is equivalent to a \textit{multimarginal optimal transport problem (MOT)} of the form:
\begin{equation}\label{Robust problem:Alternative_form}
   \inf_{\pi \in \Pi_K(\mu)} \int \c (z_1, \dots, z_K) d\pi(z_1, \dots, z_K),
\end{equation}
where $\c$ is a cost function discussed in detail throughout the paper and $\Pi_{K}(\mu)$ is a space of couplings specified in section \ref{sec:MOT}. As part of this equivalence, we explicitly describe how to construct solutions to the original problem \eqref{Robust problem:Intro} from solutions to the problem \eqref{Robust problem:Alternative_form} and its dual, offering in this way new computational strategies for solving problem \eqref{Robust problem:Intro}. Since most algorithms for OT are primal-dual (i.e., they simultaneously search for solutions to both the primal OT problem and its dual), it is actually possible to construct a saddle solution $(f^*, \mu^*)$ for \eqref{Robust problem:Intro} by running one such OT algorithm. The equivalence between \eqref{Robust problem:Intro} and \eqref{Robust problem:Alternative_form} that we study here is an extension to the multi-class case of a series of recent results connecting adversarial learning in binary classification with optimal transport: \cite{Bhagoji2019LowerBO, Nakkiran2019AdversarialRM, Pydi2021AdversarialRV, pydi2021the, Trillos2020AdversarialCN}.

In order to establish the equivalence between \eqref{Robust problem:Intro} and \eqref{Robust problem:Alternative_form}, we develop another interesting equivalent reformulation of \eqref{Robust problem:Intro} that reveals a rich geometric structure of the original adversarial problem. This reformulation takes the form of a \textit{generalized barycenter problem}
\begin{equation*}
    \inf_{\lambda, \widetilde{\mu}_1, \ldots, \widetilde{\mu}_K} \lambda(\mathcal{X})+\sum_{i \in [K]} C(\mu_i, \widetilde{\mu}_i) \quad \textup{s.t.}\; \lambda \geq \widetilde{\mu}_i, i\in [K],
\end{equation*}
which is a novel variant of the Wasserstein barycenter problems 
introduced in \cite{MR2801182,Carlier2010MatchingFT}.
In the classical Wasserstein barycenter problem, given $K$ \textit{probability} measures $\varrho_1, \ldots, \varrho_K$ defined over a Polish space $\mathcal{X}$ and a cost $c:\mathcal{X} \times \mathcal{X}\to [0,\infty]$, one tries to find a probability measure $\varrho$ such that the summed cost of transporting each of the $\varrho_i$ onto $\varrho$ is as small as possible. In our generalized problem, we try to find a nonnegative measure $\lambda$ (no longer necessarily a probability measure)  such that the total mass of $\lambda$ plus the summed cost of transporting each $\mu_i$ (not necessarily having the same total mass) onto \emph{some part} of $\lambda$ is as small as possible.  Here transporting a $\mu_i$ onto some part of $\lambda$ means we want to find a measure $\widetilde{\mu}_i\leq \lambda$ and transport $\mu_i$ to $\widetilde{\mu}_i$ in the classical optimal transport sense. This problem will be studied in detail in section \ref{sec:GenBar}.  We prove that these generalized barycenter problems can be written as appropriate MOT problems, a result that is analogous to ones in \cite{MR2801182,Carlier2010MatchingFT} for standard Wasserstein barycenter problems.

From the equivalence with the generalized barycenter problem we will be able to deduce that optimal adversarial attacks can always be obtained as suitable barycenters of $K$ or less points in the original training data set. Also, from this reformulation we will be able to recognize the structure of the cost function $\c$ in \eqref{Robust problem:Alternative_form}: for the adversary to obtain their optimal strategy, they can actually \textit{localize} their problem to sets of $K$ or fewer data points —see section \ref{sec:MOT}.     Other theoretical, methodological, and computational implications of these reformulations will be pursued in future work. See section \ref{sec:Conclusions} for a discussion on future directions for research.

In contrast to many of the existing applications of OT to ML, it is worth emphasizing that in this work OT arises naturally in connection with a learning problem, rather than as a particular way to address a certain machine learning task. For the growing literature in multimarginal optimal transportation this paper offers new examples of cost functions worthy of study. MOT is a rich topic that has been developed over the years from theoretical and applied perspectives. After the first mathematical analysis of general MOT problems in \cite{Gangbo1998OptimalMF}, there have been numerous subsequent papers establishing geometric and analytic results (e.g., \cite{Kim2013MultimarginalOT, MR3423275, MR3482267, MR3716857}) for MOT problems. MOT problems have also been used extensively in applications. For example, they appear in the so-called density functional theory in physics \cite{seidl2007strictly, buttazzo2012optimal, MR3020313, mendl2013kantorovich, MR3314838}, and in economics \cite{MR2110613, MR2564439,  Carlier2010MatchingFT}. In the machine learning community, researchers have recently explored many interesting applications, including generative adversarial networks(GANs) \cite{choi2018stargan, cao2019multi} and Wasserstein Barycenters \cite{MR2801182, cuturi2014fast, benamou2015iterative, MR3423268, srivastava2018scalable, delon2020wasserstein}, where MOTs are used. Recent works like \cite{di2020optimal, haasler2021multimarginal} develop a connection between the Schr\"{o}dinger bridge problem and MOT. MOT problems have been extended to the unbalanced setting —see  \cite{beier2021unbalanced}.

\subsection{Outline of paper}
The rest of the paper is organized as follows. In section \ref{sec:Preliminaries}, we introduce most mathematical objects and notation used throughout the rest of the paper. We also introduce the generalized Wasserstein barycenter problem, which can be interpreted as dual of the original adversarial problem \eqref{Robust problem:Intro}, and define in detail the MOT problem \eqref{Robust problem:Alternative_form}. In section \ref{sec:GenBar}, we study the aforementioned generalized Wasserstein barycenter problem and prove its equivalence with 1) a stratified barycenter problem and 2) a first version of an MOT problem. In section \ref{sec:ProofMainTheorem} we discuss the equivalence between \eqref{Robust problem:Intro} and \eqref{Robust problem:Alternative_form} through the duality results in earlier sections. In section \ref{sec:Examples}, we present a collection of examples and numerical experiments whose goal is to illustrate the theory developed throughout the paper and provide further insights into the geometric structure of adversarial learning in multiclass classification. Finally, we wrap-up the paper in section \ref{sec:Conclusions} by presenting some conclusions and discussing some future directions for research.

\section{Preliminaries}
\label{sec:Preliminaries}

Throughout the paper $(\mathcal{X}, d)$ will be a Polish space, $[K]:= \{1, \dots, K\}$ with $K \geq 2$, and $\mathcal{Z}:= \mathcal{X} \times [K]$. We regard $\mathcal{X}$ as the feature space of our learning problem and $[K]$ as the set of classes or labels.

Let $\mu$ be a finite positive Borel measure (not necessarily a probability measure) over $\Z$. Given $i\in [K]$, we use $\mu_i$ to represent the positive measure over $\X$ defined as
\begin{equation}
\mu_i(A) := \mu\left( A \times \{ i\} \right),
\label{eqn:decomposition}
\end{equation}
for all measurable subsets $A$ of $\X$. In the sequel, we use $\mu$ to represent a fixed data distribution, which we regard as an observed data distribution or training data distribution, and use $\widetilde{\mu}$ to represent any other arbitrary finite positive measure over $\Z$. Through this paper we use $\mathcal{M}_+(\mathcal{X})$ and $\mathcal{M}_+(\mathcal{Z})$ to denote the set of finite positive (Borel) measures over $\mathcal{X}$ and $\mathcal{Z}$, respectively.

Through this paper, the cost function in \eqref{Robust problem:Intro} will take the form:
 \[ 
C(\mu, \widetilde{\mu} ) := \min_{\pi \in \Gamma (\mu, \widetilde{\mu})} \int c_{\Z}(z, \tilde z) d\pi(z, \tilde z),  
\]
for some cost function $c_{\Z}: \Z \times \Z \rightarrow [0,\infty]$. Here and in the remainder of the paper the set $\Gamma(\cdot, \cdot)$ represents the set of couplings between two positive measures over the same space; for example, $\Gamma (\mu, \widetilde{\mu})$ denotes the set of positive measures over $\Z \times \Z$ with first marginal equal to $\mu$ and second marginal equal to $\widetilde{\mu}$. 

\begin{assumption}
\label{assump:CostStructure}
The function $c_\Z$ will be assumed to have the following structure:
\[ c_\Z (z, \tilde z) = \begin{cases} c(x, \tilde x) & \quad \text{ if } i = \tilde i \\ \infty  & \quad \text{ if } i \not = \tilde i,    \end{cases}   \]
for some lower semi-continuous function $c : \X \times \X \rightarrow [0, \infty]$. 

The function $c$ will be further assumed to satisfy $c(x,x)=0$ for all $x \in \mathcal{X}$ and the following compactness and coercivity properties:
\begin{itemize}
    \item if $\{x_n\}_{n \in \N}$ is a bounded sequence in $(\X,d)$ and $\{ x_n' \}_{n \in \N}$ is a sequence in $\X$ satisfying $\sup_{n \in \N} c(x_n', x_n) <\infty$, then $\{(x_n', x_n) \}_{n \in \N}$ is precompact in $\X \times \X$ (with the induced product metric). 
\end{itemize}
\end{assumption}

The structure of $c_\Z$ described in Assumption \ref{assump:CostStructure} is standard in the literature of adversarial learning and can be motivated by the fact that in many applications of interest it is natural to think that the ``true" label associated to a perturbation $\tilde x$ of a data point $x$ coincides with the true label of the original $x$.  Naturally, this is simply a modeling choice, and other cost structures of interest can be studied elsewhere. The lower semicontinuity and compactness assumptions on $c$ are technical requirements that we use in the remainder. All cost functions of interest satisfy these properties —see the examples below.

If we decompose $\mu$ and $\widetilde{\mu}$ into measures $\mu_i, \widetilde{\mu}_i$ as in \eqref{eqn:decomposition}, it is possible to write $C(\mu, \widetilde{\mu})$ as
\[ C(\mu, \widetilde{\mu})= \sum_{i \in [K]} C(\mu_i, \widetilde{\mu}_i), \]
abusing notation slightly and interpreting $C(\mu_i, \widetilde{\mu}_i)$ as

\begin{equation}
C(\mu_i, \widetilde{\mu}_i) =  \min_{\pi \in \Gamma (\mu_i, \widetilde{\mu}_i)} \int c (x, \tilde x) d\pi(x, \tilde x).
\label{eqn:OTProblemMu_i}
\end{equation}

\begin{remark}
Let us emphasize that we define $C( \mu_i, \widetilde{\mu}_i)=+\infty$ whenever the set of couplings $\Gamma(\widetilde{\mu}_i, \mu_i)$ is empty, which is the case if $\mu_i$ and $\widetilde{\mu}_i$ have different total mass.
\end{remark}

We introduce two notions that will be used throughout our analysis. Given a lower semi-continuous function $f:\mathcal{X}\to\RR$, we define
\begin{equation}\label{eq:c_transform}
f^c(x):=\inf_{x'\in \mathcal{X}} \{ f(x') + c(x',x) \},
\end{equation}
and given an upper semi-continuous function $g:\mathcal{X}\to\RR$, we define
\begin{equation}\label{eq:conjugate_c_transform}
g^{\bar{c}}(x'):=\sup_{x \in \mathcal{X}} \{ g(x) - c(x',x) \}.
\end{equation}

\begin{example} \label{ex : infty_OT cost}
Let $\veps > 0 $ and let $c(x, \tilde x)$ be given by
\[ 
c(x, \tilde x) = c_\veps(x, \tilde x) = 
\begin{cases} 
0 &\text{ if } d(x, \tilde x) \leq \veps \\ 
\infty & \text{ if } d(x, \tilde x) >\veps    
\end{cases}. 
\]
The parameter $\veps$ can be interpreted as the \textit{adversarial budget}: the larger the value of $\veps$, the wider the space of actions available to the adversary. The cost $c$ satisfies \textbf{Assumption} \ref{assump:CostStructure} provided that closed balls with finite radius in $(\X, d)$ are compact. 

Notice that in this case, the $c$-transform $f^c$ of a given function $f$ takes the form:
\[ 
f^c(x) = \inf_{x'\: : \: d(x,x')\leq \veps} f(x'). 
\]

In this setting, the adversarial problem \eqref{Robust problem:Intro} can be written as
\[ 
\inf_{f \in \mathcal{F}} \sup_{\widetilde{\mu} \: : \: W_\infty(\mu, \widetilde{\mu}) \leq \veps} R(f , \widetilde{\mu}).   
\]
where $W_\infty(\mu, \widetilde{\mu})$ is the $\infty$-OT distance between $\mu$ and $\widetilde{\mu}$ relative to the distance function:
\[ 
\delta(z, \tilde z) := \begin{cases} d(x, \tilde x) & \quad  \text{ if } y=\tilde y, \\ \infty  & \quad  \text{otherwise}. \end{cases} 
\]

\label{example:epsilonAT}
\end{example}

\begin{remark}
In the literature of machine learning there are many different versions of adversarial problems for supervised tasks, but two versions are particularly popular: data-perturbing adversarial learning (e.g., see\cite{Pydi2021AdversarialRV}) and distributional perturbing adversarial learning (e.g., see \cite{MR3959085, MR4015639}). For a rigorous analysis, distributional perturbing adversarial learning is more adequate since data-perturbing adversarial learning lacks measurability in some cases. Furthermore, putting some technical details aside, one can prove that distributional perturbing adversarial learning encompasses data-perturbing adversarial learning: see \cite{Pydi2021AdversarialRV}.

The main focus in this paper is based on the distributional setting, where given a data distribution $\mu$, an adversary can select a new distribution $\widetilde{\mu}$ in a neighborhood of the original distribution $\mu$ determined by $C$. \cite{pydi2021the} summarizes other adversarial models and discusses connections between them.
\end{remark}

\begin{example} Let $p > 0 $ and let $c(x, \tilde x)$ be given by
\[ c(x, x') = c^p(x,x') := \frac{1}{\tau} (d(x,x'))^p , \]
for some constant $\tau>0$. For this choice of cost $c$, it is possible to show, through a formal argument whose details we omit, that problem \eqref{Robust problem:Intro} can be written as
\[ \inf_{f \in \mathcal{F}} \sup_{\widetilde{\mu} \: : \: 
W_p(\mu, \widetilde{\mu}) \leq \veps} R(f , \widetilde{\mu}),   \]
for some $\veps>0$ and for $W_p(\mu, \widetilde{\mu})$ the $p$-OT distance between $\mu$ and $\widetilde{\mu}$ relative to the distance function $\delta$ from \textbf{Example} \ref{example:epsilonAT}. The relation between $\tau$ and $\veps$ is not explicit, but, qualitatively, small values of $\tau$ should correspond to small values of $\veps$.

Notice that in this case the $c$-transform $f^c$ of a given function $f$ takes the form:
\[ f^c(x) = \inf_{x' \in \X} f(x') + \frac{1}{\tau}d(x, x')^p.  \]
If $f$ is bounded below by a constant, it follows that $f^c$ is always continuous (in the $d$ metric) regardless of the continuity properties of the original $f$. 
\end{example}

The solution space $\mathcal{F}$ in \eqref{Robust problem:Intro} is the full set of weak partitions, or probabilistic classifiers, defined by
\begin{equation*}
\mathcal{F}:= \left\{ f: \mathcal{X} \to \Delta_{[K]} : f \text{ Borel measurable } \right\},
\end{equation*}
where
\begin{equation*}
\Delta_{[K]} := \left\{ (u_i)_{i \in [K]} : 0 \leq u_i \leq 1, \sum_{i \in [K]} u_i =1 \right\},
\end{equation*}
i.e., $\Delta_K$ is the set of probability distributions over $[K]$. In other words, at each $x \in \mathcal{X}$, $f(x)$ is a probability distribution over $[K]$ representing the likelihood, according to the specific classifier $f$ chosen by the learner, that a given $x$ belongs to any of the available classes. Probabilistic classifiers are widely used in applications as they allow for the use of standard optimization techniques when training models. We want to highlight that the $f$s in $\mathcal{F}$ are only assumed to be Borel measurable. This means that the learner in problem \eqref{Robust problem:Intro} can be considered to be \textit{agnostic} to any specific model for the classifiers and in that sense \eqref{Robust problem:Intro} can be interpreted as a robust generalization of the notion of Bayes classifier studied in statistical learning.  

For a given $u \in  \Delta_{[K]} $ and a given $i \in [K]$, we define the loss:
\begin{equation*}
\ell(u, i) := 1 - u_i.
\end{equation*}
Notice that $\ell(e_j , i) $ is equal to $1$ if $i \not = j$ and $0$ if $i=j$. $\ell$ thus extends the $0$-$1$ loss to weak classifiers, and from now on we will refer to it simply as the $0$-$1$ loss. For a given pair $(f, \widetilde{\mu})$ we define the risk:
\begin{equation*}\label{adversarial_risk}
    R(f, \widetilde{\mu}):= \mathbb{E}_{(\widetilde{X}, \widetilde{Y}) \sim \widetilde{\mu}}[\ell(f(\widetilde{X}), \widetilde{Y})] =\sum_{i \in [K]} \int_{\mathcal{X}} (1 - f_i(\widetilde{x}))  d\widetilde{\mu}_i(\widetilde{x}) ,
\end{equation*}
which can be regarded as a bilinear functional $R(\cdot, \cdot): \mathcal{F} \times \mathcal{M}_+(\Z) \longrightarrow \mathbb{R}_+$.  For convenience, we introduce the so-called \textit{classification power} for a pair $(f, \widetilde{\mu}) \in  \mathcal{F} \times \mathcal{M}_+(\Z)$, which is defined by
\begin{equation}\label{classifying_power}
B(f, \widetilde{\mu}) := \sum_{i \in [K]}  \int_{\mathcal{X}} f_i(\widetilde{x})  d\widetilde{\mu}_i(\widetilde{x}).
\end{equation}
With these new definitions, problem \eqref{Robust problem:Intro} is immediately seen to be equivalent to
\begin{equation}
 \sup_{f \in \mathcal{F}} \inf_{\widetilde{\mu} \in \mathcal{M}_+(\Z) }  \left\{ B(f, \widetilde{\mu}) + C(\mu, \widetilde{\mu}) \right\}.
 \label{def:RobustClassifPower}
\end{equation}
Moreover, if we denote by $\widetilde{B}^*_\mu$ the optimal value of \eqref{def:RobustClassifPower}, and by $R^*_\mu$ the optimal value of \eqref{Robust problem:Intro}, we have the identity:
\[   
R^*_\mu = \mu(\Z)   -   \widetilde{B}^*_\mu.  
\] 
We write $\mu(\Z)$ explicitly, although for the most part $\mu(\Z)$ can be thought of as being equal to one.

The dual of \eqref{def:RobustClassifPower} is obtained by swapping the $\sup$ and the $\inf$:
\begin{equation}
  \inf_{\widetilde{\mu} \in \mathcal{M}_+(\Z) }   \sup_{f \in \mathcal{F}}  \left\{ B(f, \widetilde{\mu}) + C(\mu, \widetilde{\mu}) \right\}. 
  \label{eqn:DualAdversarial}
  \end{equation}
Notice that the value of \eqref{eqn:DualAdversarial} is always greater than or equal to the value of \eqref{def:RobustClassifPower}. Instead of attempting to invoke an abstract minimax theorem implying the equality of these two quantities at this stage, we prefer to defer this discussion to later sections where in fact we will prove that, under \textbf{Assumption} \ref{assump:CostStructure}, there is no duality gap in this problem. In what follows we focus on the dual problem \eqref{eqn:DualAdversarial} and only return to problem \eqref{def:RobustClassifPower}, which is equivalent to the original adversarial problem \eqref{Robust problem:Intro}, in section \ref{sec:MOTProofs}. Notice, however, that the statement of \textbf{Theorem} \ref{thm:Main} mentions the adversarial problem explicitly. 

For fixed $\widetilde{\mu}$, notice that 
\begin{align*}
 \sup_{f \in \mathcal{F}}  \left\{ B(f, \widetilde{\mu}) + C(\mu, \widetilde{\mu}) \right\}  &=   \sup_{f \in \mathcal{F}}  \left\{   \sum_{i \in [K]}  \int_{\mathcal{X}} f_i(\widetilde{x})  d\widetilde{\mu}_i(\widetilde{x}) + C(\mu, \widetilde{\mu}) \right\}  
 \\&=  \sup_{f \in \mathcal{F}}  \left\{   \sum_{i \in [K]}  \int_{\mathcal{X}} f_i(\widetilde{x})  d\widetilde{\mu}_i(\widetilde{x})  \right\}   + C(\mu, \widetilde{\mu}).
\end{align*}
Introducing a new variable $\lambda$, a positive measure over $\X$, we can rewrite  the latter $\sup$ as:
\begin{equation*}
    \inf_{\lambda} \lambda(\mathcal{X}) \quad \textup{s.t.}\; \int_X g(x) d(\lambda-\widetilde{\mu}_i)(x)\geq 0 \;\; \textup{for all} \; g\geq0, i \in [K];
\end{equation*}     
the constraint in $\lambda$ can be simply written as $\lambda \geq\widetilde{\mu}_i$ for all $i \in [K]$. Combining the above with the structure of the cost $C(\mu, \widetilde{\mu})$, we conclude that problem \eqref{eqn:DualAdversarial} is equivalent to the generalized barycenter problem mentioned in the introduction:
\begin{equation}
\label{eq:generalized_barycenter}
    B_\mu^* := \inf_{\lambda, \widetilde{\mu}_1, \ldots, \widetilde{\mu}_K} \left\{ \lambda(\mathcal{X})+\sum_{i \in [K]} C(\mu_i, \widetilde{\mu}_i) : \text{ $\lambda \geq\widetilde{\mu}_i$ for all $i \in [K]$}\right\},
\end{equation}
where we use the notation $B_{\mu}^*$ for future reference; see Figure \ref{plot: generalized_barycenter} for a pictorial explanation.

\begin{figure}[h]
	\includegraphics[scale=0.35]{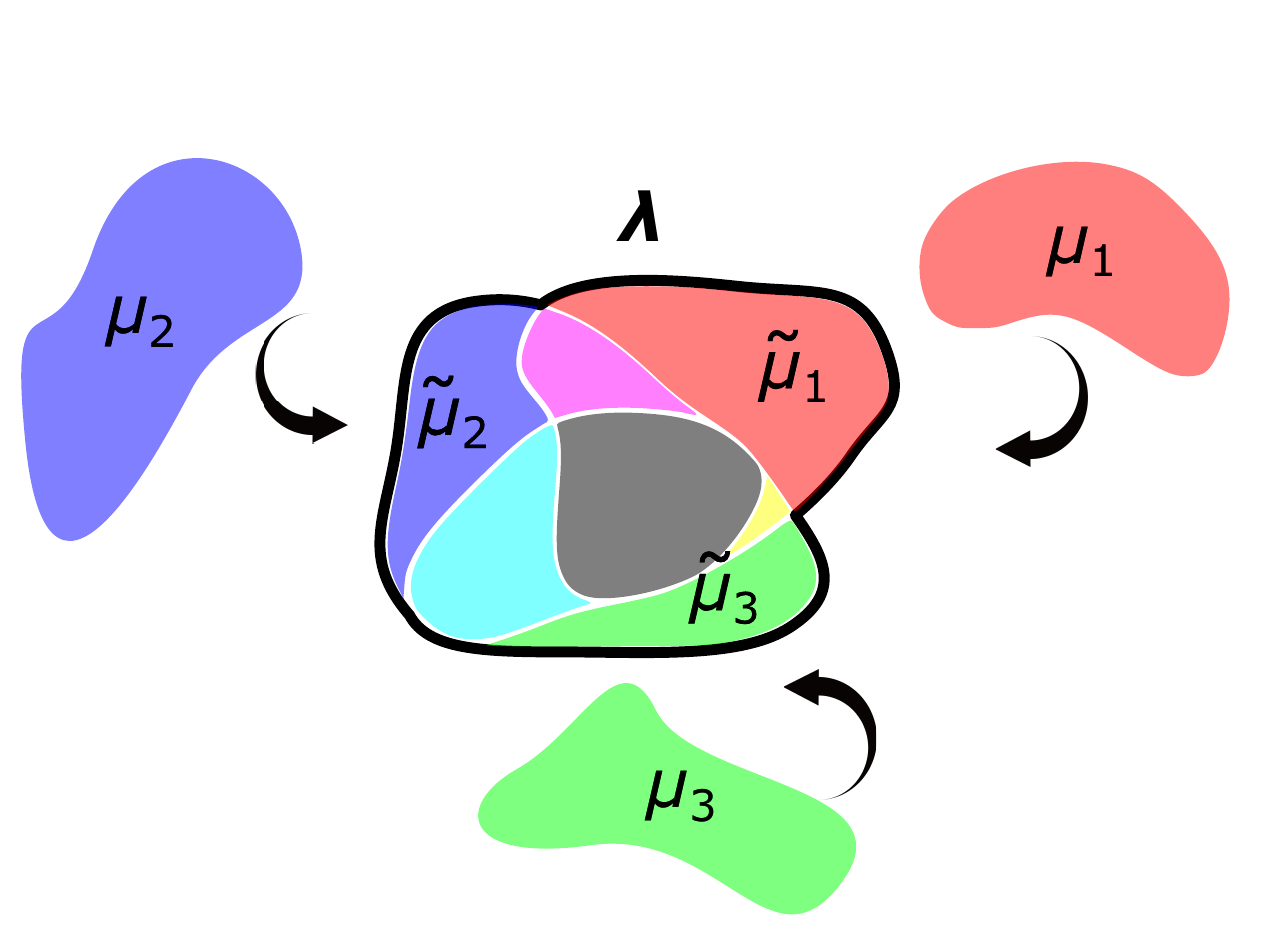}
	\centering
	\caption{Picture for \eqref{eq:generalized_barycenter}. $\mu_i$'s are first moved to $\widetilde{\mu}_i$'s and $\lambda$ is chosen to cover all $\widetilde{\mu}_i$'s: it is the smallest positive measure which is larger than all $\widetilde{\mu}_i$'s.}
	\label{plot: generalized_barycenter}
\end{figure}

\begin{remark}
\label{rem:Homogeneity}
It is straightforward to see from \eqref{eqn:DualAdversarial} that $B_\mu^*$ is $1$-homogeneous in $\mu$. That is, if $a> 0$, then $B^*_{a\mu} = a B^*_{\mu}$.
\end{remark}

\subsection{The MOT problem}
\label{sec:MOT}

\subsubsection{General MOT problems}
\label{sec:GeneralMOT}
Before providing the details of our MOT problem \eqref{Robust problem:Alternative_form}, it is worth introducing the generic MOT problem first. Let $\mathcal{S}_1, \dots, \mathcal{S}_K$ be fixed spaces and let $\c : \mathcal{S}_1 \times \dots \times \mathcal{S}_K \to \mathbb{R} \cup \{+\infty, - \infty \}$ be a cost function. For each $1 \leq k \leq K$, let $\nu_k \in \mathcal{P}(\mathcal{S}_k)$ be a Borel probability measure. The MOT problem associated to the cost function $\c$ and the measures $\nu_1, \dots, \nu_K$ is the following (possibly infinite dimensional) linear optimization problem with $K$-marginal constraints:
\begin{equation*}
	\inf_{\pi \in \Pi(\nu_1, \dots, \nu_K)} \int_{\mathcal{S}_1 \times \dots  \times \mathcal{S}_K} \c(\xi_1, \dots, \xi_K) d \pi(\xi_1, \dots, \xi_K),
\end{equation*}
where
\begin{equation*}
	\Pi(\nu_1, \dots, \nu_K) := \{ \pi \in \mathcal{P}(\mathcal{S}_1 \times \dots \times \mathcal{S}_K), \text{ s.t., } \text{ for every } i, \text{ $i$-th marginal of } \pi = \nu_i \}.
\end{equation*}
MOTs are generalizations of the standard (two marginals) optimal transport (OT) problems and their duals take the form:
\begin{equation}\label{eqn:DualMOTGeneral}
     \sup_{\phi \in \Phi} \left\{ \sum_{j=1}^K \int_{\mathcal{S}_j} \phi_j(\xi_j) d\nu_j(\xi_j) \right\},
\end{equation}
where $\Phi$ is the set of all $\phi= (\phi_1, \dots, \phi_K) \in \prod_{j=1}^K L^1 (\nu_j)$ such that 
\[ \sum_{j=1}^K \phi_j(\xi_j) \leq \c(\xi_1, \dots, \xi_K),\quad \forall (\xi_1, \dots, \xi_K) \in \mathcal{S}_1 \times \dots \times \mathcal{S}_K.\]

One of the most popular examples of MOT problems is connected to the Wasserstein Barycenter problem over $\mathcal{P}(\mathcal{X})$; see \cite{MR2110613, MR2564439, MR2801182}. Let $c : \mathcal{X} \times \mathcal{X} \to \mathbb{R} \cup \{+\infty, - \infty \}$ be a fixed pairwise cost function. In the Wasserstein barycenter problem the goal is to find a solution $\nu^*$ to the problem
\begin{equation*}
    \inf_{\nu'}  \sum_{i \in [K]} C(\nu', \nu_i) \quad \text{where } C(\nu, \nu_i) : = \inf_{\pi \in \Pi(\nu, \nu_i)} \int_{\mathcal{X} \times \mathcal{X}} c(x',x) d \pi(x',x).
\end{equation*}
Such $\nu^*$ can be interpreted as an ``average" or barycenter of the input measures $\nu_1, \dots,\nu_K$ relative to the cost $C$. It can then be showed that the above Wasserstein barycenter problem is equivalent to solving the following MOT problem
\begin{equation*}
    \inf_{\pi \in \Pi(\nu_1, \dots, \nu_K)} \int_{\mathcal{X}^K} \c(x_1, \dots, x_K) d\pi(x_1, \dots, x_K),
\end{equation*}
where
\begin{equation*}
    \c(x_1, \dots, x_K) :=  \inf_{x' \in \mathcal{X}} \sum_{i \in [K]} c(x',x_i) .
\end{equation*}
Indeed, let $\pi^*$ be a minimizer of the above MOT problem. Defining $\nu^* = T_{\#}\pi$, where $T(x_1, \dots, x_K):= \argmin_{x'} \sum_{i \in [K]} c(x',x_i)$, i.e., defining $\nu^*$ as the pushforward measure of $\pi^*$ with respect to the barycenter mapping $T$, one can recover a solution to the original barycenter problem. Conversely, one can use a Wasserstein barycenter $\nu^*$ and couplings $\pi_i$ realizing the costs $C(\nu^*, \nu_i)$ to build a solution to the MOT problem; see more details in \cite{MR2801182}.

\subsubsection{From adversarial robustness to MOT}
Now we are ready to state problem \eqref{Robust problem:Alternative_form} precisely. For this, we will need to modify the space $\mathcal{Z}$ and in particular add an extra element to it that will be denoted by the symbol $\mathghost$. The marginals of the couplings in the desired MOT problem will be probability measures over the set $\mathcal{Z}_*:=\mathcal{Z} \cup \{ \mathghost \}$. More precisely, letting $P_{i}$ represent the projection onto the $i$-th coordinate, we consider the set: 
\begin{equation}\label{eqn:CouplingsMOT}
    \Pi_K(\mu):= \left\{ \pi \in \mathcal{P}(  \mathcal{Z}_*^K ) \: : \: P_{i \sharp} \pi = \frac{1}{2 \mu(\mathcal{Z})} \mu( \cdot \cap \mathcal{Z}) + \frac{1}{2} \delta_{\mathghost} \text{ for all } i \in [K]  \right\}.
\end{equation}
Notice that in this set all $K$ marginals are the same. Dividing by the factor $\frac{1}{2 \mu(\mathcal{Z})}$, the set $\Pi_K(\mu)$ is made to be consistent with the literature on multimarginal optimal transport, where sets of couplings are typically assumed to be probability measures.

Let us now discuss the cost function for the desired MOT problem. For a given tuple $(z_1, \dots, z_K)$ in $\mathcal{Z}_*^K$, often denoted by $\vec{z}$ in the sequel for convenience, we define
\begin{equation}\label{eq:CostFunction}
    \c(z_1, \dots, z_K):= B^*_{\widehat{\mu}_{\vec z}},
\end{equation}
where $\widehat{\mu}_{\vec z}$ is the positive measure (not necessarily a probability measure) defined as:
\begin{equation*}
    \widehat{\mu}_{\vec z} :=  \frac{1}{K} \sum_{l \text{ s.t. } z_l \not = \mathghost }^K \delta_{z_l}.
\end{equation*}
Recall that $B^*_{\widehat{\mu}_{\vec{z}}}$ is equal to \eqref{eq:generalized_barycenter} (alternatively, equal to \eqref{eqn:DualAdversarial}) when $\mu$ is equal to ${\widehat{\mu}_{\vec{z}}}$. In this sense, $\c(z_1, \dots, z_K)$ of \eqref{eq:CostFunction} is the value of the generalized barycenter problem given $\widehat{\mu}_{\vec z}$ as the data distribution, or \textit{local} generalized barycenter problem.

\begin{remark}
Notice that $\widehat{\mu}_{\vec z}$ is a probability measure if and only if no element in the tuple $\vec{z}$ is $\mathghost$.
\end{remark}

Following the literature of MOT, the dual of our MOT problem can be written as 
\begin{equation}\label{eqn:DualMOT}
     \sup_{\phi \in \Phi} \left\{ \sum_{j=1}^K \int_{\mathcal{X} \times [K]} \phi_j(z_j) \frac{1}{2 \mu(\mathcal{Z})}d\mu(z_j) + \frac{1}{2} \sum_{j=1}^K \phi_j(\mathghost) \right\},
\end{equation}
where
\begin{equation}\label{eq:constraint_set}
    \Phi := \left\{ \phi= (\phi_1, \dots, \phi_K) \in \prod_{j=1}^K L^1 \big(\frac{1}{2 \mu(\mathcal{Z})}\mu + \frac{1}{2}  \delta_{\mathghost} \big) : \sum_{j=1}^K \phi_j(z_j) \leq B^*_{\widehat{\mu}_{\vec{z}}},\quad \forall \vec{z} \in \Z_*^K \right\}.
\end{equation}
We will later show that under \textbf{Assumption} \ref{assump:CostStructure} there is no duality gap between the MOT problem and its dual \eqref{eqn:DualMOT} —see \textbf{Corollary} \ref{prop:MOTNoGap}.

One of the main results of the paper is the following.

\begin{theorem}
\label{thm:Main}
Suppose that \textbf{Assumption} \ref{assump:CostStructure} holds. Let $\mu$ be a finite positive measure over $\Z$. Then \eqref{eqn:DualAdversarial} is equivalent to the MOT problem \eqref{Robust problem:Alternative_form} with set of couplings $\Pi_K(\mu)$ defined as in \eqref{eqn:CouplingsMOT}, and cost function $\c$ defined as in \eqref{eq:CostFunction}. Specifically,
\[ \frac{1}{2\mu(\Z)}B^*_{\mu} = \min_{\pi \in \Pi_K(\mu)} \int \c (z_1, \dots, z_K) d\pi (z_1, \dots, z_K).\]
Furthermore, $\eqref{def:RobustClassifPower}=\eqref{eqn:DualAdversarial}$. In addition, from a solution pair $(\pi^*, \phi^*)$ for the MOT problem and its dual one can obtain a solution pair $(f^*, \widetilde{\mu}^*)$ for \eqref{eqn:DualAdversarial} and its dual, i.e. problem \eqref{def:RobustClassifPower}. The pair $(f^*, \widetilde{\mu}^*)$ is also a saddle point for the original adversarial problem \eqref{Robust problem:Intro}.
\end{theorem}

One immediate consequence of \textbf{Theorem} \ref{thm:Main} is that with the identity
\begin{equation*}
    R^*_\mu = \mu(\Z) - \widetilde{B}^*_\mu,
\end{equation*}
one can compute $R^*_\mu$, the optimal adversarial risk, by finding the optimal value of the equivalent MOT problem. To find the latter, one could attempt to use one of the off-the-shelf algorithms in computational optimal transport. Some algorithms to solve generic MOTs that have been developed recently include the ones proposed in see \cite{benamou2015iterative, benamou2019generalized, lin2022complexity, tupitsa2020multimarginal, haasler2021multimarginal, altschuler2021wasserstein, carlier2021linear}. Our numerical results for a subsample of MNIST and CIFAR 10, shown in Figure \ref{plot: cifar_mnist}, are obtained using the algorithm discussed in \cite{lin2022complexity}, also known as MOT Sinkhorn algorithm; see subsection \ref{subsec : numerical expeirment} for more details. We want to warn the reader, however, that off-the-shelf MOT algorithms may suffer an excessive computational burden when $K$ goes beyond $4$. For this reason, it is important to develop algorithms that exploit the structure of our MOT problem, which, as we will discuss below, has the structure of a generalized barycenter problem. An investigation on more specific algorithms is left for future work.

The proof of \textbf{Theorem} \ref{thm:Main} is presented throughout section \ref{sec:ProofMainTheorem}; the expression for $(f^*, \widetilde{\mu}^*)$ in terms of $(\phi^*, \pi^*)$ is presented in \textbf{Corollary} \ref{cor:OptimalPair}. Given the definition of the cost function $\c$, \textbf{Theorem} \ref{thm:Main} states that the adversarial problem \textit{localizes} to data sets consisting of $K$ or less equally weighted points. More precisely, the problem for the adversary reduces to first determining their actions when facing arbitrary distributions supported on $K$ or fewer data points, and then finding an optimal grouping for the data in order to assemble their global strategy. The ghost element, $\mathghost$, indicates when fewer than $K$ points are being grouped by the adversary. We highlight that it is not always (globally) optimal for the adversary to group together points from all the $K$ different classes whenever it is possible. 

We emphasize that from the solution to the MOT and its dual, one can directly obtain an optimal adversarial attack and an optimal classification rule for the original adversarial problem. Note that problem \eqref{Robust problem:Alternative_form} is a problem solved by the adversary: ideally, the adversary wants to group together points $(z_1, \dots,z_K)$ for which there is a low classification power $B^*_{\widehat{\mu}_{\vec z}}$ (or alternatively large robust risk). On the other hand, the dual of \eqref{Robust problem:Alternative_form} can be interpreted as a maximization problem solved by the learner. We formalize this novel connection in subsection \ref{sec:MOTProofs}: see \textbf{Corollary} \ref{cor:OptimalPair}.

In order to prove \textbf{Theorem} \ref{thm:Main}, we will first obtain a series of equivalent reformulations of problem \eqref{eq:generalized_barycenter} which will reveal a rich geometric structure of the adversarial problem and will facilitate the connection with the desired MOT problem. These equivalent formulations are of interest in their own right.

\section{The generalized barycenter problem}
\label{sec:GenBar}

We begin this section by proving that the generalized barycenter problem always has at least one solution.  In the following subsections we will then discuss a series of equivalent problems to the generalized barycenter problem, their duals, and some geometric properties of their solutions. 

\begin{proposition}\label{prop:existence}
Suppose that $c$ is a lower semicontinuous cost satisfying the property that for any compact set $E\subset \mathcal{X}$ there exists a compact set $F\subset \mathcal{X}$ such that for all $x\in E, x'\in F, x''\in \mathcal{X} \setminus F$ we have $c(x,x')\leq c(x, x'')$.
Given finite positive measures $\mu_1, \ldots, \mu_K$ and $c$ as above,
there exists at least one solution to problem \eqref{eq:generalized_barycenter}.
\end{proposition}

\begin{remark}
If $c$ is a cost that satisfies \textbf{Assumption} \ref{assump:CostStructure}, then $c$ satisfies the hypothesis of \textbf{Proposition} \ref{prop:existence}. 
\end{remark}

\begin{remark}
Nearly identical arguments can be used to prove that the various reformulations of \eqref{eq:generalized_barycenter} that we will consider throughout this section have minimizers. For this reason, in what follows, we will simply assume the existence of minimizers without explicitly proving their existence.
\end{remark}

\begin{proof}
Using transportation plans to compute the cost $C(\mu_i, \tilde{\mu}_i)$ in \eqref{eq:generalized_barycenter}, we can rewrite the problem in the following form
\begin{align*}
    &\inf_{\lambda, \, \pi_1,\ldots, \pi_K} \left\{\lambda(\mathcal{X})+\sum_{i \in [K]} \int_{\mathcal{X}\times\mathcal{X}} c(x, x')d\pi_i(x, x')\right\}\\
    &\textup{s.t. } \pi_i(\mathcal{X}\times E)\leq \lambda(E), \pi_i(E\times \mathcal{X})=\mu_i(E) \text{ for all } i \in [K], \quad \forall E \subseteq \X \text{ Borel}.
\end{align*}
Note that a feasible solution to this problem exists since we may choose $\lambda, \pi_1, \ldots, \pi_K$ such that $\lambda:=\sum_{i \in [K]}\mu_i$ and for all $f\in C_c(\mathcal{X}\times \mathcal{X})$ $\int_{\mathcal{X}\times\mathcal{X}}f(x,x')d\pi_i(x,x'):=\int_{\mathcal{X}} f(x,x)d\mu_i(x)$. Also note that with these choices, the problem attains the value $\sum_{i \in [K]} \mu_i(\mathcal{X})$.

Let $\lambda^n, \pi_1^n, \ldots, \pi_K^n$ be a sequence of feasible solutions such that 
\begin{align*}
    t&:=\inf_{\lambda, \pi_1,\ldots, \pi_K} \lambda(\mathcal{X})+\sum_{i \in [K]} \int_{\mathcal{X}\times\mathcal{X}} c(x, x') d\pi_i(x, x')\\
    &= \lim_{n\to\infty}  \lambda^n(\mathcal{X})+\sum_{i \in [K]} \int_{\mathcal{X}\times\mathcal{X}} c(x, x') d\pi_i^n(x, x').
\end{align*}
From our work above and the nonnegativity of the transport cost, $\lambda^n(\mathcal{X})$ is uniformly bounded by $\sum_{i \in [K]} \mu_i(\mathcal{X})$. Furthermore, we may assume that for any Borel set $E$
\[
\sum_{i \in [K]} \int_{\mathcal{X}\times E} d\pi_i^n(x, x') \geq \lambda^n(E),
\]
otherwise we could delete mass from $\lambda^n$ and attain a smaller value.
Given some $\epsilon>0$, let $E_{\epsilon}\subset \mathcal{X}$ be a compact set such that $\sum_{i \in [K]} \mu_i(\mathcal{X}\setminus E_{\epsilon})\leq \epsilon$.
Let $F_{\epsilon}$ be a compact set such that for all $x\in E_{\epsilon}, x'\in F_{\epsilon} $ and $x''\in \mathcal{X}\setminus F_{\epsilon}$ we have $c(x,x')\leq c(x,x'')$. If $\lambda^n$ gives more than $\epsilon$ to $\mathcal{X}\setminus F_{\epsilon}$ then some of this mass must be transported to $E_{\epsilon}$.  Since the transportation cost would be cheaper if the excess mass was placed inside of $F_{\epsilon}$ instead of $\mathcal{X}\setminus F_{\epsilon}$, it follows that $\lambda^n(\mathcal{X}\setminus F_{\epsilon})\leq \epsilon$.  Therefore, the $\lambda^n$ are a tight family.

The tightness of $\lambda^n$ and $\mu_1,\ldots, \mu_K$ implies that $\pi_1^n,\ldots, \pi_K^n$ are a tight family.  Therefore, we can extract a subsequence that converges weakly to a limit $\lambda^*, \pi_1^*,\ldots, \pi_K^*$. From the lower semicontinuity of the cost, it follows that $\{\lambda^*, \pi_1^*,\ldots, \pi_K^* \}$ is a minimizer. 
\end{proof}

\subsection{A first MOT reformulation of \eqref{eq:lambda_decomposed_formulation} and geometric consequences}

In the rest of what follows, we shall let $S_K$ denote the power set of $[K]$ except for the empty set and for every $i\in [K]$ we let $S_{K}(i)=\{A\in S_K: i\in A\}$. We can reduce \eqref{eq:generalized_barycenter} to a more concrete problem by partitioning $\lambda$ and each of $\mu_i$'s properly, eliminating the variables $\widetilde{\mu}_i$'s from the optimization. We start with the following observation.


\begin{lemma}
\label{lem:DecomposistionA}
Let $u_1, \dots, u_K \in [0,1]$ be such that $\max_{i=1, \dots, K} u_i =1$. Then there exists a collection of non-negative scalars $\{ r_A \}_{A \in S_K}$ such that the following two conditions hold:
\begin{enumerate}
    \item $1=\sum_{A \in S_K} r_A $.
    \item $u_i = \sum_{A \in S_K(i)} r_A$ for all $i=1,\dots,K$.
\end{enumerate}
\end{lemma}
\begin{proof}
Without loss of generality we can assume that the $u_i$ are arranged in increasing order. That is,
\[ 
0 \leq u_1 \leq u_2 \leq \dots, \leq u_K=1. 
\]
Let $i'$ be the first $i$ such that $u_{i}>0$. We set
\begin{align*}
    r_{\{ i', \dots, K \}} &:= u_{i'}\\
    r_{\{ i'+1, \dots, K \}} &:=  u_{i' +1} -  u_{i'}\\
    r_{\{ i'+2, \dots, K \}} &:=  u_{i' +2} -  u_{i'+1}\\
    &\quad \vdots \\
    r_{\{ K\} } &:= 1- u_{K-1}. 
\end{align*}
and $r_A=0$ for all other sets. It is straightforward to check that the collection $\{ r_A \}_{A \in S_K}$ defined in this way satisfies the required conditions.
\end{proof}


\begin{proposition}
\label{prop:EquivalenceWithPartitions}
Problem \eqref{eq:generalized_barycenter} is equivalent to
\begin{equation}\label{eq:lambda_decomposed_formulation}
\begin{aligned}
    &\inf_{ \{ \lambda_A, \mu_{i,A} : i \in [K], A \in S_K\}} \sum_{A\in S_K}  \left\{ \lambda_A(\mathcal{X})+\sum_{i\in A} C(\lambda_A, \mu_{i,A}) \right\}\\
    &\textup{s.t. } \sum_{A\in S_K(i)} \mu_{i,A}=\mu_i  \textup{ for all } i\in [K].
\end{aligned}
\end{equation}
\end{proposition}

\begin{proof} We split the proof into two parts.

\textbf{Step 1:}
Suppose that $\lambda , \tilde{\mu}_1, \dots, \tilde \mu_K$ is feasible for problem \eqref{eq:generalized_barycenter}. In particular, $\lambda \geq \tilde \mu _i$ for all $i$. Let us denote by $\frac{d \tilde \mu_i}{d \lambda}$ the Radon-Nikodym derivative of $\tilde \mu_i$ w.r.t. $\lambda$. Notice that $\frac{d \tilde \mu_i}{d\lambda} \leq 1$ because $\lambda$ dominates $\tilde \mu_i$. Moreover, without the loss of generality we can assume that for every $x \in \spt(\lambda)$ we have $\max_{i=1, \dots, K} \frac{d \tilde \mu_i}{d \lambda}(x) =1$, for otherwise we could modify $\lambda$ and potentially reduce the energy in \eqref{eq:generalized_barycenter} while maintaining the constraints. 

For each $x \in \spt(\lambda)$ we  apply Lemma \ref{lem:DecomposistionA} with $u_i(x) := \frac{d\tilde \mu_i}{d\lambda}(x)$ to obtain a collection of scalars $\{ r_A(x) \}_{A \in S_K}$ satisfying:
\begin{enumerate}
    \item[(i)] $1=\sum_{A \in S_K} r_A(x) $.
    \item[(ii)] $ u_i(x) = \sum_{A \in S_K(i)} r_A(x)$ for all $i=1,\dots,k$.
\end{enumerate}
Notice that the functions $r_A(\cdot)$ can be constructed in a measurable way as it follows from the proof of Lemma \ref{lem:DecomposistionA}.
For each $A\in S_K$ we define the measure $\lambda_A$ as
\[ \frac{d \lambda_A}{d\lambda}(x):= r_A(x),  \]
and for $A$ and $i \in A$ we define 
\[ \tilde \mu_{i,A}:= \lambda_A. \]
See Figure \ref{plot: combined_figure} (a) for an illustration of the $\lambda_A$'s. From the above definitions and the properties of the functions $r_A$ we deduce 
\[ \sum_{A \in S_K} d\lambda_A(x) = \sum_{A \in S_K}  r_A(x) d\lambda(x)  = d\lambda(x)   \] 
and
\[ \sum_{A \in S_K(i)} d\tilde \mu_{i,A}(x) = \sum_{A \in S_K(i)}  r_A(x) d\lambda(x)  =  \frac{d \tilde \mu_i}{d\lambda }(x) d\lambda(x) =  d \tilde \mu_i(x). \]

Now, let $\pi_i \in \Gamma(\mu_i, \tilde \mu_{i})$ be a coupling realizing the cost $C(\mu_i, \tilde \mu_i)$, i.e., a minimizer of \eqref{eqn:OTProblemMu_i}, and use the disintegration theorem to write it as
\[ d\pi_i(x, \tilde x) = d\pi_i^*(x |\tilde x) d\tilde \mu_i( \tilde x),    \]
where $d\pi_i^*(\cdot |\tilde x)$ is the conditional of $x$ given $\tilde x$ according to the joint distribution $\pi_i^*$. For each $A \in S_K$ and $i \in A$ we define the measure $\pi_{i,A}$ according to
\[  d\pi_{i,A} (x, \tilde x) := d\pi_i^*(x|\tilde x) d\tilde \mu_{i,A}(\tilde x). \]
Finally, we set $\mu_{i,A}$ to be the first marginal of $\pi_{i,A}$. 

It is now straightforward to show that $\{ \lambda_A, \mu_{i,A} \}$ is feasible for \eqref{eq:lambda_decomposed_formulation}. Moreover, 
\[ \lambda(\X) + \sum_{i=1}^k C(\mu_i, \tilde \mu_i) \geq \sum_{A\in S_K}  \left\{ \lambda_A(\mathcal{X})+\sum_{i\in A} C(\lambda_A, \mu_{i,A}) \right\}. \]

\textbf{Step 2:}
Conversely, suppose that $\{\lambda_A\}_{A}$, $\{  \mu_{i,A}\}_A$ is feasible for \eqref{eq:lambda_decomposed_formulation}. Set $\lambda:=  \sum_{A \in S_K} \lambda_A$ and for every $i$ let $\tilde \mu_i := \sum_{A \in S_{k}(i)} \lambda _A$. Clearly we have $\lambda \geq \tilde \mu_{i}$ for all $i$. Moreover, let $\pi_{i,A} \in \Gamma(\mu_{i,A}, \lambda_A)$ realizing the cost $C(\lambda_A,\mu_{i,A})$. See (b) of Figure \ref{plot: combined_figure} to understand how $\mu_{i,A}$ is transported to $\lambda_A$. Finally, for each $i$ we set
\[ \pi_i:= \sum_{A \in S_K(i)} \pi_{i,A}.   \]
With these constructions it is now straightforward to show that
\[  \sum_{A\in S_K}  \left\{ \lambda_A(\mathcal{X})+\sum_{i\in A} C(\lambda_A, \mu_{i,A}) \right\}\geq \lambda(\X) + \sum_{i=1}^k C(\mu_i, \tilde \mu_i).
\]
\end{proof}

\begin{figure}[h]
    \centering
	\includegraphics[scale=0.35]{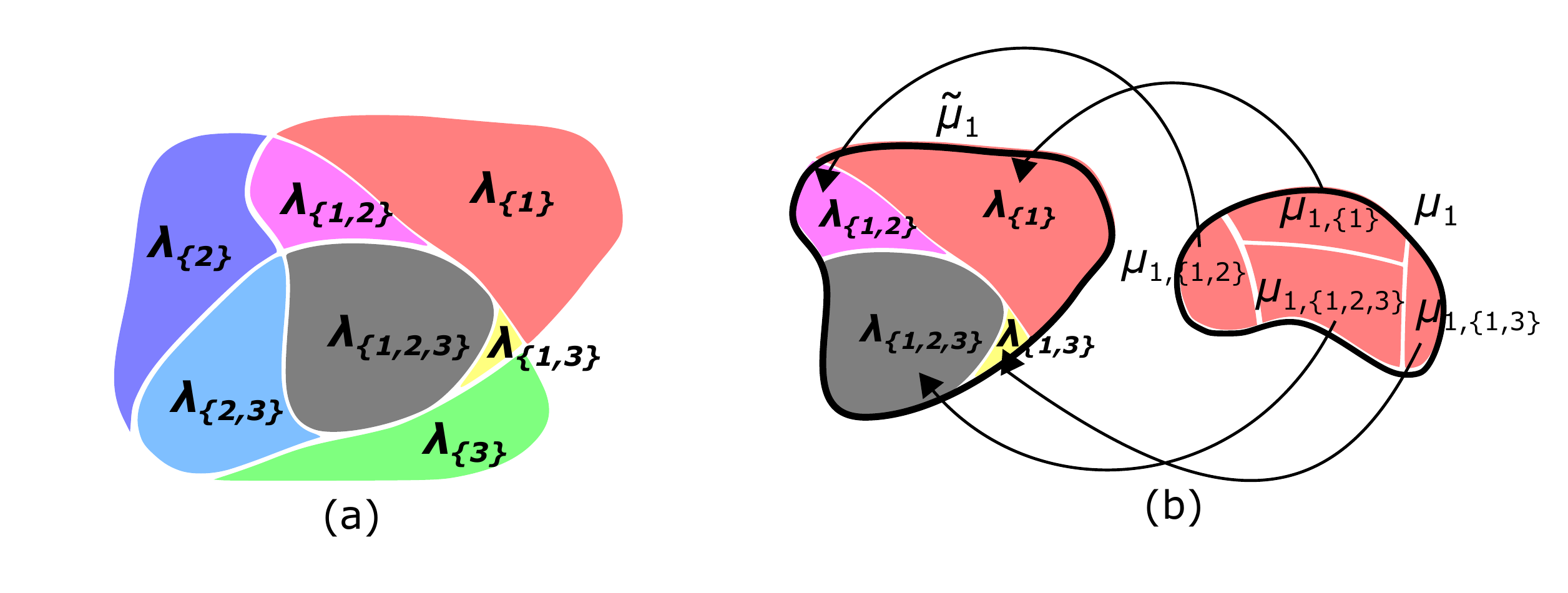}
	\centering
	\caption{(a) : Illustration of a partition of $\lambda$. (b) : Illustration of the transport from $\mu_{1, A}$'s to $\lambda_A$'s.}
	\label{plot: combined_figure}
\end{figure}


\begin{figure}[h]
    \centering
	\includegraphics[scale=0.35]{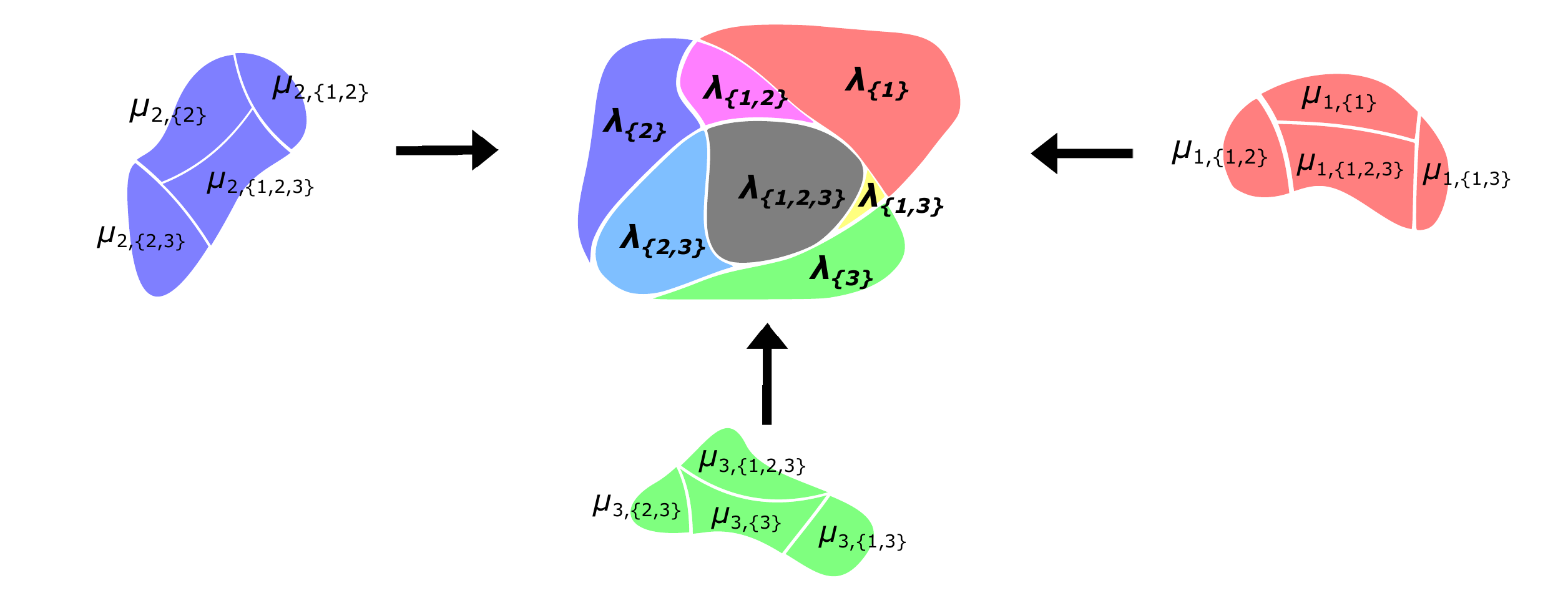}
	\centering
	\caption{Picture for \eqref{eq:lambda_decomposed_formulation}. Each of $\mu_{i,A}$'s is transported to $\lambda_A$ for all $i \in A$.}
	\label{plot: mu_lambda_2}
\end{figure}

\begin{remark}
Figure \ref{plot: mu_lambda_2} illustrates the partitions for $\lambda$ and the $\mu_i$'s. To keep notation from getting too complicated, in the sequel we shall assume that $\mu_{i,A}$ is defined for all $i\in [K]$ and $A\subseteq S_K$, however, note that if $i\notin A$, then $\mu_{i,A}$ plays no role in the optimization \eqref{eq:lambda_decomposed_formulation}.
\end{remark}

Suppose that for some $A\in S_K$ we fix a choice of $\mu_{i,A}$ for all $i\in A$.  With the $\mu_{i,A}$ fixed,  we can determine the corresponding optimal $\lambda_A^*=\lambda_A^*(\mu_{1,A},\ldots, \mu_{K,A})$ by solving the classic Wasserstein barycenter problem. Indeed, the optimal choice must be an element of
\begin{equation}\label{eq:classical_A}
\argmin_{\lambda_A} \sum_{i\in A} C(\lambda_A, \mu_{i,A}).
\end{equation}
Note that here we do not need to consider the mass of $\lambda_A$, since the value of the optimization problem will be $+\infty$ if $\lambda_A$ does not have the same mass as all of the $\mu_{i,A}$ (or if the $\mu_{i,A}$ themselves do not all have the same mass).

It is well known that problem \eqref{eq:classical_A} can be reformulated as a multimarginal optimal transport problem \cite{MR2801182}; see also our subsection \ref{sec:GeneralMOT}. To that end, given $A\subseteq [K]$,
define $c_A: \mathcal{X}^K \to \RR$ 
\begin{equation}\label{eq:cost_A}
c_A(x_1, \ldots, x_K):=\inf_{x'\in \mathcal{X}} \sum_{i\in A} c(x',x_i),
\end{equation}
and $T_A:\mathcal{X}^K\to \mathcal{X}$
\begin{equation}\label{eq:barycenter_A}
T_A(x_1, \ldots, x_K):=\argmin_{x'\in \mathcal{X}} \sum_{i\in A} c(x',x_i).
\end{equation}
\begin{remark}
\label{rem:Barycenter}
If $\argmin_{x'\in \mathcal{X}} \sum_{i\in A} c(x',x_i)$ is not unique, we can consider using an additional selection procedure. For example, when $\X=\R^d$ we can still recover a unique mapping by choosing $T_A$ to be the element of $\argmin_{x'\in \mathcal{X}} \sum_{i\in A} c(x',x_i)$ that is closest (in the Euclidean distance) to the Euclidean barycenter $\frac{1}{|A|}\sum_{i\in A} x_i$.
\end{remark}

With the definition of $c_A$, we can rewrite \eqref{eq:classical_A} as the multimarginal optimal transport problem
\begin{equation}\label{eq:multi_A}
    \inf_{\pi_A} \int_{\mathcal{X}^K} c_A(x_1, \ldots, x_K)d\pi_A(x_1, \ldots, x_K) \quad \textup{s.t. } \mathcal{P}_{i\,\#}\pi_A=\mu_{i,A}  \textup{ for all } i\in A,
\end{equation}
where $\mathcal{P}_i$ is the projection map $(x_1, \ldots, x_K)\mapsto x_i$.
Again, even though $\pi_A$ is defined over $\mathcal{X}^K$,  only the coordinates $i$ where $i\in A$ play a role in the optimization problem. Indeed, $c_A$ is independent of the other coordinates and we only have marginal constraints for $i\in A$.

Using \eqref{eq:multi_A} we can now eliminate $\lambda_A$ and all of the $\mu_{i,A}$'s from problem \eqref{eq:lambda_decomposed_formulation} and reformulate the optimization as the multimarginal problem
\begin{equation}\label{eq:multimarginal_decomposed}
\begin{aligned}
    &\inf_{ \{ \pi_A : A \in S_K\} } \sum_{A\in S_K} \int_{\mathcal{X}^K} \big(c_A(x_1,\ldots, x_K)+1\big) d\pi_A(x_1, \dots, x_K)\\
    &\textup{s.t.} \sum_{A\in S_K(i)}\mathcal{P}_{i\,\#}\pi_A=\mu_i  \textup{ for all } i\in [K].
\end{aligned}
\end{equation}

The next two propositions formally prove the equivalence between \eqref{eq:lambda_decomposed_formulation} and \eqref{eq:multimarginal_decomposed}. They will also allow us to establish some important geometric properties of optimal generalized barycenters.

\begin{proposition}\label{prop:multimarginal_to_lambda}
Let $c$ be a cost satisfying \textbf{Assumption} \ref{assump:CostStructure}.
Given measures $\mu_1, \ldots, \mu_K$, let $\{\pi_A\}_{A\in S_K}$ be a feasible solution to \eqref{eq:multimarginal_decomposed}. For each $(x_1,\ldots, x_K)\in \mathcal{X}^K$ and $A\in S_K$, let $f_A(x_1,\ldots, x_K)$ be a choice of element in $T_A(x_1,\ldots x_K)$, where we recall the definition of $T_A(x_1,\ldots, x_K)$ from \eqref{eq:barycenter_A}. 

If for each $A\in S_K$ and $i\in A$ we set $\tilde{\lambda}_A=f_{A\,\#}\pi_A$ and $\tilde{\mu}_{i,A}=\mathcal{P}_{i\,\#}\pi_A$, then $\{\tilde{\lambda}_A, \tilde{\mu}_{i,A} : A\in S_K, i\in A \}$ is a feasible solution to \eqref{eq:lambda_decomposed_formulation} and
\[
\sum_{A\in S_K} \tilde{\lambda}_A(\mathcal{X})+\sum_{i\in A}C(\tilde{\lambda}_A,\tilde{\mu}_{i,A})\leq \sum_{A\in S_K} \int_{\mathcal{X}^K} \big(c_A(x_1,\ldots, x_K)+1\big)d\pi_A(x_1,\ldots, x_K).
\]
\end{proposition}

\begin{proof}
Since $\sum_{A\in S_K(i)}\mathcal{P}_{i\,\#}\pi_A=\mu_{i}$, it is automatic that $\sum_{A\in S_K(i)}\tilde{\mu}_{i,A}=\mu_i$. Since pushforwards do not affect the total mass of a measure, so we also have $\tilde{\mu}_{i,A}(\mathcal{X})=\tilde{\lambda}_A(\mathcal{X})$ for all $A\in S_K$ and $i\in A$. Hence, $\{\tilde{\lambda}_A, \tilde{\mu}_{i,A}\}_{A\in S_K, i\in A}$ is a feasible solution to \eqref{eq:lambda_decomposed_formulation}.

For each $A\in S_K$ and $i\in A$, choose $\vp_{i,A}, \psi_{i,A}\in C_b(\mathcal{X})$ that satisfy, for all $x,x'\in \mathcal{X}$,
\[
\vp_{i,A}(x)-\psi_{i,A}(x')\leq c(x,x').
\]
We can then compute
\begin{align*}
    &\int_{\mathcal{X}} \vp_{i,A}(x_i)d\tilde{\mu}_{i,A}(x_i)-\int_{\mathcal{X}}\psi_{i,A}(x')d\tilde{\lambda}_A(x')\\
    =& \int_{\mathcal{X}} \vp_{i,A}(x_i)d\tilde{\mu}_{i,A}(x_i)-\int_{\mathcal{X}^K} \psi_{i,A}(f_A(x_1,\ldots, x_K))d\pi_{A}(x_1,\ldots, x_K)\\
   \leq& \int_{\mathcal{X}} \vp_{i,A}(x_i)d\tilde{\mu}_{i,A}(x_i)+\int_{\mathcal{X}^K} \Big(c\big(x_i, f_A(x_1,\ldots, x_K)\big)-\vp_{i,A}(x_i)\Big)d\pi_{A}(x_1,\ldots, x_K)\\
    =&\int_{\mathcal{X}^K} c\big(x_i, f_A(x_1,\ldots, x_K)\big)d\pi_{A}(x_1,\ldots, x_K).
\end{align*}
Thus, 
\begin{align*}
    &\sum_{i\in A}\int_{\mathcal{X}} \vp_{i,A}(x_i)d\tilde{\mu}_{i,A}(x_i)-\int_{\mathcal{X}}\psi_{i,A}(x')d\tilde{\lambda}_A(x')\\
    \leq& \int_{\mathcal{X}^K} \sum_{i\in A}c\big(x_i, f_A(x_1,\ldots, x_K)\big)d\pi_{A}(x_1,\ldots, x_K)\\
    =&\int_{\mathcal{X}^K} c_A(x_1,\ldots, x_K)d\pi_{A}(x_1,\ldots, x_K),
\end{align*}
where we have used the definition of $f_A, T_A$, and $c_A$ to obtain the last equality. Hence, 
\begin{align*}
    &\sum_{A\in S_K} \tilde{\lambda}_A(\mathcal{X})+\sum_{i\in A}\int_{\mathcal{X}} \vp_{i,A}(x_i)d\tilde{\mu}_{i,A}(x_i)-\int_{\mathcal{X}}\psi_{i,A}(x')d\tilde{\lambda}_A(x')\\
    \leq& \sum_{A\in S_K}\int_{\mathcal{X}^K}  \big(c_A(x_1,\ldots, x_K)+1\big)d\pi_{A}(x_1,\ldots, x_K).
\end{align*}
Taking the supremum over all admissible choices of $\vp_{i,A}, \psi_{i,A}$ and exploiting the dual formulation of optimal transport,
\begin{equation*}
    \sum_{A\in S_K}\tilde{\lambda}_A(\mathcal{X})+\sum_{i\in A} C(\tilde{\lambda}_A, \tilde{\mu}_{i,A})\leq \sum_{A\in S_K}\int_{\mathcal{X}^K}  \big(c_A(x_1,\ldots, x_K)+1\big)d\pi_{A}(x_1,\ldots, x_K),
\end{equation*}
which is the desired result we want.
\end{proof}

In the next proposition we will show that any feasible solution of problem \eqref{eq:lambda_decomposed_formulation} induces a feasible solution of \eqref{eq:multimarginal_decomposed} with lesser or equal value. This will prove the equivalence between problems \eqref{eq:lambda_decomposed_formulation} and \eqref{eq:multimarginal_decomposed} and will provide a powerful geometric characterization of optimal generalized barycenters.

\begin{proposition}\label{prop:lambda_to_multimarginal}
Let $c$ be a cost satisfying \textbf{Assumption} \ref{assump:CostStructure}.
Given measures $\mu_1,\ldots, \mu_K$,  let $\mu_{i,A}, \lambda_A$ be feasible solutions to problem \eqref{eq:lambda_decomposed_formulation}. Let $\gamma_{i,A}\in \mathcal{M}(\mathcal{X}\times \mathcal{X})$ be an optimal plan for the transport of $\mu_{i,A}$ to $\lambda_A$ with respect to the cost $c$.  Let $\gamma_A\in \mathcal{M}(\mathcal{X}^{K+1})$ such that for all $i\in A$ and $g\in C_b(\mathcal{X}\times \mathcal{X})$
\[
\int_{\mathcal{X}^{K+1}} g(x_i, x')d\gamma_A(x_1,\ldots, x_K, x')=\int_{\mathcal{X}^{K+1}} g(x_i, x')d\gamma_{i,A}(x_i,x').
\]
If we define $\tilde{\pi}_A$ on $\mathcal{X}^K$ such that for any $h\in C_b(\mathcal{X}^K)$ we have 
\begin{equation*}
    \int_{\mathcal{X}^K} h(x_1,\ldots, x_K)d\pi_A(x_1,\ldots, x_K)=\int_{\mathcal{X}^{K+1}} h(x_1,\ldots, x_K)d\gamma_A(x_1,\ldots, x_K,x'),
\end{equation*}
then $\tilde{\pi}_A$ is a feasible solution to \eqref{eq:multimarginal_decomposed} and 
\[
\sum_{A\in S_K} \int_{\mathcal{X}^K} \big(c_A(x_1,\ldots, x_K)+1\big)d\tilde{\pi}_A(x_1,\ldots, x_K)\leq \sum_{A\in S_K} \lambda_A(\mathcal{X})+\sum_{i\in A}C(\lambda_A,\mu_{i,A}).
\]
Therefore, $\eqref{eq:lambda_decomposed_formulation}=\eqref{eq:multimarginal_decomposed}$.
\end{proposition}

\begin{proof}
We begin by noting that the marginal constraints on $\gamma_A$ are compatible in the sense that for any $g \in C_b(\mathcal{X})$ and $i\in A$ we have 
\[
\int_{\mathcal{X}} g(x')d\gamma_{i,A}(x_i,x')=\int_{\mathcal{X}}g(x')d\lambda_A(x').
\]
Thus, each $\gamma_A$ is well-defined.

Using the definition of $d\tilde{\pi}_A$ and then $c_A$, it follows that
\begin{align*}
    &\sum_{A\in S_K} \int_{\mathcal{X}^K} \big(c_A(x_1,\ldots, x_K)+1
\big)d\tilde{\pi}_A(x_1,\ldots, x_K)\\
    =&\sum_{A\in S_K} \int_{\mathcal{X}^{K+1}} \big(c_A(x_1,\ldots, x_K)+1
\big)d\gamma_A(x_1,\ldots, x_K,x')\\
    \leq& \sum_{A\in S_K} \int_{\mathcal{X}^{K+1}} \big(1+\sum_{i\in A} c(x_i,x')
\big)d\gamma_A(x_1,\ldots, x_K,x')\\
    =& \sum_{A\in S_K} \int_{\mathcal{X}^{K+1}} \big(1+\sum_{i\in A} c(x_i,x')
\big)d\gamma_{i,A}(x_i,x')\\
    =&\sum_{A\in S_K} \lambda_A(\mathcal{X})+C(\mu_{i,A},\lambda_A)
\end{align*}
where the final equality follows from the fact that $\gamma_{i,A}$ is an optimal plan for the transport of $\mu_{i,A}$ to $\lambda_A$.
\end{proof}

In addition to proving the equivalence between problems \eqref{eq:lambda_decomposed_formulation} and \eqref{eq:multimarginal_decomposed}, \textbf{Proposition} \ref{prop:multimarginal_to_lambda} and \textbf{Proposition} \ref{prop:lambda_to_multimarginal} have the following very important geometric consequences.

\begin{corollary}\label{cor:barycenter_geometry}
Let $c$ be a cost satisfying \textbf{Assumption} \ref{assump:CostStructure}. Given measures $\mu_1,\ldots, \mu_K$, let $\lambda$ be an optimal generalized barycenter and let $\{\lambda_A\}_{A\in S_K}$ be a decomposition of $\lambda$ and $\{\mu_{i,A}\}_{A\in S_K(i)}$ a decomposition of each $\mu_i$ that are optimal for \eqref{eq:lambda_decomposed_formulation}. Recalling \eqref{eq:barycenter_A}, let $T_A(x_1,\ldots,x_K):=\argmin_{x\in\mathcal{X}} \sum_{i\in A}c(x,x_i)$.  If we define $T_A:=\{T_A(x_1,\ldots, x_K): x_1 \in \spt (\mu_1),\ldots, x_K \in \spt(\mu_K)\}$ and $T=\cup_{A \subseteq [K]} T_A$, then $\lambda_A(\mathcal{X})=\lambda_A(T_A)$, $\lambda(\mathcal{X})=\lambda(T)$ and the optimal measures $\widetilde{\mu}_i$ in \eqref{eq:generalized_barycenter} can be assumed to satisfy $\widetilde{\mu}_i(\X ) = \widetilde{\mu}_i(T) $ as well.  

In particular, if $f_A(x_1,\ldots, x_K)$ is a choice of element from $T_A(x_1,\ldots, x_K)$ for each $A\in S_K$ and $(x_1,\ldots, x_K)\in \mathcal{X}^K$, then there exists an optimal barycenter $\lambda_f$ such that $\lambda_f(\mathcal{X})=\lambda_f(F)$ where $F=\bigcup_{A\in S_K}\bigcup_{(x_1,\ldots x_K)\in \spt(\mu_1)\times \cdots \times \spt(\mu_K)} f_A(x_1,\ldots, x_K)$.
\end{corollary}

\begin{remark}
In the case where we have a tuple $(x_1,\ldots, x_K)\in \spt(\mu_1)\times\cdots\times \spt(\mu_K)$ such that $\sum_{i\in A} c(x,x_i)=+\infty$ for all $x\in \mathcal{X}$,  we set $T_A(x_1,\ldots, x_K)=\emptyset$.
\end{remark}

\begin{proof}
From \textbf{Proposition} \ref{prop:lambda_to_multimarginal}, we can use $\{\lambda_A\}_{A\in S_K}$ and $\{\mu_{i,A}\}_{A\in S_K, i\in A}$ to construct measures $\{\tilde{\pi}_A\}_{A\in S_K}$ with
\begin{equation}\label{eq:lambda_to_multimarginal}
\sum_{A\in S_K} \int_{\mathcal{X}^K} \big(c_A(x_1,\ldots, x_K)+1\big)d\tilde{\pi}_A(x_1,\ldots, x_K)\leq \sum_{A\in S_K} \lambda_A(\mathcal{X})+\sum_{i\in A}C(\lambda_A,\mu_{i,A}).
\end{equation}
From \textbf{Proposition} \ref{prop:multimarginal_to_lambda}, we can then use $\tilde{\pi}_A$ to construct decompositions $\{\tilde{\lambda}_A\}_{A\in S_K}$ and $\{\tilde{\mu}_{i,A}\}_{A\in S_K, i\in A}$ such that
\begin{equation}\label{eq:multimarginal_to_lambda}
\sum_{A\in S_K} \tilde{\lambda}_A(\mathcal{X})+\sum_{i\in A}C(\tilde{\lambda}_A,\tilde{\mu}_{i,A})
\leq \sum_{A\in S_K} \int_{\mathcal{X}^K} \big(c_A(x_1,\ldots, x_K)+1\big)d\tilde{\pi}_A(x_1,\ldots, x_K).
\end{equation}
Examining the proof of \textbf{Proposition} \ref{prop:lambda_to_multimarginal}, it follows that the inequality in \eqref{eq:lambda_to_multimarginal} is strict if $\lambda_A(\mathcal{X})>\lambda_A(T_A)$. In that case, combining \eqref{eq:lambda_to_multimarginal} and \eqref{eq:multimarginal_to_lambda} would contradict the optimality of $\lambda$. Therefore, $\lambda_A(T_A)=\lambda_A(\mathcal{X})$. The final statements follow from the constraints satisfied by the $\tilde{\mu}_i$ and the construction in \textbf{Proposition} \ref{prop:multimarginal_to_lambda}.
\end{proof}

When $\mu_1,\ldots, \mu_K$ are supported on a finite set of points, \textbf{Corollary} \ref{cor:barycenter_geometry} has the following consequence.

\begin{corollary}\label{cor:finite_means_finite}
If $\mu_1,\ldots, \mu_K$ are measures that are supported on a finite set of points and $c$ is a cost satisfying \textbf{Assumption} \ref{assump:CostStructure}, then there exists a solution $\lambda$ to the optimal generalized barycenter problem \eqref{eq:generalized_barycenter} that is supported on a finite set of points. 

In particular, if each $\mu_i$ is supported on a set of $n_i$ points, then there exists an optimal barycenter that is supported on at most $\sum_{A\in S_K}\prod_{i\in A} n_i\leq 2^K\prod_{i=1}^K n_i$ points.
\end{corollary}

\begin{remark}
Notice that the bound mentioned at the end of \textbf{Corollary} \ref{cor:finite_means_finite} is a worst case bound. In practice, especially when data sets have a favourable geometric structure, the optimal barycenter $\lambda$ may have a much sparser support. See section \ref{ex : toy_example}.
\end{remark}

\begin{proof}
For each $i\in [K]$ we can assume there exists a finite set $X_i\subset \mathcal{X}$ such that $\mu_i$ is supported on $X_i$.  For each $A\in S_K$, let $f_A:X_i^K\to \mathcal{X}$ be a function such that
\[
f_A(x_1,\ldots, x_K)\in T_A(x_1,\ldots, x_K)
\]
for all $(x_1,\ldots, x_K)\in X_i^K$,
where we recall the definition of $T_A$ from \eqref{eq:barycenter_A}.
We can now construct the set
\[
F=\bigcup_{A\in S_K}\bigcup_{(x_1,\ldots, x_K)\in \prod_{i=1}^K X_i} \{f_A(x_1,\ldots, x_K)\},
\]
which is necessarily finite.
Indeed, if we set $n_i=|X_i|$, then $F$ has at most $\sum_{A\in S_K} \prod_{i\in A} n_i$ elements. By \textbf{Corollary} \ref{cor:barycenter_geometry}, there exists an optimal barycenter supported on $F$ only.
\end{proof}

\subsection{A second MOT reformulation of \eqref{eq:lambda_decomposed_formulation}}

Note that in problem \eqref{eq:multimarginal_decomposed} we need to find a  distribution $\pi_A$ for each $A\in S_K$.  Hence, it is natural to wonder if we can reformulate problem \eqref{eq:multimarginal_decomposed} in such a way that we only need to find a single distribution $\gamma$.  Here one must be careful, as the previous formulations of the problem do not require the input distributions $\mu_1, \ldots, \mu_K$ to have the same mass. As a result, if we try to work over a space of probability distributions whose marginals are $\mu_1, \ldots, \mu_K$, then we cannot recover the full generality of \eqref{eq:multimarginal_decomposed}.  

To overcome this difficulty, we will define $\gamma$ over the slightly larger space $(\mathcal{X} \times [0,1])^K$.  The extra coordinate will help us track the mass associated to each label $i$. Define $\widetilde{c}:(\mathcal{X} \times [0,1])^K\to\RR$ by 
\begin{equation}
\label{eq:cost_full_multimarginal}
\begin{aligned}
    &\widetilde{c}((x_1, r_1), \ldots, (x_K, r_K))\\
    & :=\inf_{m:S_K\to \RR} \sum_{A\in S_K} m_A\big(c_A(x_1, \ldots, x_K)+1\big) \quad \textup{s.t.} \sum_{A\in S_K(i)} m_A=r_i. 
\end{aligned}
\end{equation}
For each $i\in [K]$, let $\tilde{\mathcal{P}}_i$ be the projection $((x_1, r_1), \ldots, (x_K, r_K))\mapsto  x_i$. In what follows, we use $(\vec x , \vec r)$ to denote the tuple $((x_1, r_1), \dots, (x_K, r_K))$. We then claim that problem \eqref{eq:multimarginal_decomposed} is equivalent to 
\begin{equation}\label{eq:full_multimarginal}
    \inf_{\gamma} \int_{(\mathcal{X} \times [0,1])^K} \widetilde{c}(\vec{x}, \vec{r}) d\gamma(\vec{x}, \vec{r}) \quad \textup{s.t. }  \tilde{\mathcal{P}}_{i\,\#}(r_i\gamma)=\mu_i  \textup{ for all }   i\in [K].
\end{equation}

\begin{proposition}
\label{prop:FirstMOT}
Problems \eqref{eq:multimarginal_decomposed} and \eqref{eq:full_multimarginal} are equivalent, and thus \eqref{eq:full_multimarginal} is also equivalent to \eqref{eqn:DualAdversarial}, \eqref{eq:generalized_barycenter} and \eqref{eq:lambda_decomposed_formulation}.
\end{proposition}

\begin{proof}
Given a feasible solution $\pi_{\{1\}}, \ldots, \pi_{[K]}$ to problem \eqref{eq:multimarginal_decomposed}, define $\gamma$ such that for every continuous and bounded function $f:(\mathcal{X} \times [0,1])^K\to \RR$ we have
\begin{equation*}
    \int_{(\mathcal{X} \times [0,1])^K} f(\vec{x}, \vec{r}) d\gamma (\vec{x}, \vec{r}) = \sum_{A\in S_K} \int_{\mathcal{X}^K} f\big((x_1,\chi_A(1)), \ldots, (x_K,\chi_A(K))\big) d\pi_A(x_1, \dots, x_K).
\end{equation*}
where $\chi_A(i)=1$ if $i\in A$ and zero otherwise. We can then check that $\gamma$ is feasible for \eqref{eq:full_multimarginal}, since for any function $g:\mathcal{X}\to\RR$
\begin{align*}
    \int_{(\mathcal{X} \times [0,1])^K} r_ig(x_i) d\gamma \big(\vec{x}, \vec{r}) &=\sum_{A\in S_K(i)} \int_{\mathcal{X}^K} g(x_i) d\pi_A( x_1, \dots, x_K )\\
    &=\int_{\mathcal{X}} g(x_i) d\mu_i(x_i),
\end{align*}
where the final equality uses the fact that $\sum_{A \in S_K(i)}\mathcal{P}_{i\,\#}\pi_A=\mu_i$.

Next, we observe that for any $A\in S_K$ and a tuple of the form $((x_1,\chi_A(1)), \ldots, (x_K,\chi_A(K))\big)$ we have 
\[
\widetilde{c}((x_1,\chi_A(1)), \ldots, (x_K,\chi_A(K))\big)\leq c_A(x_1, \ldots, x_K)+1.
\]
Therefore, 
\[
\int_{(\mathcal{X} \times [0,1])^K} \widetilde{c}(\vec{x}, \vec{r}) d\gamma (\vec{x}, \vec{r}) \leq \sum_{A\in S_K} \int_{\mathcal{X}^K} \big(c_A(x_1, \dots, x_K) + 1 \big) d\pi_A (x_1, \dots, x_K).
\]

Conversely, suppose that $\gamma$ is a feasible solution to \eqref{eq:full_multimarginal}. Given a tuple $(\vec{x}, \vec{r})$, let
\[
m_A(\vec{x}, \vec{r}) \in \argmin_{m:S_K\to \RR} \sum_{A\in S_K} m_A\big(c_A(x_1, \ldots, x_K)+1\big) \quad \textup{s.t.} \sum_{A\in S_K(i)} m_A=r_i. 
\]
Given $A\in S_K$ define $\pi_A$ such that for any continuous and bounded function $h:\mathcal{X}^K\to\RR$ we have
\begin{equation*}
    \int_{\mathcal{X}^K}h(x_1, \dots, x_K) d\pi_A(x_1, \dots, x_K) =\int_{(\mathcal{X} \times [0,1])^K} h(x_1, \dots, x_K) m_A (\vec{x}, \vec{r}) d\gamma (\vec{x}, \vec{r}).
\end{equation*}
We can then check that for any continuous and bounded function $g:\mathcal{X}\to\RR$
\begin{align*}
    \sum_{A\in S_K(i)}\int_{\mathcal{X}^K}g(x_i) d\pi_A(x_1, \ldots, x_K)&=\int_{(\mathcal{X} \times [0,1])^K}r_i g(x_i) d\gamma (\vec{x}, \vec{r})\\
    &=\int_{\mathcal{X}}g(x_i)\mu_i(x_i),
\end{align*}
where we have used the fact that $\sum_{A\in S_K(i)} m_A (\vec{x}, \vec{r}) = r_i$ in the first equality.
Thus, our construction gives us a feasible solution to \eqref{eq:multimarginal_decomposed}. Evaluating the objective in \eqref{eq:multimarginal_decomposed} we see that 
\begin{align*}
    &\sum_{A\in S_K}\int_{\mathcal{X}^K} (c_A(x_1, \dots, x_K) +1) d\pi_A(x_1, \dots, x_K)\\
    =&\int_{(\mathcal{X} \times [0,1])^K}\sum_{A\in S_K} m_A (\vec{x}, \vec{r}) (c_A(x_1, \dots, x_K) + 1) d\gamma (\vec{x}, \vec{r})\\
    =&\int_{(\mathcal{X} \times [0,1])^K}\widetilde{c} (\vec{x}, \vec{r}) d\gamma (\vec{x}, \vec{r})
\end{align*}
where the final equality uses the definition of $\widetilde{c}$ and our choice of $m_A (\vec{x}, \vec{r})$. Thus, the two problems have the same optimal value and any feasible solution to one problem can be easily converted into a feasible solution to the other. 
\end{proof}

\subsection{Localization}

In this section we show that the cost function $\widetilde{c}$ in problem \eqref{eq:full_multimarginal} is equal to $B^*_{\widehat{\mu}}$ for a measure $\widehat{\mu}$ that depends on the arguments of $\widetilde{c}$. This result can be interpreted as a localization property for problem \eqref{eq:generalized_barycenter} (and hence for problem \eqref{eqn:DualAdversarial} as well). Compare with the discussion after \textbf{Theorem} \ref{thm:Main}.

\begin{lemma}
\label{lemma:Localization}
Let $\tilde x_1, \dots, \tilde x_k \in \mathcal{X}$, and let $0 \leq \tilde r_1, \dots, \tilde r_k \leq 1$. Then $\widetilde{c}( (\tilde x_1, \tilde r_1), \dots, (\tilde x_K, \tilde r_K))$ defined in \eqref{eq:cost_full_multimarginal} is equal to  $B^*_{\widehat{\mu}}$, where 
\[ \widehat{\mu} := \sum_{i \in [K]} \tilde r_i \delta_{(\tilde x_i,i)}.  \]
\end{lemma}

\begin{proof}
To prove this claim we first notice that by \textbf{Proposition} \ref{prop:FirstMOT} $B^*_{\widehat{\mu}}$ is equal to
\[  
\inf_{\gamma} \int_{(\mathcal{X} \times [0,1])^K}  \widetilde{c}(\vec{x}, \vec{r}) d\gamma (\vec{x}, \vec{r}), 
\]
where $\gamma$ is in the constraint set of problem \eqref{eq:full_multimarginal}. For a feasible $\gamma$, notice that $\gamma$ must concentrate on the set $\{(\vec x , \vec r) : x_i = \tilde x_i,  i \in [K]  \}$. Applying the disintegration theorem to $\gamma$, we can rewrite the objective function evaluated at $\gamma$ as
\[   \int_{[0,1]^K} \widetilde{c} ((\tilde x_1, r_1), \dots, (\tilde x_K , r_K) ) d \gamma_{r}(r_1, \dots, r_K), \]
where $\gamma_r$ is a positive measure over $[0,1]^K$ satisfying the constraints:
\begin{equation}
\label{eq:ConstraintsAux}
  \int_{[0,1]} r_i d\gamma_r(r_1, \dots, r_K) = \tilde r_i , \quad \forall i =1, \dots, K.  
\end{equation}
It is clear that the map associating a feasible $\gamma$ to a $\gamma_r$ satisfying \eqref{eq:ConstraintsAux} is onto, and thus, we can rewrite $B^*_{\widehat{\mu}}$ as
\begin{align*}
\begin{split}
 B^*_{\widehat{\mu}} &=  \inf_{\gamma_r }  \int_{[0,1]^K} \widetilde{c} ((\tilde x_1, r_1), \dots, (\tilde x_K, r_K)) d\gamma_r(r_1, \dots, r_K)
   \\ & = \inf_{\gamma_r}   \int_{[0,1]^K} \inf_{ \{ m_A \}_A \in G(r_1, \dots, r_K) } \left\{   \sum_{A \in S_K} m_A(1 + c_A(\tilde x_1, \dots, \tilde x_K))  \right\} d\gamma_r(r_1, \dots, r_K)  
  \\&= \inf_{\gamma_r} \inf_{\{ m_A \}_A \in G }  \int_{[0,1]^K}  \left\{   \sum_{A \in S_K} m_A(r_1, \dots, r_K) \cdot (1 + c_A(\tilde x_1, \dots, \tilde x_K))  \right\} d\gamma_r(r_1, \dots, r_K) 
   \\&=  \inf_{\{ m_A \}_A \in G }  \inf_{\gamma_r} \int_{[0,1]^K}  \left\{   \sum_{A \in S_K} m_A(r_1, \dots, r_K) \cdot (1 + c_A(\tilde x_1, \dots, \tilde x_K))  \right\} d\gamma_r(r_1, \dots, r_K).
  \end{split}
\end{align*}
In the above, the set $G(r_1, \dots, r_K)$ is the set of $\{ m_A\}_{A \in S_K}$ satisfying the constraints in \eqref{eq:cost_full_multimarginal} for the specific tuple $\big( (\tilde x_1, r_1), \dots, (\tilde x_K, r_K) \big)$, while $G$ is the set of $\{ m_A \}_{A}$ where each $m_A$ is a functions with inputs $r_1, \dots, r_K$ satisfying $\{ m_A(r_1, \dots, r_K) \}_A \in G(r_1, \dots, r_K) $.

We can now write the term
\begin{align*}
    &\int_{[0,1]^K}  \left\{   \sum_{A \in S_K} m_A(r_1, \dots, r_k) \cdot (1 + c_A(\tilde x_1, \dots, \tilde x_k))  \right\} d\gamma_r(r_1, \dots, r_K)\\
    =&\sum_{A \in S_K} m_{A, \gamma} (1+ c_A (\tilde x_1, \dots, \tilde x_k)),
\end{align*}
where we define
\[ m_{A, \gamma_r} := \int  m_A(r_1, \dots, r_k)  d\gamma_r(r_1, \dots, r_K).    \]
Notice that
\begin{align*}
\begin{split}
 \sum_{A \in S_K(i)} m_{A, \gamma_r} & = \sum_{A \in S_K(i)} \int_{[0,1]^K}  m_A(r_1, \dots, r_k)  d\gamma_r(r_1, \dots, r_K) 
 \\& = \int_{[0,1]^K}   \left(\sum_{A \in S_K(i)}m_A(r_1, \dots, r_k) \right)  d\gamma_r(r_1, \dots, r_K) 
 \\& = \int_{[0,1]^K} r_i d\gamma_r(r_1, \dots, r_K)
 \\& = \tilde r_i. 
 \end{split}
 \end{align*}
Conversely, notice that given a collection of functions $\tilde m_A$ satisfying the constraint in \eqref{eq:cost_full_multimarginal} for the tuple $(\tilde x_1, \tilde r_1), \dots,(\tilde x_K, \tilde r_K)$, it is straightforward to find $\gamma_r$ such that $\tilde m_A = m_{A, \gamma_r}$ for all $A$. It now follows that
\[
B^*_{\widehat{\mu}} = \inf_{\tilde m_A } \sum_{A} \tilde m_A(1+ c_A(\tilde x_1, \dots, \tilde x_k)) = \widetilde{c}((\tilde x_1, \tilde r_1), \dots,(\tilde x_K, \tilde r_K)), \]
as we wanted to prove.
\end{proof}

\subsection{Dual Problems}

In this section we discuss the dual problems of the different formulations of the generalized barycenter problem studied in section \ref{eq:lambda_decomposed_formulation}.

\begin{proposition}
\label{prop:Duals}
The dual problems to \eqref{eq:generalized_barycenter}, \eqref{eq:multimarginal_decomposed}, and \eqref{eq:full_multimarginal} can be written as
\begin{equation}\label{eq:barycenter_dual}
\begin{aligned}
    &\sup_{f_1, \ldots, f_K\in C_b(\mathcal{X})} \sum_{i \in [K]} \int_{\mathcal{X}} f_i^c(x_i) d\mu_i(x_i)\\ 
    &\textup{s.t. }  f_i(x)\geq 0, \; \sum_{i \in [K]} f_i(x)\leq 1, \textup{ for all } x \in \mathcal{X}, i\in [K],
\end{aligned}
\end{equation}
\begin{equation}\label{eq:mot_decomposed_dual}
\begin{aligned}
    &\sup_{g_1, \ldots, g_K\in C_b(\mathcal{X})} \sum_{i \in [K]} \int_{\mathcal{X}} g_i(x_i) d\mu_i(x_i)\\
    &\textup{s.t. }  \sum_{i\in A} g_i(x_i)\leq 1+c_A(x_1,\ldots, x_K)  \textup{ for all }  (x_1,\ldots, x_K)\in \mathcal{X}^K, A\in S_K,
\end{aligned}
\end{equation}
and
\begin{equation}\label{eq:mot_dual}
\begin{aligned}
    &\sup_{h_1, \ldots, h_K\in C_b(\mathcal{X})} \sum_{i \in [K]} \int_{\mathcal{X}} h_i(x_i) d\mu_i(x_i)\\
    &\textup{s.t. }  \sum_{i \in [K]} r_ih_i(x_i)\leq \widetilde{c} (\vec{x}, \vec{r})  \textup{ for all }  (\vec{x}, \vec{r}) \in (\mathcal{X} \times [0,1])^K, 
\end{aligned}
\end{equation}
respectively.

Let $f_1,\ldots, f_K$; $g_1,\ldots, g_K$;  $h_1,\ldots, h_K$ be feasible solutions to problems \eqref{eq:barycenter_dual}, \eqref{eq:mot_decomposed_dual}, and \eqref{eq:mot_dual} respectively. Problems \eqref{eq:mot_decomposed_dual} and \eqref{eq:mot_dual} have the same feasible set and hence are identical. Furthermore, $g_i':=f_i^c$ is a feasible solution to \eqref{eq:mot_decomposed_dual} and $f_i'=\max \{ g_i, 0 \}^{\bar{c}}$ is a feasible solution to \eqref{eq:barycenter_dual}, hence  the optimization of \eqref{eq:mot_decomposed_dual} can be restricted to nonnegative $g_i$ that satisfy $g_i=g_i^{\bar{c}c}$.  In particular, \eqref{eq:barycenter_dual}, \eqref{eq:mot_decomposed_dual}, and \eqref{eq:mot_dual} all have the same optimal value.
\end{proposition}

\begin{proof}
The derivation of the dual problems is standard.

To see the equivalence between problems \eqref{eq:mot_decomposed_dual} and \eqref{eq:mot_dual}, fix some $h_1, \ldots, h_K$ that are feasible for \eqref{eq:mot_dual} and choose some $B\in S_K$ and $(x_1,\ldots, x_K)\in \mathcal{X}^K$ such that $c_B(x_1,\ldots, x_K)<\infty.$  Choose
\[
m^*\in \argmin_{m:S_K\to\RR} \sum_{A\in S_K} m_A(1+c_A(x_1,\ldots, x_K))\quad \textup{s.t.} \sum_{A\in S_K(i)} m_A=\chi_B(i),
\]
where $\chi_B(i)=1$ if $i\in B$ and zero otherwise. Note that the choice $m_A=1$ if $A=B$ and $m_A=0$ otherwise is feasible for the above optimization.
Therefore, the optimality of $m^*$ implies that
\begin{align*}
    1+c_B(x_1,\ldots, x_K) &\geq \sum_{A\in S_K} m_A^*(1+c_A(x_1,\ldots, x_K))\\
    &=\widetilde{c}((x_1,\chi_B(1)), \ldots, (x_k,\chi_B(k))\big)\\
    &\geq \sum_{i \in [K]} r_ih_i(x_i)\\
    &=\sum_{i\in B} h_i(x_i).
\end{align*}
Thus, we see that the $h_i$ are feasible for \eqref{eq:mot_decomposed_dual} since $B$ and $(x_1, \ldots, x_K)$ were arbitrary.

Conversely, fix some $g_1, \ldots, g_K$ that are feasible for \eqref{eq:mot_decomposed_dual} and some $(\vec{x}, \vec{r}) \in (\mathcal{X} \times [0,1])^K$. Choose
\[
n^*\in \argmin_{m:S_K\to\RR} \sum_{A\in S_K} m_A(1+c_A(x_1,\ldots, x_K))\quad \textup{s.t.} \sum_{A\in S_K(i)} m_A=r_i,
\]
and observe that
\begin{align*}
    \sum_{i \in [K]} r_i g_i(x_i) &=\sum_{i \in [K]}g_i(x_i)\sum_{A\in S_K(i)}n^*_A\\
    &=\sum_{A\in S_K} n_A^*\sum_{i\in A} g_i(x_i)\\
    &\leq \sum_{A\in S_K} n_A^*(1+c_A(x_1,\ldots, x_K))\\
    &=\widetilde{c}((x_1,r_1), \ldots, (x_K, r_K)),
\end{align*}
where we used the feasibility of the $g_i$. Thus, the $g_i$ are feasible for \eqref{eq:mot_dual}.  Since both problems are optimizing the same functional over the same constraint set, we see that \eqref{eq:mot_decomposed_dual} and \eqref{eq:mot_dual} are identical.

Now suppose that $f_1,\ldots, f_K$ and $g_1, \ldots, g_K$ are feasible solutions to problems \eqref{eq:barycenter_dual} and \eqref{eq:mot_decomposed_dual} respectively and define $g_i'=f_i^c$ and $f_i'=\max \{ g_i, 0 \}^{\bar{c}}$. 
Given $A\in S_K$, $x_1,\ldots, x_K\in \mathcal{X}^K$, and $r>0$ we can choose $x_{r}$ such that
\[
\sum_{i\in A} c(x_{r}, x_i)\leq r+c_A(x_1,\ldots, x_K).
\]
Then we see that 
\[
\sum_{i\in A} g_i'(x_i)\leq \sum_{i\in A} f(x_{r})+c(x_{r},x_i)\leq r+1+c_A(x_1, \ldots, x_K). 
\]
Letting $r\to 0$, we see that the $g_i'$ are feasible for \eqref{eq:mot_decomposed_dual}.  Hence, the optimal value of (\ref{eq:mot_decomposed_dual}) cannot lie strictly below the optimal value of (\ref{eq:barycenter_dual}).

It remains to verify the feasibility of the $f_i'$.  We begin by showing that if $g_1, \ldots, g_K$ are feasible for \eqref{eq:mot_decomposed_dual} then $\max \{ g_1, 0 \}, \ldots, \max\{ g_K, 0 \}$ are also feasible. 
Fix $A\in S_K$ and $(x_1,\ldots, x_K)\in \mathcal{X}^K$.  Let $A'=\{i\in A: g_i(x_i)>0\}$.  We then see that 
\[
\sum_{i\in A} \max \{ g_i(x_i), 0 \}=\sum_{i\in A'} g_i(x_i)\leq 1+c_{A'}(x_1,\ldots, x_K)\leq 1+c_A(x_1,\ldots, x_K)
\]
where the final inequality follows from the definition of $c_A$ and the fact that $A'\subseteq A$.  Now we are ready to verify the feasibility of the $f_i'$.  Clearly $f_i'(x)\geq 0$ since $c(x,x)=0$ for all $x \in \mathcal{X}$.
 Given $x \in \mathcal{X}$, fix $r>0$ and for each $i\in [K]$, choose $x_{i,r}\in X$ such that 
\[
(\max \{ g_i, 0 \})^{\bar{c}}(x)\leq \max(g_i(x_{i,r}),0)-c(x_{i,r},x)+r.
\]
We then have 
\begin{align*}
    \sum_{i \in [K]} \max \{ g_i, 0 \}^{\bar{c}}(x) &\leq \sum_{i \in [K]} \max \{ g_i(x_{i,r}),0 \}-c(x_{i,r},x)+r\\
    &\leq 1+r+c_{[K]}(x_{1,r},\ldots, x_{k,r})-\sum_{i \in [K]}c(x_{i,r},x),
\end{align*}
where the final inequality follows from the feasibility of $\max \{ g_i, 0 \}$. Now from the definition of $c_{[K]}$, the last line is bounded above by $1+r$. Sending $r\to 0$ we are done. 

Notice that the above arguments prove that whenever $g_1,\ldots, g_K$ are feasible for \eqref{eq:mot_decomposed_dual}, then $\max\{ g_1, 0 \}^{\bar{c}c}, \ldots, \max\{ g_K, 0 \}^{\bar{c}c}$ are also feasible for \eqref{eq:mot_decomposed_dual}. Since $u\leq u^{\bar{c}c}$ for any function $u:\mathcal{X}\to\RR$, it follows that 
\begin{equation*}
    \sum_{i \in [K]}\int_{\mathcal{X}} g_i(x)d\mu_i(x)\leq \sum_{i \in [K]}\int_{\mathcal{X}} \max\{ g_i, 0\}^{\bar{c}c}(x) d\mu_i(x).
\end{equation*}
Since we showed that $\max \{ g_i, 0 \}^{\bar{c}}$ was feasible for (\ref{eq:barycenter_dual}), it follows that (\ref{eq:mot_decomposed_dual}) cannot attain a larger value than (\ref{eq:barycenter_dual}).  Hence, we have shown that (\ref{eq:mot_decomposed_dual}) and (\ref{eq:barycenter_dual}) have the same optimal value.
\end{proof}

We now want to show that the dual problems attain the same values as the original primal problems. We begin with a minimax lemma for the following partial optimal transport problem.

\begin{lemma}\label{lem:partial_ot_minimax}
Suppose that $c$ is a bounded Lipschitz cost that satisfies the hypotheses of \textbf{Proposition} \ref{prop:existence}. If $\mathcal{B}\subset \mathcal{M}(\mathcal{X})$ is a weakly compact and convex set, then given measures $\mu_1,\ldots, \mu_K, \in\mathcal{M}(\mathcal{X})$, let  we have the following minimax formula
\begin{align*}
&\min_{\rho, \nu_i\in\mathcal{B}, \nu_i\leq \rho} \sum_{i \in [K]}C(\mu_i, \nu_i)\\
= &\max_{\vp_i, \psi_i\in C_b(\mathcal{X})} \min_{\rho\in \mathcal{B}} \sum_{i \in [K]}\int_{\mathcal{X}} \vp_i(x)d\mu_i(x)-\psi_i(x')d\rho(x')\\
&\quad \textup{s.t. } \vp_i(x)-\psi_i(x')\leq c(x,x'),  \psi_i(x')\geq 0.
\end{align*}
\end{lemma}

\begin{proof}
Using the dual formulation of optimal transport, we can write
\[
 C(\mu_i,\nu_i)= \sup_{\vp_i, \psi_i\in \Phi_c} J_i(\nu_i, \vp_i, \psi_i) \quad \textup{s.t.} \;\; \vp_i(x)-\psi_i(x')\leq c(x,x').
\]
where 
\[
J_i(\nu_i,\vp_i, \psi_i)=\int_{\mathcal{X}} \vp_i(x)d\mu_i(x)-\psi_i(x)d\nu_i(x),
\]
and $\Phi_c=\{(\vp_i, \psi_i)\in C_b(\mathcal{X})\times C_b(\mathcal{X}): \vp_i(x)-\psi_i(x')\leq c(x,x') \;\; \textup{for all} \;\; x,x'\in \mathcal{X}\}$. 
For each $\vp_i, \psi_i\in C_b(\mathcal{X})$ fixed, the mapping $(\rho, \nu_i)\mapsto J_i(\nu_i, \vp_i, \psi_i)$ is linear and lower semicontinuous with respect to the weak convergence of measures. For any $\rho, \nu_i$ fixed, the mapping $(\vp_i, \psi_i)\mapsto J_i(\nu_i, \vp_i, \psi_i)$ is linear and upper semicontinuous with respect to strong convergence in $C_b(\mathcal{X})$. Since the constraint sets $\nu_i\leq \rho$ and $\Phi_c$ are convex, we are in a situation where Sion's minimax theorem applies. Therefore,
\[
\min_{\rho, \nu_i\in \mathcal{B}, \nu_i\leq \rho} \sup_{\vp_i, \psi_i\in\Phi_c} \sum_{i \in [K]} J_i(\nu_i, \vp_i, \psi_i)= \sup_{\vp_i, \psi_i\in\Phi_c} \min_{\rho, \nu_i\in \mathcal{B}, \nu_i\leq \rho} \sum_{i \in [K]} J_i(\nu_i, \vp_i, \psi_i)
\]
Since
\[
\min_{\nu_i\leq \rho} \sum_{i \in [K]} J_i(\nu_i, \vp_i, \psi_i)=\sum_{i \in [K]} \int_{\mathcal{X}} \vp_i(x)d\mu_i(x)-\max(\psi_i(x'),0)d\rho(x'),
\]
we have 
\[
\min_{\rho, \nu_i\in \mathcal{B}, \nu_i\leq \rho} \sum_{i \in [K]} C(\mu_i, \nu_i)=\sup_{\vp_i, \psi_i\in\Phi_c} \min_{\rho\in \mathcal{B}} \sum_{i \in [K]} \int_{\mathcal{X}} \vp_i(x)d\mu_i(x)-\max(\psi_i(x'),0)d\rho(x').
\]
If we replace $\vp_i$ by $\psi_i^c$ and $\psi_i$ by $\max(\psi_i,0)^{c\bar{c}}$ then the value of the problem can only improve.  Since we assume that $c$ is bounded and Lipschitz, it follows that $\psi_i^{c}$ and $\psi_i^{\bar{c}c}$ are bounded and Lipschitz.  Thus, we can restrict the supremum to a compact subset of $\Phi_c$ where $\psi_i\geq 0$.  Thus, the supremum is actually attained by some pair $(\vp_i^*, \psi_i^*)\in \Phi_c$ with $\psi_i^*\geq 0$, $\vp_i^*=(\psi_i^*)^c$ and $(\psi_i^*)^{c\bar{c}}=\psi_i^*.$   
\end{proof}

Using \textbf{Lemma} \ref{lem:partial_ot_minimax} we can prove that there is no duality gap for bounded and Lipschitz costs. We will then show that there is no duality gap for general costs by approximation.

\begin{proposition}\label{prop:barycenter_lip_duality}
Given measures $\mu_1,\ldots, \mu_K$ and a bounded Lipschitz cost $c$ satisfying the assumptions in \textbf{Proposition} \ref{prop:existence}, suppose that $\lambda, \tilde{\mu}_1,\dots, \tilde{\mu}_K$ are optimal solutions to \eqref{eq:generalized_barycenter}. If $\vp_i^*, \psi_i^*\in C_b(\mathcal{X})$ are the optimal Kantorovich potentials for the partial transport of $\mu_i$ to $\lambda$ (c.f \textbf{Lemma} \ref{lem:partial_ot_minimax}), then $\vp_1^*,\ldots, \vp_K^*$ are optimal solutions to problem \eqref{eq:mot_decomposed_dual}, $\psi_1^*,\ldots, \psi_K^*$ are optimal solutions to \eqref{eq:barycenter_dual}, and the values of \eqref{eq:barycenter_dual}-\eqref{eq:mot_dual} are equal to \eqref{eq:generalized_barycenter}. In other words, there is no duality gap.
\end{proposition}

\begin{proof}
If we fix some convex weakly compact subset $\mathcal{B}\subset \mathcal{M}(\mathcal{X})$ containing $\lambda$, then it follows from \textbf{Lemma} \ref{lem:partial_ot_minimax} and the optimality of $\lambda$ that there exists $\vp_i^*, \psi_i^*$ such that
\begin{equation}\label{eq:minimax_optimization}
\lambda(\mathcal{X})+\sum_{i \in [K]} C(\mu_i, \tilde{\mu}_i)=\min_{\rho\in \mathcal{B}} \rho(\mathcal{X})+\sum_{i \in [K]} \int_{\mathcal{X}} \vp_i^*(x)d\mu_i(x)-\psi_i^*(x')d\rho(x'),
\end{equation}
$\psi_i^*(x')\geq 0$, and $(\vp_i^*)^{\bar{c}}(x')=\psi_i^*(x'), (\psi_i^*)^{c}(x)=\vp_i^*(x)$ for all $1\leq i\leq K$ and $x,x'\in \mathcal{X}$.
If there exists $x'\in \mathcal{X}$ such that $\sum_{i \in [K]} \psi_i^*(x')>1$, then we can make the right hand side of \eqref{eq:minimax_optimization} smaller than the left hand side by choosing $\rho=M\delta_{x'}$ for some sufficiently large value of $M$. Hence, it follows that $\sum_{i \in [K]} \psi_i^*(x)\leq 1$ everywhere. Thus, the $\psi_i^*$ are feasible solutions to problem \eqref{eq:barycenter_dual} and, by \textbf{Proposition} \ref{prop:Duals}  $(\psi_i^*)^c=\vp_i^*$ are feasible solutions to \eqref{eq:mot_decomposed_dual}. Finally, if we choose $\rho=0$, it follows that 
\[
\eqref{eq:generalized_barycenter}=\lambda(\mathcal{X})+\sum_{i \in [K]} C(\mu_i, \tilde{\mu}_i)\leq \sum_{i \in [K]} \vp_i^*(x)d\mu_i(x)\leq \eqref{eq:mot_decomposed_dual}=\eqref{eq:barycenter_dual}\leq \eqref{eq:generalized_barycenter}
\]
where the second last equality follows from \textbf{Proposition} \ref{prop:Duals} and the last inequality holds trivially by duality. Therefore, we can infer that there is no duality gap.
\end{proof}

\begin{proposition} \label{prop:barycenter_duality}
Given measures $\mu_1, \ldots, \mu_K$, if $c$ is a cost that satisfies \textbf{Assumption} \ref{assump:CostStructure}, then problems \eqref{eq:barycenter_dual}-\eqref{eq:mot_dual} all have the same value as \eqref{eq:generalized_barycenter}.
\end{proposition}

\begin{remark}
Note that we do not claim that the supremums in \eqref{eq:barycenter_dual}-\eqref{eq:mot_dual} are attained.
\end{remark}

\begin{proof}
Let $\eta:[0,\infty)\to[0,\infty)$ be a smooth strictly increasing function such that $\eta(x)=x$ for $x\leq 1$ and $\eta(x)\leq 2$ for all $x\in [0,\infty)$.
For each $j\in \ZZ_+$, define
\[
\tilde{c}_j(x,x'):=\inf_{(x_1,x_1')\in \mathcal{X}\times \mathcal{X}} c(x_1,x_1')+jd(x,x_1)+j d(x', x_1),
\]
and $c_j(x,x'):=j\eta(\frac{\tilde{c}_j(x,x')}{j})$. It then follows that $c_j$ is a bounded Lipschitz cost that satisfies the assumptions of \textbf{Proposition} \ref{prop:existence}.  Since $c$ is lower semicontinuous it is straightforward to check that $c_j$ converges to $c$ pointwise everywhere.

Let $\alpha_j$ and $\beta_j$ denote the optimal values of Problems \eqref{eq:generalized_barycenter} and \eqref{eq:mot_decomposed_dual} respectively with cost $c_j$. From \textbf{Proposition} \ref{prop:barycenter_lip_duality} we know that $\alpha_j=\beta_j$. Let $\alpha, \beta$ denote the optimal values of Problems \eqref{eq:generalized_barycenter} and \eqref{eq:mot_decomposed_dual} respectively with the original cost $c$. Since we already know that $\beta\leq \alpha$, our goal is to show that $\alpha
\leq \beta$.

Exploiting the fact that $c_j$ is increasing with respect to $j$, if $g_1^{j_0},\ldots, g_K^{j_0}$ is a feasible solution to \eqref{eq:mot_decomposed_dual} for the cost $c_{j_0}$, then it is also a feasible solution to \eqref{eq:mot_decomposed_dual} for $c$. Therefore, 
$\lim_{j\to\infty}\beta_j\leq \beta$.

On the other hand, let $\lambda^j$ and $\tilde{\mu}_1^j,\ldots, \tilde{\mu}_K^j$ be optimal solutions to \eqref{eq:generalized_barycenter} with the cost $c_j$. Let $\pi_i^j$ be the optimal transport plan between $\mu_i$ and $\tilde{\mu}_i^j$. Arguing as in \textbf{Proposition} \ref{prop:existence}, it follows that $\lambda^j$ and $\pi_i^j$ are tight with respect to $j$. Thus, there exists a subsequence (that we do not relabel) such that $\lambda^j$ converges weakly to some $\lambda$ and $\pi_i^j$ converges weakly to some $\pi_i$. Fix some $j_0$ and note that for all $j\geq j_0$
\[
\alpha_j=\lambda^j(\mathcal{X})+\sum_{i \in [K]} \int_{\mathcal{X}}c_j(x,x')d\pi_i^j(x,x')\geq \lambda^j(\mathcal{X})+\sum_{i \in [K]} \int_{\mathcal{X}}c_{j_0}(x,x')d\pi_i^j(x,x').
\]
Therefore, 
\[
\liminf_{j\to\infty} \alpha_j\geq \lambda(\mathcal{X})+\sum_{i \in [K]} \int_{\mathcal{X}}c_{j_0}(x,x')d\pi_i(x,x').
\]
Taking a supremum over $j_0$, it follows that 
\[
\liminf_{j\to\infty} \alpha_j\geq  \lambda(\mathcal{X})+\sum_{i \in [K]} \int_{\mathcal{X}}c(x,x')d\pi_i(x,x')\geq \alpha.
\]
Thus, $\alpha\leq \liminf_{j\to\infty} \alpha_j= \liminf_{j\to\infty} \beta_j=\beta$.  Thanks to \textbf{Proposition} \ref{prop:Duals}, it follows that \eqref{eq:generalized_barycenter} and \eqref{eq:barycenter_dual}-\eqref{eq:mot_dual}, all have the same optimal value. 
\end{proof}

\section{Proof of Theorem \ref{thm:Main}}
\label{sec:ProofMainTheorem} 
 
In this section, we prove \textbf{Theorem} \ref{thm:Main} and return to the adversarial problem \eqref{Robust problem:Intro}.

\subsection{Theorem \ref{thm:Main}: upper bound}
\label{sec:UpperBound}

First we show that
\[
\frac{1}{2\mu(\Z)}B^*_{\mu}  \leq \inf_{\pi \in \Pi_K(\mu)} \int_{\mathcal{Z}^K_*} \c (z_1, \dots, z_K) d\pi (z_1, \dots, z_K).
\]
To see this, recall that $B^*_{\mu}$
is, according to \textbf{Proposition} \ref{prop:FirstMOT}, equal to
\begin{equation*}
    \inf_{\gamma \in \Upsilon_\mu } \int_{(\mathcal{X} \times [0,1])^K} \widetilde{c} (\vec{x}, \vec{r}) d\gamma (\vec{x}, \vec{r}) \quad \textup{s.t. }  \tilde{\mathcal{P}}_{i\,\#}(r_i\gamma)=\mu_i  \textup{ for all }  i\in [K].
\end{equation*}
Here and in what follows we use $\Upsilon_{\mu}$ to denote the set of positive measures satisfying $\tilde{\mathcal{P}}_{i\,\#}(r_i\gamma)=\mu_i \; \textup{for all} \; i\in \{1,\ldots, K \}$.

Let $\pi \in \Pi_K(\mu)$, and for given $\vec{z}=(z_1, \dots, z_K) \in \Z_*^K$, let $\gamma_{\vec{z}} \in \Upsilon_{\widehat{\mu}_{\vec z}}$ be a solution for problem \eqref{eq:full_multimarginal} (when $\mu = \widehat{\mu}_{\vec z}$). We define a measure $\gamma$ as follows: 
\[ \int_{(\X \times [0,1])^K} h(\vec x , \vec r) d \gamma(\vec x , \vec r) := \int_{\Z_*^K } \left( \int_{(\X \times [0,1])^K} h(\vec x , \vec r) d \gamma_{\vec z}(\vec x , \vec r) \right) d \pi(z_1, \dots, z_K)    \]
for every test function $h: (\X \times [0,1])^K \rightarrow \R$.

We check that $\gamma \in \Upsilon_{\frac{1}{2 \mu(\Z)}\mu}$. Indeed, for any test function $g : \X \rightarrow \R$ we have:
\begin{align*}
    \int_{(\mathcal{X} \times [0,1])^K} r_i g(x_i) d\gamma(\vec x , \vec r) &= \int_{\Z_*^K } \left( \int_{(\mathcal{X} \times [0,1])^K} r_i g(x_i) d \gamma_{\vec z}(\vec x , \vec r) \right) d \pi(z_1, \dots, z_K)\\
    &= \frac{1}{K} \int_{\Z_*^K } \left(  \sum_{j: z_j \not = \mathghost } g(x_j) \mathds{1}_{i_j=i} \right) d \pi(z_1, \dots, z_K)\\
    &= \frac{1}{2 \mu(\Z)} \int_{\mathcal{X}} g(x) d\mu_i(x).
\end{align*}
Let us now compute the cost associated to this $\gamma$:
\begin{align*}
    \int_{(\mathcal{X} \times [0,1])^K} \widetilde{c} (\vec{x}, \vec{r}) d\gamma(\vec x , \vec r) &= \int_{\Z_*^K } \left( \int_{(\mathcal{X} \times [0,1])^K} \widetilde{c}(\vec x , \vec r) d \gamma_{\vec z}(\vec x , \vec r) \right) d \pi(z_1, \dots, z_K) \\
    &= \int_{\Z_*^K } B^*_{\widehat{\mu}_{\vec z}} d \pi(z_1, \dots, z_K) \\
    &= \int_{\mathcal{Z}^K_*} \c(z_1, \dots, z_K) d\pi(z_1, \dots, z_K).
\end{align*}

Combining the above with \textit{Remark} \ref{rem:Homogeneity}, we conclude that
\[ 
\frac{1}{2 \mu(\Z)}B^*_{\mu} = B^*_{\frac{1}{2\mu(\Z)}\mu} =  \inf_{\gamma \in \Upsilon_{\frac{1}{2\mu(\Z) } \mu}} \int_{(\mathcal{X} \times [0,1])^K} \widetilde{c} (\vec x , \vec r) d\gamma (\vec{x}, \vec{r})  \leq \inf_{\pi \in \Pi_K (\mu)} \int_{\mathcal{Z}^K_*} \c (\vec{z}) d\pi (\vec{z}).  
\]

\subsection{Theorem \ref{thm:Main}: lower bound}
Now, it is sufficient to show
\[  
\inf_{\pi \in \Pi_K(\mu)} \int \c (z_1, \dots, z_K) d\pi (z_1, \dots, z_K) \leq \frac{1}{2\mu(\Z)}B^*_{\mu}.
\]

First, observe that for any $\phi \in \Phi$ we have:
\begin{align*}
    &\sum_{j=1}^K \int_{\mathcal{X} \times [K]} \phi_j(z_j) \frac{1}{2 \mu(\mathcal{Z})}d\mu(z_j) +  \frac{1}{2} \sum_{j=1}^K \phi_j(\mathghost) \\
    = &\sum_{i \in [K]} \int_{\mathcal{X}} \Big( \sum_{j=1}^K \phi_j(x_i,i) + \sum_{j=1}^K \phi_j(\mathghost) \Big) \frac{1}{2 \mu(\mathcal{Z})}d\mu_i(x_i).
\end{align*}
For each $i \in [K]$, define
\begin{equation*}
    \psi_i(x_i):= \sum_{j=1}^K \phi_j(x_i,i) + \sum_{j=1}^K \phi_j(\mathghost).
\end{equation*}
It is thus clear from the above computation and definition that 
\begin{equation}
    \sum_{j=1}^K \int_{\mathcal{X} \times [K]} \phi_j(z_j) \frac{1}{2 \mu(\mathcal{Z})}d\mu(z_j) + \frac{1}{2}\sum_{j=1}^K \phi_j(\mathghost)  =  \sum_{i \in [K]} \int_{\mathcal{X}} \psi_i(x_i) \frac{1}{2 \mu(\mathcal{Z})}d\mu_i(x_i).
    \label{eq:EqualCostDual}
\end{equation}
Our goal is now to show that $\{ \psi_i : i \in [K]\}$ is feasible for problem  \eqref{eq:mot_decomposed_dual} (working with the normalized measure $\frac{1}{2 \mu(\Z)}\mu$). We start with a preliminary lemma and an example illustrating the strategy behind the proof of this fact. The precise statement appears in \textbf{Proposition} \ref{Prop:LowerBound} below.

\begin{lemma}\label{lem : bound_dual_potential}
Given $(z_1, \dots, z_K) \in \mathcal{Z}_*^K$, let $A = \{ j \in [K] : z_j \neq \mathghost \}$. Suppose that for each $j \in A$ $z_j = (x_j, j)$. Then, for each $\phi \in \Phi$,
\begin{equation}\label{eq : bound_dual_potential}
    \sum_{j=1}^K \phi_j(z_j) \leq \frac{1}{K} + \frac{1}{K} c_A.
\end{equation}
\end{lemma}

\begin{proof}
Since $\phi \in \Phi$, it suffices to show that
\begin{equation*}
    B^*_{\widehat{\mu}_{\vec{z}}} \leq \frac{1}{K} + \frac{1}{K} c_A,
\end{equation*}
where
\begin{equation*}
   \widehat{\mu}_{\vec{z}} = \sum_{l \text{ s.t. } z_l \not = \mathghost }^K \frac{1}{K} \delta_{z_l} = \sum_{j \in A} \frac{1}{K} \delta_{z_j} = \sum_{j \in A} \frac{1}{K} \delta_{(x_j,j)}.
\end{equation*}
For simplicity, assume that $A = \{1, \dots, S\}$. By \textbf{Lemma} \ref{lemma:Localization},
\begin{equation*}
    B^*_{\widehat{\mu}_{\vec{z}}} = \widetilde{c} ( (x_1, \frac{1}{K}), \dots, (x_S, \frac{1}{K}), (x_{S+1}, 0), \dots, (x_K, 0)),
\end{equation*}
where we can pick $x_{S+1}, \dots, x_K$ arbitrarily. Let $m_A = \frac{1}{K}$ and $m_{A'}=0$ for $A' \neq A$. It is easy to check that such $m$ is feasible for \eqref{eq:cost_full_multimarginal} since $r_s = \frac{1}{K}$ for $1 \leq s \leq S$ and $r_j = 0$ for $j \notin A$. So, \eqref{eq:cost_full_multimarginal} implies
\begin{equation*}
    \widetilde{c} ( (x_1, \frac{1}{K}), \dots, (x_S, \frac{1}{K}), (x_{S+1}, 0), \dots, (x_K, 0)) \leq \frac{1}{K} + \frac{1}{K} c_A.
\end{equation*}
The conclusion follows.
\end{proof}

We now present specific examples which illustrate why $\{ \psi_i : i \in [K]\}$ is feasible for \eqref{eq:mot_decomposed_dual}, that is, we need to show that for any $(x_1, \dots, x_K) \in \X^K$ and for any $A \in S_K$ we have \begin{equation*}
    \sum_{i \in A} \psi_i(x_i) \leq 1 + c_A.
\end{equation*}

Let $K=4$ and suppose that $A=\{1, 2, 3\}$. 
Expanding the $\psi_i$'s we get:
\begin{equation*}
    \psi_1(x_1) + \psi_2(x_2) + \psi_3(x_3) = \sum_{i \in [3]} \sum_{j=1}^4 \phi_j(x_i, i) + 3\sum_{j=1}^4 \phi_j(\mathghost),
\end{equation*}
or, after a rearrangement of the summands:
\begin{align*}
    &\phi_1(x_1, 1) + \phi_2(x_2, 2) + \phi_3(x_3, 3) + \phi_4(\mathghost)\\
    + &\phi_2(x_1, 1) + \phi_3(x_2, 2) + \phi_4(x_3, 3) + \phi_1(\mathghost)\\
    + &\phi_3(x_1, 1) + \phi_4(x_2, 2) + \phi_1(x_3, 3) + \phi_2(\mathghost)\\
    + &\phi_4(x_1, 1) + \phi_1(x_2, 2) + \phi_2(x_3, 3) + \phi_3(\mathghost)\\
    + &2\sum_{j=1}^4 \phi_j(\mathghost).
\end{align*}
We can bound the first line above using \eqref{eq : bound_dual_potential}:
\begin{equation*}
    \phi_1(x_1, 1) + \phi_2(x_2, 2) + \phi_3(x_3, 3) + \phi_4(\mathghost) \leq \frac{1}{4} + \frac{1}{4}c_A.
\end{equation*}
The same argument holds for the second, third and fourth lines. For the last line, notice that $\c(\mathghost, \dots, \mathghost)=0$. Hence, the last line is bounded above by $0$ and we can now deduce that
\begin{equation*}
    \psi_1(x_1) + \psi_2(x_2) + \psi_3(x_3) \leq 1 + c_A.
\end{equation*}

The above situation becomes less trivial if $|A|$ is much smaller than $K$. To illustrate, let $K=9$ and suppose that $A = \{1,2\}$. Rearranging the $\phi_j$'s as above we will not be able to obtain the desired upper bound since the total number of $\phi_j(\mathghost)$'s available is in this case $K |A| = 18$ while the required number of $\phi_j(\mathghost)$'s in the analogous arrangement as above would be at least $K(K-|A|) = 63$. To overcome this problem, we need to rearrange the $\phi_j$'s further in order to reduce the required number of $\phi_j(\mathghost)$'s and deduce from this refined analysis the desired upper bound.

First of all, construct a $9 \times 9$ arrangement in the following way: for the $k$-th row in the arrangement, let the $k$-th and the $(k+1)$-th elements be $\phi_k(x_1, 1)$ and $\phi_{k+1}(x_2,2)$, respectively, and let the remaining elements be ``empty". Note that here $k$ and $k+1$ are considered modulo $9$; for example, $10 \equiv 1 \mod 9$, and an empty element means literally no element. We merge rows in the following way: merge together the $1$-st, the $3$-rd, the $5$-th and the $7$-th rows, i.e. replace empty elements for none-empty ones coming from other rows; likewise, merge together the $2$-nd, the $4$-th, the $6$-th and the $8$-th rows; finally, keep the $9$-th row as is. By the above construction, the $1$-st, the $3$-rd, the $5$-th and the $7$-th rows share no common $\phi_j$. Let $\emptyset_j$ denote an empty element at the $j$-th coordinate. The resulting arrangement can be written as:
\begin{align*}
   & \phi_1(x_1, 1), \phi_2(x_2, 2), \phi_3(x_1, 1), \phi_4(x_2, 2), \phi_5(x_1,1), \phi_6(x_2,2), \phi_7(x_1, 1), \phi_8(x_2, 2), \emptyset_9,\\
   & \emptyset_1, \phi_2(x_1, 1), \phi_3(x_2, 2), \phi_4(x_1, 1),\phi_5(x_2, 2), \phi_6(x_1,1), \phi_7(x_2,2), \phi_8(x_1, 1), \phi_9(x_2, 2),\\
   &\phi_1(x_2, 2), \emptyset_2, \emptyset_3, \emptyset_4, \emptyset_5, \emptyset_6, \emptyset_7, \emptyset_8, \phi_9(x_1, 1),
\end{align*}
with the first row representing the merge of rows 1-3-5-7, the second row representing the merge of rows 2-4-6-8, and the last row representing row 9.

Notice that the above arrangement contains all $\phi_j(x_s, s)$'s. Furthermore, the number of $\emptyset_j$ for each $1 \leq j \leq 9$ is exactly $1$. Filling $\emptyset_j$'s with $\phi_j(\mathghost)$'s, and using the fact that the number of $\phi_j(\mathghost)$'s for each $1 \leq j \leq 9$ is $2$, it follows that
\begin{align*}
    \psi_1(x_1) + \psi_2(x_2) &= \sum_{j=1}^4 \big( \phi_{2j-1}(x_1,1) + \phi_{2j}(x_2,2)\big) + \phi_9(\mathghost)\\
    &\quad+  \phi_1(\mathghost) + \sum_{j=1}^4 \big( \phi_{2j}(x_1,1) + \phi_{2j+1}(x_2,2) \big)\\
    &\quad+ \phi_1(x_2,2) + \sum_{j=2}^8 \phi_j(\mathghost) + \phi_9(x_1,1)\\
    &\quad+ \sum_{j=1}^9 \phi_j(\mathghost).
\end{align*}
Observe that for $(z_1, \dots, z_K) = ((x_1, 1), (x_2,2), \dots, (x_1, 1), (x_2,2), \mathghost)$, $\widehat{\mu}_{\vec{z}} = \frac{4}{9} \delta_{(x_1,1)} + \frac{4}{9} \delta_{(x_2,2)}$. Factoring out the $4$ (see \textit{Remark} \ref{rem:Homogeneity}) and applying \eqref{eq : bound_dual_potential}, what we obtain is
\begin{equation*}
    \sum_{j=1}^4 \big( \phi_{2j-1}(x_1,1) + \phi_{2j}(x_2,2)\big) + \phi_9(\mathghost) \leq B^*_{\widehat{\mu}_{\vec{z}}} \leq \frac{4}{9} + \frac{4}{9}c_A.
\end{equation*}
Similarly, the second and third lines can be bounded by $\frac{4}{9} + \frac{4}{9}c_A$ and $\frac{1}{9} + \frac{1}{9}c_A$, respectively. Since $\sum_{j=1}^9 \phi_j(\mathghost) \leq 0$, it follows that
\begin{equation*}
    \psi_1(x_1) + \psi_2(x_2) \leq 1 + c_A.
\end{equation*}

The above two situations help us illustrate the general strategy for proving that the resulting $\psi_i$ are feasible: the idea is to arrange summands appropriately so that we can utilize \textbf{Lemma} \ref{lem : bound_dual_potential} in the most effective way possible. In the following proposition we state precisely our aim and prove it by such strategy.

\begin{proposition}
\label{Prop:LowerBound}
Let $(\phi_1, \dots, \phi_K) \in \Phi$ be a feasible dual potential. For each $i \in [K]$, define
\begin{equation*}
    \psi_i(x_i):= \sum_{j=1}^K \phi_j(x_i,i) + \sum_{j=1}^K \phi_j(\mathghost), \quad x_i \in \X.
\end{equation*}
Then $\{ \psi_i : i \in [K]\}$ is feasible for \eqref{eq:mot_decomposed_dual}. 
\end{proposition}

\begin{proof}
Fix $K$ and $A \in S_K$. Without loss of generality, assume that $A=\{1, \dots, S\}$. We need to show that
\begin{equation}\label{ineq : lower_bound of MOT}
    \sum_{i \in A} \psi_i(x_i) \leq 1 + c_A.
\end{equation}

First, suppose $K$ is divisible by $S$. For each $1 \leq s \leq S$ and $1 \leq j \leq K$, let 
\[ u(s,j):= \begin{cases} (s+j-1 \mod S) & \text{ if } s+j-1  \not = 0 \mod S\\ S & \text{ if }  s+j-1  =0 \mod S. \end{cases}\] Rearranging the sum of the $\psi$'s, it follows that
\begin{align*}
    \sum_{i \in A} \psi_i(x_i) &= \sum_{j=1}^K \sum_{s=1}^S \phi_j(x_s, s) + S \sum_{j=1}^K \phi_j(\mathghost)\\
    &= \sum_{s=1}^S \sum_{j=1}^K \phi_j(x_{u(s,j)}, u(s,j)) + S \sum_{j=1}^K \phi_j(\mathghost).
\end{align*}
Note that for each $1 \leq s \leq S$, $|\{ u(s,j) : 1 \leq j \leq K\} |=\frac{K}{S}$, and hence
\begin{equation*}
    \widehat{\mu}_{\vec{z}} = \sum_{u(s,j)=1}^S \frac{\frac{K}{S}}{K} \delta_{(x_{u(s,j)}, u(s,j))}.
\end{equation*}
Factoring out $\frac{K}{S}$ and applying \eqref{eq : bound_dual_potential}, 
\begin{equation*}
    \sum_{j=1}^K \phi_j(x_{u(s,j)}, u(s,j)) \leq  \frac{K}{S} \big(\frac{1}{K} + \frac{1}{K} c_A \big) = \frac{1}{S} + \frac{1}{S}c_A.
\end{equation*}
Since $\sum_{j=1}^K \phi_j(\mathghost) \leq 0$, it is deduced that 
\begin{align*}
    \sum_{i \in A} \psi_i(x_i) &= \sum_{s=1}^S \sum_{j=1}^K \phi_j(x_{u(s,j)}, u(s,j)) + S \sum_{j=1}^K \phi_j(\mathghost)\\
    & \leq \sum_{s=1}^S \big( \frac{1}{S} + \frac{1}{S}c_A \big)\\
    &=1 + c_A,
\end{align*}
proving the desired inequality in the first case.

Now suppose that $K$ is not divisible by $S$. For each $1 \leq s \leq S$ and each $1 \leq k \leq K$, let
\[ v(s,k):= \begin{cases} (s+k-1 \mod K) & \text{ if } s+k-1  \not =0 \mod K\\ K & \text{ if }  s+k-1  =0 \mod K.\end{cases}\]
Construct a $K \times K$ arrangement in the following way: for each $1 \leq s \leq S$ we set the $v(s,k)$-th element to be $\phi_{v(s,k)}(x_s, s)$, and we set the remaining elements to be empty. We use $\emptyset_j$ to denote an empty element at the $j$-th coordinate. Note that the $k$-th row has $\phi_{v(1,k)}(x_1,1), \dots, \phi_{v(S,k)}(x_S,S)$ as non-empty elements, which are placed from the $v(1,k)$-th coordinate to the $v(S,k)$-th coordinate, while it has $(K-S)$ many empty elements. For example, the $3$-rd row is
\begin{equation*}
    \emptyset_1, \emptyset_2, \phi_3(x_1, 1), \dots, \phi_{S+2}(x_S,S), \emptyset_{S+3}, \dots, \emptyset_K.
\end{equation*}
We split this case into two further subcases.

First, assume that $\lfloor \frac{K}{S} \rfloor=1$. In this case, we have $K(K-S) \leq KS$. For each $1 \leq k \leq K$, collect all the $\phi_j(\mathghost)$'s such that $j \notin A_k:=\{v(1,k), \dots, v(S,k)\}$. Notice that for fixed $j$, the number of $k$'s such that $j \notin A_k$ is exactly $K-S$ since all $\phi_j(x_s, s)$'s are contained in this arrangement and $\lfloor \frac{K}{S} \rfloor=1$. In other words, the total number of $\emptyset_j$ is smaller than the total number of $\phi_j(\mathghost)$. From the above and  an application of \eqref{eq : bound_dual_potential}, we deduce that 
\begin{align*}
    \sum_{i \in A} \psi_i(x_i) &= \sum_{k=1}^K \Big( \sum_{s=1}^{S} \phi_{v(s,k)}(x_s, s) + \sum_{j \notin A_k} \phi_j(\mathghost) \Big) + (2S - K)\sum_{j=1}^K \phi_j(\mathghost)\\
    &\leq \sum_{k=1}^K \big( \frac{1}{K} + \frac{1}{K}c_A \big) \\
    &= 1 + c_A,
\end{align*}
proving the desired inequality in this case.

Finally, assume that $\lfloor \frac{K}{S} \rfloor >1$. Here the idea is to merge $\lfloor \frac{K}{S} \rfloor$-many rows to a single row. We do this in the following way: for each $1\leq s\leq S$, we merge together the $s$-th row, the $(S+s)$-th row, $\dots$, and the $((\lfloor \frac{K}{S} \rfloor -1)S + s)$-th row, to obtain a single row which will be re-indexed by $s$. In the original arrangement, since the $((m-1)S +s)$-th row has $\phi_{v(s, (m-1)S + 1)}(x_1, 1), \dots, \phi_{v(s, mS)}(x_S, S)$ as non-empty elements, the rows that get merged share no common $\phi_j$. We keep the last $(K - \lfloor \frac{K}{S} \rfloor S)$-many rows in the original arrangement the same, and for convenience we let the indices of these rows be unchanged. After this procedure, we obtain $S$-many merged rows and $(K - \lfloor \frac{K}{S} \rfloor S)$-many remaining original rows. Now, it is necessary to count, for every fixed $j$, the total number of empty elements $\emptyset_j$ in this new arrangement. If the number of $\emptyset_j$'s was smaller than or equal to $S$ for all $1 \leq j \leq K$, we would be done since the number of $\phi_j(\mathghost)$ is $S$ for each $j$, whence it would be possible to replace the $\emptyset_j$'s with $\phi_j(\mathghost)$'s. We show that this is indeed the case.

For each merged row, its non-empty elements are
\begin{equation*}
    \phi_{v(s, 1)}(x_1,1), \dots, \phi_{v(s, S)}(x_S,S), \dots, \phi_{v(s, (\lfloor \frac{K}{S} \rfloor-1)S + 1)}(x_1, 1), \dots, \phi_{v(s, \lfloor \frac{K}{S} \rfloor S) }(x_S, S).
\end{equation*}
Observe that for each merged row, the index $j$ of $\emptyset_j$ varies from $v(s, \lfloor \frac{K}{S} \rfloor S + 1)$ to $v(s, K)$. The definition of $v(s, k)$ yields that
\begin{align}
    v(s, \lfloor \frac{K}{S} \rfloor S + 1) &= \lfloor \frac{K}{S} \rfloor S + s \text{ if $1 \leq s \leq K - \lfloor \frac{K}{S} \rfloor S$,} \label{eq : v_starting1}\\
    v(s, \lfloor \frac{K}{S} \rfloor S + 1) &= \lfloor \frac{K}{S} \rfloor S + s - K \text{ if $K - \lfloor \frac{K}{S} \rfloor S + 1 \leq s \leq S$}\label{eq : v_starting2}
\end{align}
and
\begin{align}
    v(s, K) &=K \text{ if $s=1$,} \label{eq : v_ending1}\\
    v(s, K) &= s - 1  \text{ if $2 \leq s \leq S$} \label{eq : v_ending2}.
\end{align} 
To count the total number of $\emptyset_j$'s in the merged rows, let's consider the following sub-cases.
\begin{enumerate}
    \item[(i)] $\lfloor \frac{K}{S} \rfloor S + 1 \leq j \leq K$ : By \eqref{eq : v_starting1}, if $1 \leq s \leq K - \lfloor \frac{K}{S} \rfloor S$, then the $s$-th row has $\emptyset_j$ for $\lfloor \frac{K}{S} \rfloor S + s \leq j \leq K$. Also, by \eqref{eq : v_starting2} and \eqref{eq : v_ending2}, if $K - \lfloor \frac{K}{S} \rfloor S +1 \leq s \leq S$, then no merged row has such $\emptyset_j$. Hence, the number of $\emptyset_j$ is $j - \lfloor \frac{K}{S} \rfloor S$. 
    \item[(ii)] $S \leq j \leq \lfloor \frac{K}{S} \rfloor S$ : It follows from \eqref{eq : v_starting1} and \eqref{eq : v_starting2} that either $v(s, \lfloor \frac{K}{S} \rfloor S + 1) > \lfloor \frac{K}{S} \rfloor S$ or $v(s, \lfloor \frac{K}{S} \rfloor S + 1) < S$. Similarly, it follows from \eqref{eq : v_ending1} and \eqref{eq : v_ending2} that either $v(s, K) > \lfloor \frac{K}{S} \rfloor S$ or $v(s, K) < S$. Since the index $j$ of $\emptyset_j$ of the $s$-th merged row varies from $v(s, \lfloor \frac{K}{S} \rfloor S + 1)$ to $v(s,K)$, the number of $\emptyset_j$ is $0$.
    \item[(iii)] $S - (K - \lfloor \frac{K}{S} \rfloor S) + 1 \leq j \leq S-1$ : By \eqref{eq : v_starting2} and \eqref{eq : v_ending2}, if $S - (K - \lfloor \frac{K}{S} \rfloor S) + 1 \leq j \leq S-1$, then $\emptyset_j$ appears from the $(j+1)$-st merged row to the $S$-th merged row. Hence, the number of $\emptyset_j$ is $S - j$.
    \item[(iv)] $1 \leq j \leq S - (K - \lfloor \frac{K}{S} \rfloor S)$ : Similar to (iii), if $1 \leq j \leq S - (K - \lfloor \frac{K}{S} \rfloor S)$, then $\emptyset_j$ appears from the $(j+1)$-st merged row to the $S$-th merged row. Hence, the number of $\emptyset_j$ is $K - \lfloor \frac{K}{S} \rfloor S$.
\end{enumerate}
To summarize, in the merged rows
\begin{equation}\label{eq : number_merged}
    \text{ the number of } \emptyset_j =\left\{ \begin{array}{ll}
j - \lfloor \frac{K}{S} \rfloor S &\text{ for $\lfloor \frac{K}{S} \rfloor S + 1 \leq j \leq K$},\\
0 &\text{ for $ S \leq j \leq \lfloor \frac{K}{S} \rfloor S$},\\
S -j &\text{ for $S - (K - \lfloor \frac{K}{S} \rfloor S) + 1 \leq j \leq S-1$},\\
K - \lfloor \frac{K}{S} \rfloor S &\text{ for $1 \leq j \leq S - (K - \lfloor \frac{K}{S} \rfloor S)$}.
\end{array} \right.
\end{equation}

Now, it remains to count the total number of $\emptyset_j$ in the last $(K - \lfloor \frac{K}{S} \rfloor S)$-many remaining original rows. In this part, each row has only $S$-many non-empty elements. Recall that we still use the same index $k$ for these remaining rows. Precisely, for $\lfloor \frac{K}{S} \rfloor S + 1 \leq k \leq K$, the $k$-th row has
\begin{equation*}
    \phi_{v(1,k)}(x_1,1), \phi_{v(2, k)}(x_2,2), \dots, \phi_{v(S, k)}(x_S,S).
\end{equation*}
Recall that $A_k := \{v(1,k), \dots, v(S, k)\}$. To count the total number of $\emptyset_j$'s in the original rows, let's consider the following sub-cases.
\begin{enumerate}
    \item[(i)] $\lfloor \frac{K}{S} \rfloor S + 1 \leq j \leq K$. : If $1 \leq j +1 - k \leq S$, by the definition of $v(s,k)$, then $j \in A_k$. In other words, each $k$-th row has $\emptyset_j$ for $k > j$. Hence, the number of $\emptyset_j$ is $K - j$.
    \item[(ii)] $S \leq j \leq \lfloor \frac{K}{S} \rfloor S$ : From the definition of $v(s,k)$ and the range of $k$, we deduce that if $\lfloor \frac{K}{S} \rfloor S + 1 \leq k \leq K$, then $v(1,k) > \lfloor \frac{K}{S} \rfloor S$ and $v(S,k) < S$. In other words, $\emptyset_j$ for $S \leq j \leq \lfloor \frac{K}{S} \rfloor S$ appears in every row. Hence, the number of $\emptyset_j$ is $K - \lfloor \frac{K}{S} \rfloor S$.
    \item[(iii)] $S - (K - \lfloor \frac{K}{S} \rfloor S) + 1 \leq j \leq S-1$ : Since $\lfloor \frac{K}{S} \rfloor S + 1 \leq k \leq K$, if $v(S,k)=S+k - K < j$, then $j \notin A_k$. This yields that if $\lfloor \frac{K}{S} \rfloor S + 1 \leq k \leq K - S + j$, then the $k$-th row has $\emptyset_j$. Hence, the number of $\emptyset_j$ is $K - \lfloor \frac{K}{S} \rfloor S - S + j$.
    \item[(iv)] $1 \leq j \leq S - (K - \lfloor \frac{K}{S} \rfloor S)$ : Since $v(S, \lfloor \frac{K}{S} \rfloor S + 1) = S - (K - \lfloor \frac{K}{S} \rfloor S)$, if $1 \leq j \leq S - (K - \lfloor \frac{K}{S} \rfloor S)$ and $\lfloor \frac{K}{S} \rfloor S + 1 \leq k \leq K$, then  $j \in A_k$. Hence, the number of $\emptyset_j$ is $0$.
\end{enumerate}
To summarize, in the remaining original rows
\begin{equation}\label{eq : number_remaining}
    \text{ the number of } \emptyset_j =\left\{ \begin{array}{ll}
K -j &\text{ for $\lfloor \frac{K}{S} \rfloor S + 1 \leq j \leq K$},\\
K - \lfloor \frac{K}{S} \rfloor S  &\text{ for $ S \leq j \leq \lfloor \frac{K}{S} \rfloor S$},\\
K - \lfloor \frac{K}{S} \rfloor S - S + j &\text{ for $S - (K - \lfloor \frac{K}{S} \rfloor S) + 1 \leq j \leq S-1$},\\
0 &\text{ for $1 \leq j \leq S - (K - \lfloor \frac{K}{S} \rfloor S)$}.
\end{array} \right.
\end{equation}
Combining \eqref{eq : number_merged} with \eqref{eq : number_remaining}, the total number of $\emptyset_j$ is always exactly equal to $K - \lfloor \frac{K}{S} \rfloor S$ which is always less than $S$. This allows us to replace every $\emptyset_j$ with a $\phi_j(\mathghost)$. Accordingly, using $\sum \phi_j(\mathghost) \leq 0$, we deduce that
\begin{align*}
    \sum_{i \in A} \psi_i(x_i) &\leq \sum_{\text{merged rows}} \left( \sum_{v(s,j)} \phi_{v(s,j)}(x_s, s) + \sum_{l \neq v(s,j)}\phi_l(\mathghost) \right)\\
    &\quad+ \sum_{\text{remaining rows}} \left( \sum_{v(s,j)} \phi_{v(s,j)}(x_s, s) + \sum_{l \neq v(s,j)}\phi_l(\mathghost) \right).
\end{align*}
Let's focus on the first summation over merged rows. Notice that there are $\lfloor \frac{K}{S} \rfloor S$ many non-empty elements and the set of arguments of such non-empty elements is $\{ (x_1,1), \dots, (x_S,S) \}$. Thus,
\begin{equation*}
    \widehat{\mu}_{\vec{z}} = \sum_{s=1}^S \frac{\lfloor \frac{K}{S} \rfloor}{K} \delta_{(x_s, s)}.
\end{equation*}
Factoring out $\lfloor \frac{K}{S} \rfloor$ and applying \eqref{eq : bound_dual_potential}, we obtain
\begin{equation*}
   \sum_{v(s,j)} \phi_{v(s,j)}(x_s, s) + \sum_{l \neq v(s,j)}\phi_l(\mathghost) \leq \frac{\lfloor \frac{K}{S} \rfloor}{K} + \frac{\lfloor \frac{K}{S} \rfloor}{K} c_A.
\end{equation*}
On the other hand, for the second summation over remaining rows, there are $S$ many non-empty elements. Thus,
\begin{equation*}
    \widehat{\mu}_{\vec{z}} = \sum_{s=1}^S \frac{1}{K} \delta_{(x_s, s)}.
\end{equation*}
\eqref{eq : bound_dual_potential} immediately implies
\begin{equation*}
    \sum_{v(s,j)} \phi_{v(s,j)}(x_s, s) + \sum_{l \neq v(s,j)}\phi_l(\mathghost) \leq \frac{1}{K} + \frac{1}{K}c_A.
\end{equation*}
Note that the number of merged rows is $S$ and the number of remaining original rows is $K - \lfloor \frac{K}{S} \rfloor S$, respectively. Combining all arguments, we can infer that
\begin{align*}
    \sum_{i \in A} \psi_i(x_i) &\leq \frac{\lfloor \frac{K}{S} \rfloor S}{K} + \frac{\lfloor \frac{K}{S} \rfloor S}{K} c_A + \frac{K - \lfloor \frac{K}{S} \rfloor S}{K} + \frac{K - \lfloor \frac{K}{S} \rfloor S}{K}c_A\\
    &= 1 + c_A,
\end{align*}
obtaining the desired inequality in the last remaining case. 
\end{proof}

In summary, we have proved that for a given $\phi=(\phi_1, \dots, \phi_K) \in \Phi$, its associated $(\psi_1, \dots, \psi_K)$ (which satisfies \eqref{eq:EqualCostDual}) is feasible for \eqref{eq:mot_decomposed_dual}. Consequently, this leads to \begin{equation} \label{eq:DualsIneqs}
    \eqref{eqn:DualMOT} \leq \frac{1}{2 \mu(\Z)} \eqref{eq:mot_decomposed_dual} 
\end{equation}
In turn, by the equivalence between \eqref{eq:mot_decomposed_dual} and \eqref{eq:generalized_barycenter} by \textbf{Proposition} \ref{prop:barycenter_duality}, this automatically implies that
\begin{equation*}
    \eqref{eqn:DualMOT} \leq \frac{1}{2\mu(\Z)} B^*_\mu.
\end{equation*}
Finally, combining with \textbf{Corollary} \ref{prop:MOTNoGap} below (which establishes that under \textbf{Assumption} \ref{assump:CostStructure} there is no duality gap for the MOT problem \eqref{Robust problem:Alternative_form}) we obtain the desired inequality relating the minimum value for the MOT problem and $B_\mu^*$.

\subsection{Returning to the adversarial problem \eqref{Robust problem:Intro}}
\label{sec:MOTProofs}

We begin by establishing that, under \textbf{Assumption} \ref{assump:CostStructure}, the cost $\c$ is lower semi-continuous with respect to a suitable notion of convergence.

\begin{proposition}\label{prop: class power lsc_weakly}
Let $\mathcal{Z}_*=\mathcal{Z} \cup \{\mathghost\}$ on which $\mathghost$ is considered as an isolated point. Let  $\widehat{d}$ be defined according to:
\begin{equation*}
    \widehat{d}(z, z'):= \left\{ \begin{array}{ll}
d(x,x') & \textrm{if $i=i'$,}\\
\infty & \textrm{if $i \neq i'$ or $z=\mathghost$ and $z' \in \mathcal{Z}$(vice-versa),}\\
0 & \textrm{if $z=z'=\mathghost$.}
\end{array} \right.
\end{equation*}
Define $\widehat{d}_K$ on $\mathcal{Z}_*^K$ by
\begin{equation*}
    \widehat{d}_K ( (z_1, \dots, z_K), (z'_1, \dots, z'_K)):= \max_{i \in [K]} \widehat{d}(z_i, z'_i).
\end{equation*}
Recall
\begin{equation*}
    \c(z_1, \dots, z_K):= B^*_{ \widehat{\mu}_{\vec z}}
\end{equation*}
where $\widehat{\mu}_{\vec z}$ is defined as
\begin{equation*}
    \widehat{\mu}_{\vec z} :=  \frac{1}{K} \sum_{l \text{ s.t. } z_l \not = \mathghost }^K \delta_{z_l}.
\end{equation*}
Under \textbf{Assumption} \ref{assump:CostStructure}, $\c$ is lower semi-continuous on $(\mathcal{Z}_*^K, \widehat{d}_K)$.
\end{proposition}

\begin{remark}
Note that $(\mathcal{Z}_*^K, \widehat{d}_K)$ is still a Polish space.
\end{remark}

\begin{proof}
Suppose $\vec{z}^n=(z^n_1, \dots, z^n_K)$ converges to $\vec{z}=(z_1, \dots, z_K)$ in $(\mathcal{Z}_*^K, \widehat{d}_k)$. Without loss of generality, assume that $z_1, \dots, z_L = \mathghost$ for all $1 \leq L \leq K$. If $L=K$, the claim would be trivial and so we can focus on the case $L<K$. By the definition of $\widehat{d}_K$, without loss of generality we can further assume that $z^n_1, \dots, z^n_L = \mathghost$ for all $n$, and likewise, for each $L+1 \leq j \leq K$, we can assume that $i^n_j = i_j$ for all $n$, for otherwise the convergence would not hold due to the definition of $\widehat{d}_K$.

By \textbf{Lemma} \ref{lemma:Localization} we have
\begin{equation}
\label{eq:AuxCostLowerSemi}
\c(z_1^n, , \dots, z_K^n) = B^*_{\widehat{\mu}_{\vec z^n}} =  \inf_{m:S_K\to \RR} \sum_{A\subseteq \{L+1, \dots, K \}} m_A\big(c_A(x_{L+1}^n, \ldots, x^n_K)+1\big),   
\end{equation}
where the min ranges over all $\{m_A\}_{A \subseteq \{ L+1, \dots, K \} }$ such that $ \sum_{A\in S_K(i) \cap \{L+1, \dots, K \}} m_A= \frac{1}{K}, \quad \forall i =L+1, \dots, K$.

We now claim that for every $A \subseteq \{ L+1, \dots, K \}$,
\[ c_A(x_{L+1}, \dots, x_{K}) \leq \liminf_{n \rightarrow \infty} c_A(x^n_{L+1}, \dots, x^n_{K}).   \]
Indeed, if the right hand side is equal to $+\infty$, then there is nothing to prove. If the right hand side is finite, we may then find a sequence $\{ \tilde x^n \}_{n \in \N}$ such that
\[  \liminf_{n \rightarrow \infty}  \sum_{i\in A} c(\tilde x^n,x^n_i) =\liminf_{n \rightarrow \infty} c_A(x^n_{L+1}, \dots, x^n_{K}) <\infty. \]
By the compactness property in \textbf{Assumption} \ref{assump:CostStructure} it follows that up to subsequence (not relabeled) we have that $\{ \tilde x^n \}_{n \in \N}$ converges toward a point $\tilde x\in \X$. Combining with the lower semi-continuity of $c$, we deduce that
\[ c_A(x_{L+1}, \dots, x_K) \leq \sum_{i \in A} c(\tilde x , x_i) \leq \liminf_{n \rightarrow \infty}c_A(x^n_{L+1}, \dots, x^n_{K}), \]
as we wanted to show.

Returning to \eqref{eq:AuxCostLowerSemi}, we can find for each $n \in \N$ a collection of feasible $\{ m^n_{A} \}_{A \subseteq \{ L+1, \dots, K\}}$ such that
\[   \liminf_{n \rightarrow \infty } \sum_{A\subseteq \{L+1, \dots, K \}} m_A^n\big(c_A(x_{L+1}^n, \ldots, x^n_K)+1\big) =  \liminf_{n \rightarrow \infty } \c(z_1^n , \dots , z_K^n).  \]
Using the Heine-Borel theorem in Euclidean space, we can assume without the loss of generality that for every $A$, $m^n_A$ converges to some $m_A$ as $n \rightarrow \infty$. The resulting collection of $m_A$ is feasible for the problem defining $\c(z_1, \dots, z_K)$ and thus, using the lower semicontinuity of $c_A$ established earlier, we deduce:
\[ \c(z_1, \dots, z_K) \leq    \sum_{A\subseteq \{L+1, \dots, K \}} m_A\big(c_A(x_{L+1}^n, \ldots, x^n_K)+1\big) \leq \liminf_{n \rightarrow \infty } \c(z_1^n , \dots , z_K^n). \]
\end{proof}

\begin{corollary}(Duality of MOT)
\label{prop:MOTNoGap}
Under \textbf{Assumption} \ref{assump:CostStructure},
\begin{align*}
    &\inf_{\pi \in \Pi_K(\mu)} \int_{\mathcal{\mathcal{Z}_*}^K} \c(z_1, \dots, z_K) d\pi(z_1, \dots, z_K)\\
    = &\sup_{\phi \in \Phi} \left\{ \sum_{j=1}^K \int_{\mathcal{X} \times [K]} \phi_j(z_j) \frac{1}{2 \mu(\mathcal{Z})}d\mu(z_j) +  \frac{1}{2} \sum_{j=1}^K \phi_j(\mathghost) \right\}.
\end{align*}
Furthermore, a minimizer $\pi^*$ exists, hence the infimum is indeed the minimum.
\end{corollary}

\begin{proof}
From \textbf{Proposition} \ref{prop: class power lsc_weakly} it follows that the cost function $\c(z_1, \dots, z_K)$ is lower semi-continuous on ($\Z_{*}^K, \widehat{d}_K)$, which is a Polish space. Applying \textbf{Theorem} $1.3$ in \cite{MR1964483}, which is stated for the usual optimal transport, but that can be generalized to the MOT setting, we obtain the desired duality. The existence of a minimizer $\pi^*$ follows from the lower semi-continuity of $\c(z_1, \dots, z_K)$ and the compactness of $\Pi_K(\mu)$.
\end{proof}

\begin{corollary}\label{cor : robust=dual}
Under \textbf{Assumption} \ref{assump:CostStructure}, \eqref{def:RobustClassifPower}=\eqref{eqn:DualAdversarial}.
\end{corollary}

\begin{proof}
By the upper bound from section \ref{sec:UpperBound} we have 
\[\frac{1}{2\mu(\Z)}B^*_{\mu}  \leq \min_{\pi \in \Pi_K(\mu)} \int \c (z_1, \dots, z_K) d\pi (z_1, \dots, z_K).\]
On the other hand, from \eqref{eq:DualsIneqs} and \textbf{Corollary} \ref{prop:MOTNoGap} we have 
\[ \min_{\pi \in \Pi_K(\mu)} \int \c (z_1, \dots, z_K) d\pi (z_1, \dots, z_K) = \eqref{eqn:DualMOT}\leq \frac{1}{2 \mu(\Z)} \eqref{eq:mot_decomposed_dual} \leq \frac{1}{2\mu(\Z)} B_{\mu}^*.   \]
Combining these two inequalities we conclude that all the above terms must be equal. In particular, $\eqref{eq:mot_decomposed_dual} = B_\mu^*$. Finally, by \textbf{Proposition} \ref{prop:Duals} we know that $\eqref{eq:mot_decomposed_dual}= \eqref{eq:barycenter_dual}=\eqref{def:RobustClassifPower}$. In particular, $\eqref{eqn:DualAdversarial}= B_\mu^* = \eqref{def:RobustClassifPower}$.
\end{proof}

\begin{corollary}[Correction of Corollary 33 of \cite{trillos2023multimarginal}]
\label{cor:OptimalPair}
Suppose that \textbf{Assumption} \ref{assump:CostStructure} holds and that $(\pi^*, \phi^*)$ is a solution pair for the MOT problem and its dual. Define $f^*$ and $\widetilde{\mu}^*$ according to:
\begin{equation}\label{eq : correct form robust classifier}
    f_i^*(\widetilde{x}):= \max  \left\{ \sup_{x \in \spt(\mu_i)} 
    \left\{ \sum_{j=1}^K \phi^*_j(x,i) + \sum_{j=1}^K \phi^*_j(\mathghost) -c(x, \widetilde{x}) \right\}, 0 \right\}
\end{equation}
and for any test function $h$ on $\mathcal{X}$,
\begin{equation*}
    \int_{\mathcal{X}} h(\widetilde{x}) d\widetilde{\mu}^*_i(\widetilde{x}):= \int_{\mathcal{Z}_*^K} \left\{ \int_{\mathcal{X}} h(\widetilde{x}) d\widetilde{\mu}^*_{\vec{z}, i}(\widetilde{x}) \right\} d\pi^*(\vec{z}),
\end{equation*}
where $\widetilde{\mu}^*_{\vec{z},i}$ is the $i$-th marginal of $\widetilde{\mu}^*_{\vec{z}}$, an optimal adversarial attack which achieves $\c (z_1, \dots, z_K)$ given $\vec{z}=(z_1, \dots, z_K)$. Suppose $f^*$ is measurable. Then $(f^*, \widetilde{\mu}^*)$ is a saddle for problem \eqref{Robust problem:Intro}.
\end{corollary}

\begin{remark}
Here, we do not claim that $f^*$ is in general measurable. However, if either $c$ is continuous or $\mu$ is an empirical measure with a finite support, then $f^*$ can be shown to be measurable. See \textbf{Remark} 5.5 and \textbf{Remark} 5.11 in \cite{MR2459454}.

Notice that the supremum in the definition of $f_i^*$ is only taken over $\spt(\mu_i)$. 
\end{remark}

\begin{remark}
In \cite{trillos2023multimarginal}, we define a robust classifier $f_i$ as
\begin{equation}\label{eq : previous form robust classifier}
    f_i(\widetilde{x}):=\sup_{x \in \spt(\mu_i)} \left\{ \max \left\{ \sum_{j=1}^K \phi^*_j(x,i) + \sum_{j=1}^K \phi^*_j(\mathghost), 0 \right\} -c(x, \widetilde{x}) \right\},
\end{equation}
which is what appears inside the outer max in the definition of $f_i^*$ in \eqref{eq : correct form robust classifier}. This small correction guarantees that $f_i^*$ is non-negative, which was not guaranteed in the original definition. As we see below, taking the max does not affect the fact that $\sum_{i \in [K]} f_i^*(\tilde x) \leq 1$.  


\end{remark}

\begin{proof}
First, we will verify that $f^*$ is feasible for \eqref{eq:barycenter_dual}. Since $f^*_i \geq 0$ by construction, it is sufficient to show that
\begin{equation}\label{eq: feasbility of f^*}
    \sum_{i \in [K]} f^*_i(\widetilde{x}) \leq 1
\end{equation}
for all $\widetilde{x} \in \mathcal{X}$.

By \textbf{Proposition} \ref{Prop:LowerBound}, $\psi=(\psi_1, \dots, \psi_K)$ is feasible for \eqref{eq:mot_decomposed_dual}, where
\begin{equation*}
    \psi_i(x) = \sum_{j=1}^K \phi^*_j(x,i) + \sum_{j=1}^K \phi^*_j(\mathghost).
\end{equation*}
Fix $\widetilde{x}$ and let $A= \{i \in [K] : f^*_i(\widetilde{x}) > 0 \}$. Then
\begin{align*}
    \sum_{i \in [K]} f^*_i(\widetilde{x}) &= \sum_{i \in A } f^*_i(\widetilde{x}) = \sum_{i \in A }  \sup_{x_i \in \spt(\mu_i)} 
    \left\{ \psi_i(x_i) - c(x_i, \widetilde{x}) \right\}.
\end{align*}
Recalling \eqref{eq:cost_A},
\begin{equation*}
    \sum_{i \in A } c(x_i, \widetilde{x}) \geq c_A(x_i : i \in A).
\end{equation*}
Combined with the fact that $\sum_{i \in A } \psi_i(x_i) \leq 1 + c_A(x_i : i \in A)$, it follows that
\begin{equation*}
    \sum_{i \in [K]} f^*_i(\widetilde{x}) \leq \sup_{x_i \in \spt(\mu_i) : i \in A } \left\{ \sum_{i \in A } \psi_i(x_i) - c_A(x_i : i \in A) \right\} \leq 1,
\end{equation*}
which verifies that $f^*$ is feasible for \eqref{eq:barycenter_dual} under the assumption that it is measurable.

Now, we will show that $(f^*, \widetilde{\mu}^*)$ is a saddle point for problem \eqref{eqn:DualAdversarial}. More explicitly, we show that for any $f \in \mathcal{F}$ and for any $\widetilde{\mu}$,
\begin{equation}\label{eq : saddle_point_classification_power}
   B(f, \widetilde{\mu}^*) + C(\mu, \widetilde{\mu}^*) \leq B(f^*, \widetilde{\mu}^*) + C(\mu, \widetilde{\mu}^*) \leq B(f^*, \widetilde{\mu}) + C(\mu, \widetilde{\mu}).
\end{equation}


First we compute $  B(f^*, \widetilde{\mu}^*) + C(\mu, \widetilde{\mu}^*)$. Notice that
\begin{align*}
    B(f^*, \widetilde{\mu}^*) + C(\mu, \widetilde{\mu}^*) &= \sum_{i=1}^K \int_{\mathcal{X}} f^*_i(\widetilde{x}_i) d\widetilde{\mu}^*_i(\widetilde{x}_i) + \sum_{i=1}^K C(\mu_i, \widetilde{\mu}^*_i)\\
    &=\sum_{A \in S_K} \sum_{i \in A} \left\{  \int_{\mathcal{X}} f^*_i(\widetilde{x}_i) d\lambda^*_A(\widetilde{x}_i) + C(\mu_{i,A}, \lambda^*_A) \right\}\\
    &= \sum_{A \in S_K} \left\{ \int_{\mathcal{X}^K} \big( \sum_{i \in A}f^*_i(T_A(\vec{x}))  + c_A(\vec{x}) \big) d\pi^*_A(\vec{x}) \right\}\\
    &\leq \sum_{A \in S_K} \left\{ \int_{\mathcal{X}^K} \big( 1  + c_A(\vec{x}) \big) d\pi^*_A(\vec{x}) \right\},
\end{align*}
where $\lambda^*_A$ and $\pi^*_A$ correspond to $\widetilde{\mu}^*$. The last inequality follows from \eqref{eq: feasbility of f^*}. Hence,
\begin{align*}
    B(f^*, \widetilde{\mu}^*) + C(\mu, \widetilde{\mu}^*) &\leq \sum_{A \in S_K}  \int_{\mathcal{X}^K} \big( 1 + c_A(\vec{x}) \big) d\pi^*_A(\vec{x})\\
    &= \int_{\mathcal{Z}_*^K} \c(z_1, \dots, z_K) d\pi^*(z_1, \dots, z_K)\\
    &= B^*_{\mu}.
\end{align*}
On the other hand, the definition of $f^*_i$ implies that for any $x_i$ in the support of $\mu_i$ we have
\begin{equation}\label{eq: lowerbound_f^*}
    f^*_i(\widetilde{x}_i) \geq \sum_{j=1}^K \phi^*_j(x_i,i) + \sum_{j=1}^K \phi^*_j(\mathghost) - c(x_i, \widetilde{x}_i).
\end{equation}
Using $\sum_{A \in S_K(i)} \mu_{i,A} = \mu_i$ and \eqref{eq: lowerbound_f^*}, the optimality of $\phi^*$ implies that
\begin{align*}
    B(f^*, \widetilde{\mu}^*) + C(\mu, \widetilde{\mu}^*) &= \sum_{A \in S_K} \sum_{i \in A} \left\{  \int_{\mathcal{X} \times \mathcal{X}} \big( f^*_i(\widetilde{x}_i) + c( x_i, \widetilde{x}_i) \big) d\pi^*_i(x_i,\widetilde{x}_i) \right\}\\
    &\geq \sum_{A \in S_K} \sum_{i \in A} \left\{  \int_{\mathcal{X} \times \mathcal{X}} \big( \sum_{j=1}^K \phi^*_j(x_i,i) + \sum_{j=1}^K \phi^*_j(\mathghost)  \big) d\pi^*_i( x_i, \widetilde{x}_i) \right\}\\
    &= \sum_{A \in S_K} \sum_{i \in A} \left\{  \int_{\mathcal{X}} \big( \sum_{j=1}^K \phi^*_j(x_i,i) + \sum_{j=1}^K \phi^*_j(\mathghost)  \big) d\mu_{i,A}(x_i) \right\}\\
    &= \sum_{j=1}^K \int_{\mathcal{Z}} \phi_j(z_j) d\mu(z_j) + \mu(\mathcal{Z})\sum_{j=1}^K \phi_j(\mathghost)\\
    &= B^*_{\mu}.
\end{align*}
From the above we infer that
\begin{equation*}
    B(f^*, \widetilde{\mu}^*) + C(\mu, \widetilde{\mu}^*) = B^*_{\mu}.
\end{equation*}

Now we can prove \eqref{eq : saddle_point_classification_power}. The first inequality of \eqref{eq : saddle_point_classification_power} is straightforward, since the definition of $B_{\mu}^*$ in \eqref{eq:generalized_barycenter} and the optimality of $\widetilde{\mu}^*$ imply that
\begin{align*}
    B(f, \widetilde{\mu}^*) + C(\mu, \widetilde{\mu}^*) &\leq \sup_{f \in \mathcal{F}} \left\{ B(f, \widetilde{\mu}^*) + C(\mu, \widetilde{\mu}^*) \right\} = B_{\mu}^* = B(f^*, \widetilde{\mu}^*) + C(\mu, \widetilde{\mu}^*).
\end{align*}

For the second inequality of \eqref{eq : saddle_point_classification_power}, let arbitrary $\widetilde{\mu}$ be fixed and $\pi_i \in \Gamma( \mu_i, \widetilde{\mu}_i)$ be an optimal coupling for each $i \in [K]$. Then,
\begin{align*}
     B(f^*, \widetilde{\mu}) + C(\mu, \widetilde{\mu}) &= \sum_{i \in [K]} \int_{\mathcal{X}} f^*_i(\widetilde{x}) d\widetilde{\mu}_i(\widetilde{x}) + \sum_{i \in [K]} C( \mu_i,\widetilde{\mu}_i)\\
    &= \sum_{i \in [K]} \int_{\mathcal{X} \times \mathcal{X}} \left( f^*_i(\widetilde{x}) + c(x, \widetilde{x}) \right) d\pi_i(x, \widetilde{x}).
\end{align*}
Applying \eqref{eq: lowerbound_f^*} yields that
\begin{align*}
    B(f^*, \widetilde{\mu}) + C(\mu, \widetilde{\mu}) &\geq \sum_{i \in [K]} \int_{\mathcal{X} \times \mathcal{X}} \left( \sum_{j=1}^K \phi^*_j(x,i) + \sum_{j=1}^K \phi^*_j(\mathghost) \right) d\pi_i(x, \widetilde{x})\\
    &= \sum_{i \in [K]} \int_{\mathcal{X} \times \mathcal{X}} \left( \sum_{j=1}^K \phi^*_j(x,i) + \sum_{j=1}^K \phi^*_j(\mathghost) \right) d\mu_i(x)\\
    &= B_{\mu}^*\\
    &= B(f^*, \widetilde{\mu}^*) + C(\mu, \widetilde{\mu}^*).
\end{align*}

Therefore, $(f^*, \widetilde{\mu}^*)$ is a saddle point for \eqref{eqn:DualAdversarial}, hence for \eqref{def:RobustClassifPower} and \eqref{Robust problem:Intro} also. 
\end{proof}

\begin{remark}
Many recent papers have tried to analyze adversarial learning from a game-theoretic perspective \cite{bose2020adversarial, Meunier2021MixedNE, pydi2021the}. This approach is natural: the learner aims at maximizing the classification power $B^*_{\mu}$ while the adversary aims at maximizing the loss $R^*_{\mu}$(hence to minimize $B^*_{\mu}$): this is a standard zero-sum game. Our main results thus provide a way to build Nash equilibria for the adversarial problem using a series of equivalent formulations taking the form of generalized barycenter problems or MOTs.
\end{remark}

\begin{corollary}
\label{cor:Permutations}
Let $\pi^*$ be a solution of the MOT problem \eqref{Robust problem:Alternative_form} and let $F: \Z_*^K \rightarrow \Z_*^K$ be defined according to
\[ F(z_1, \dots, z_K) = (z_{\sigma(1)}, \dots, z_{\sigma(K)} ), \]
for $\sigma : [K] \rightarrow [K] $ a permutation. Then any convex combination of $F_{\sharp} \pi^* $ and $\pi^*$ is also a solution. 
\end{corollary}

\begin{proof}
This follows immediately from the fact that the cost function $\c$ is invariant under permutations and the fact that all marginals of $\pi^*$ are the same.
\end{proof}

\section{Examples and Numerical experiments}\label{sec:Examples}
Through this section, the cost $c$ is as in \textbf{Example} \ref{ex : infty_OT cost}. This cost has been widely used in adversarial learning literature and distributional robust optimization literature. Examples in this section illuminate how our general framework of generalized barycenter and MOT finds applications in practice.

\subsection{Recovery of the binary case}
Consider the binary case $K=2$. Our goal is to show that our results recover the result in \cite{Trillos2020AdversarialCN}. 

Let $z_1,z_2 \in \Z_*$. If both $z_1$ and $z_2$ are $\mathghost$, then $\c(z_1, z_2)=0$. If only one of them is $\mathghost$, then the cost is $\frac{1}{2}$. Finally, consider the case where $z_1, z_2 \not = \mathghost$. First assume that $i_1 = i_2=1$. In that case,
\begin{equation*}
    \widehat{\mu}_{\vec z} = \frac{1}{2} \delta_{(x_1,1)} + \frac{1}{2} \delta_{(x_2,1)}.
\end{equation*}
Since only class $1$ is represented in this configuration, there is no meaningful adversarial attack in this case, and thus,
\begin{equation*}
    B^*_{ \widehat{\mu}_{\vec z}} = 1.
\end{equation*}
Assume now that $i_1=1$ and $i_2=2$. In that case,
\begin{equation*}
    \widehat{\mu}_{\vec z} = \frac{1}{2}\widehat{\mu}_1 + \frac{1}{2} \widehat{\mu}_2 = \frac{1}{2}\delta_{(x_1,1)} + \frac{1}{2} \delta_{(x_2,2)},
\end{equation*}
and the adversary can attack meaningfully if and only if $d(x_1, x_2) \leq 2\veps$. Thus,
\begin{equation*}
    B^*_{\widehat{\mu}_{\vec z}} =\left\{ \begin{array}{ll}
\frac{1}{2} & \textrm{if $d(x_1, x_2) \leq 2\veps$,}\\
1 & \textrm{if $d(x_1, x_2) > 2\veps$.}
\end{array} \right. 
\end{equation*}
To summarize,
\begin{equation*}
     \c(z_1, z_2)=\left\{ \begin{array}{ll}
\frac{1}{2} & \textrm{if $i_1 \neq i_2$ and $d(x_1, x_2) \leq 2\veps$,}\\
1 & \textrm{if $i_1=i_2$ or $d(x_1, x_2) > 2\veps$,}\\
\frac{1}{2} & \textrm{if exactly one of $z_i$'s is $\mathghost$,}\\
0 & \textrm{if $z_1=z_2=\mathghost$.}
\end{array} \right. 
\end{equation*}

In \cite{Trillos2020AdversarialCN}, it is proved that
\begin{equation*}
    B^*_{\mu} = \inf_{\tilde \pi \in \Gamma(\mu, \mu)} \int_{\mathcal{Z} \times \mathcal{Z}}  \Big(\frac{\text{cost}_{\veps}(z_1, z_2) + 1}{2} \Big)  d \tilde\pi(z_1, z_2),
\end{equation*}
where 
\begin{equation*}
    \text{cost}_{\veps}(z_1, z_2) =\left\{ \begin{array}{ll}
0 & \textrm{if $i_1 \neq i_2$ and $d(x_1, x_2) \leq 2\veps$,}\\
1 & \textrm{if $i_1=i_2$ or $d(x_1, x_2) > 2\veps$.}\\
\end{array} \right. 
\end{equation*}
In other words, in the binary case, it is unnecessary to introduce the element $\mathghost$. To illustrate this point, assume for simplicity that $\mu(\Z)=1$. Notice that every $\tilde \pi \in \Gamma(\mu, \mu)$ induces a $ \pi \in \Pi_2(\mu)$ as follows:
\[ \int_{\Z_* \times \Z_*} \varphi(z_1, z_2) d \pi(z_1,z_2):=  \frac{1}{2}\int_{\Z \times \Z} \varphi(z_1, z_2) d\tilde \pi(z_1, z_2) + \frac{1}{2}\varphi(\mathghost,\mathghost),  \]
where $\varphi : \Z_* \times \Z_* \rightarrow \R $ is an arbitrary test function. The cost associated to the induced $\pi$ is:
\[ 2 \int_{\Z_* \times \Z_*} \c(z_1, z_2 )d\pi(z_1, z_2) =  \int_{\Z \times \Z} \c(z_1, z_2) d\tilde \pi(z_1, z_2)  =  \int_{\Z \times \Z}  \Big(\frac{\text{cost}_{\veps}(z_1, z_2) + 1}{2} \Big) d\tilde \pi(z_1, z_2). \]

On the other hand, let $\pi$ be a solution for the MOT problem \eqref{Robust problem:Alternative_form} (such a solution exists thanks to \textbf{Proposition} \ref{prop:MOTNoGap}). Thanks to \textbf{Corollary} \ref{cor:Permutations}, we can assume without loss of generality that
\[ \pi(A \times A') = \pi(A' \times A),  \]
for all $A,A'$ measurable subsets of $\Z_*$. We now define $\tilde{\pi}$ according to:
\begin{align*}
 \int_{\Z \times \Z} \tilde{\varphi}(z_1, z_2) d \tilde \pi(z_1, z_2) &:= 2\int_{\Z \times \Z} \tilde \varphi(z_1, z_2) d \pi(z_1, z_2)\\
 &\quad\,\,+  \int_{\Z \times \{ \mathghost\}} \tilde \varphi(z_1, z_1) d \pi(z_1, z_2) +  \int_{\{ \mathghost\} \times \Z   } \tilde \varphi(z_2, z_2) d \pi(z_1, z_2),   
 \end{align*}
for test functions $\tilde \varphi : \Z \times \Z \rightarrow \R$. It follows that $\tilde \pi \in \Gamma(\mu, \mu)$. Moreover, from the above formula and the expressions for the cost $\c$ we get
\[ \int_{\Z \times \Z} \Big(\frac{\text{cost}_{\veps}(z_1, z_2) + 1}{2} \Big) d\tilde \pi(z_1, z_2) =  \int_{\Z \times \Z} \c(z_1, z_2) d\tilde \pi(z_1, z_2) = 2 \int_{\Z_* \times \Z_*} \c(z_1, z_2) d\pi(z_1, z_2).  \]

The above computations show that our results indeed recover those from \cite{Trillos2020AdversarialCN} for the binary case.


\subsection{Toy example: three points distribution}\label{ex : toy_example}

Let's assume that $K=3$ and $\mu$ is such that 
\begin{equation*}
    \mu_1 = \omega_1 \delta_{x_1}, \,\,\, \mu_2 = \omega_2 \delta_{x_2}, \,\,\, \mu_3= \omega_3 \delta_{x_3},
\end{equation*}
for three points $x_1,x_2, x_3$ in Euclidean space. Without loss of generality, assume further that $\omega_1 \geq \omega_2 \geq \omega_3 > 0$ and $\sum \omega_i = 1$. Let $\veps > 0$ be given and consider the cost from \textbf{Example} \ref{ex : infty_OT cost} with $d$ as the Euclidean distance (for simplicity). We will explicitly construct an optimal robust classifier and an optimal adversarial attack for this problem. Even in this simple scenario, one can observe non-trivial situations.

Since for every $\widetilde{\mu}_i$ such that $W_{\infty}(\omega_i \delta_{x_i}, \widetilde{\mu}_i) \leq \veps$ we have
\begin{equation*}
    \int_{\mathcal{X}} f_i(x_i) d\widetilde{\mu}_i(x_i) = \int_{\overline{B}(x_i, \veps)} f_i(x_i)  d\widetilde{\mu}_i(x_i),
\end{equation*}
where $\overline{B}(x, r) = \{x' : d(x, x') \leq r\}$, we can assume without loss of generality that $\spt (\widetilde{\mu}_i) \subseteq \overline{B}(x_i, \veps)$. Notice that it is sufficient to consider $f \in \mathcal{F}$ restricted to $\overline{B}(x_1, \veps) \cup \overline{B}(x_2, \veps) \cup \overline{B}(x_3, \veps)$ (in fact, problem \eqref{Robust problem:Intro} can not disambiguate the values of $f$ outside of this set). We consider $4$ non-trivial configurations and one trivial one. Figure \ref{plot: toy_example} below illustrates how the adversary perturbs the original data distribution in each of the non-trivial cases. 

\textbf{Case 1.} $d(x_i, x_j) > 2\veps$ for all $1 \leq i \neq j \leq 3$. This is a trivial case. We claim that for any $\widetilde{\mu}_i$ such that $W_{\infty}(\omega_i\delta_{x_i}, \widetilde{\mu}_i) \leq \veps$, $((\mathds{1}_{\overline{B}(x_1, \veps)}, \mathds{1}_{\overline{B}(x_2, \veps)}, \mathds{1}_{\overline{B}(x_3, \veps)}), (\widetilde{\mu}_1, \widetilde{\mu}_2, \widetilde{\mu}_3) )$ is a saddle point for \eqref{Robust problem:Intro}. This is straightforward, since $ \spt (\widetilde{\mu}_i) \cap \spt (\widetilde{\mu}_j)  = \emptyset$, and thus it can be deduced that $(\mathds{1}_{\overline{B}(x_1, \veps)}, \mathds{1}_{\overline{B}(x_2, \veps)}, \mathds{1}_{\overline{B}(x_3, \veps)})$ is a dominant strategy for the learner. It is easy to check that $B^*_{\mu} = 1$ in this case.

\textbf{Case 2.} There is some $\overline{x}$ such that $d(\overline{x}, x_i) \leq \veps$ for all $1 \leq i \leq 3$. We claim that $( (1,0,0), (\omega_1 \delta_{\overline{x}}, \omega_2 \delta_{\overline{x}}, \omega_3 \delta_{\overline{x}}) )$ is a saddle point. First, $\omega_i\delta_{\overline{x}}$ is feasible for all $1 \leq i \leq 3$, since $\overline{x} \in \overline{B}(x_i, \veps)$ for all $i$. Now, given $(\omega_1 \delta_{\overline{x}}, \omega_2 \delta_{\overline{x}}, \omega_3 \delta_{\overline{x}})$, the best strategy for the learner is to choose class $1$ deterministically for all points, since $\omega_1 \geq \omega_2 \geq \omega_3$. On the other hand, given $(1,0,0)$, any adversarial attack yields the same classification power. From this we conclude that $( (1,0,0), (\omega_1 \delta_{\overline{x}}, \omega_2 \delta_{\overline{x}}, \omega_3 \delta_{\overline{x}}) )$ is indeed a saddle point. Notice that $B^*_{\mu}=\omega_1$ in this case.

\textbf{Case 3.} Two points are close to each other while the other point is far from them. For simplicity, we only consider the case $d(x_1, x_2) \leq 2\veps$, $d(x_1, x_3) > 2\veps$ and $d(x_2, x_3) > 2\veps$. The other cases are treated similarly. Let $\overline{x}_{12} =  \frac{x_1 + x_2}{2}$, and define $\widehat{f} = (\mathds{1}_{\overline{B}(x_1, \veps) \cup \overline{B}(x_2, \veps)}, 0, \mathds{1}_{\overline{B}(x_3, \veps)} )$ and $\widehat{\mu} = (\omega_1\delta_{\overline{x}_{12}}, \omega_2 \delta_{\overline{x}_{12}}, \widetilde{\mu}_3)$ for arbitrary $\widetilde{\mu}_3$ with $W_\infty(\widetilde{\mu}_3, \omega_3\delta_{x_3} ) \leq \veps$. We claim that $(\widehat{f}, \widehat{\mu})$ is a saddle point. For any $(f_1, f_2, f_3) \in \mathcal{F}$ we have
\begin{align*}
    B_{\mu}(f, \widehat{\mu})&= \int_{\mathcal{X}} f_1(x) \omega_1 \delta_{\overline{x}_{12}}(x) + \int_{\mathcal{X}} f_2(x) \omega_2 \delta_{\overline{x}_{12}}(x) + \int_{\mathcal{X}} f_3(x)  d\widetilde{\mu}_3(x)\\
    &= \omega_1 f_1(\overline{x}_{12}) + \omega_2 f_2(\overline{x}_{12}) + \int_{\mathcal{X}} f_3(x)  \widetilde{\mu}_3(x)\\
    &\leq \omega_1 + \omega_3\\
    &= \int_{\mathcal{X}} \mathds{1}_{\overline{B}(x_1, \veps) \cup \overline{B}(x_2, \veps)} \omega_1 \delta_{\overline{x}_{12}}(x) + \int_{\mathcal{X}} 0 \, \omega_2 \delta_{\overline{x}_{12}}(x) + \int_{\mathcal{X}} \mathds{1}_{\overline{B}(x_3, \veps)} d\widetilde{\mu}_3(x).
\end{align*}
On the other hand, given $(\mathds{1}_{\overline{B}(x_1, \veps) \cup \overline{B}(x_2, \veps)}, 0, \mathds{1}_{\overline{B}(x_3, \veps)} )$, for any $(\widetilde{\mu}_1, \widetilde{\mu}_2, \widetilde{\mu}_3)$,
\begin{align*}
    B_{\mu}(\widehat{f}, \widetilde{\mu}) &= \int_{\mathcal{X}} \mathds{1}_{\overline{B}(x_1, \veps) \cup \overline{B}(x_2, \veps)}  d\widetilde{\mu}_1(x) + \int_{\mathcal{X}} 0 \,  d\widetilde{\mu}_2(x) + \int_{\mathcal{X}} \mathds{1}_{\overline{B}(x_3, \veps)}  d\widetilde{\mu}_3(x)\\
    &= \omega_1 + \omega_3\\
    &= B_{\mu}(\widehat{f}, \widehat{\mu})
\end{align*}
where the second equality follows from the assumption on the configuration of points. The above computations imply the claim. In this case $B^*_{\mu} = \omega_1 + \omega_3$.

\textbf{Case 4.} $d(x_i, x_j) \leq 2\veps$ for any $x_i, x_j$ but $\overline{B}(x_1, \veps) \cap \overline{B}(x_2, \veps) \cap \overline{B}(x_3, \veps) = \emptyset$. Note that when $K=2$, $d(x_1, x_2) \leq 2\veps$ if and only if $\overline{B}(x_1, \veps) \cap \overline{B}(x_2, \veps) \neq \emptyset$. However, when $K \geq 3$, these cases are not equivalent anymore. There are two subcases to consider depending on the magnitudes of the weights $(\omega_1, \omega_2, \omega_3)$.

\textbf{Case 4 - (i)} $\omega_1 < \omega_2 + \omega_3$. In this case, we can find some $\alpha_i \in [0, \omega_i]$ for all $1 \leq i \leq 3$ such that
\begin{equation*}
    \alpha_1 = \omega_2 - \alpha_2, \,\,\, \alpha_2 = \omega_3 - \alpha_3 \text{ and } \alpha_3 = \omega_1 - \alpha_1.
\end{equation*}
Precisely,
\begin{equation*}
    \alpha_1  = \frac{\omega_1 + \omega_2 - \omega_3}{2}, \,\,\,  \alpha_2 = \frac{\omega_2 + \omega_3 - \omega_1}{2}, \text{ and } \alpha_3 = \frac{\omega_3 + \omega_1 - \omega_2}{2}.
\end{equation*}
Note that for all $i$, $\alpha_i \geq 0$ since $\omega_1 \leq \omega_2 + \omega_3$. Let $\overline{x}_{12} \in \overline{B}(x_1, \veps) \cap \overline{B}(x_2, \veps)$, $\overline{x}_{13} \in \overline{B}(x_1, \veps) \cap \overline{B}(x_3, \veps)$ and $\overline{x}_{23} \in \overline{B}(x_2, \veps) \cap \overline{B}(x_3, \veps)$. Construct the following measures
\begin{align*}
    \widehat{\mu}_1 &:=  \big(\alpha_1 \delta_{\overline{x}_{12}} + (\omega_1 - \alpha_1) \delta_{\overline{x}_{13}} \big) = \big( (\frac{\omega_1 + \omega_2 - \omega_3}{2}) \delta_{\overline{x}_{12}} + (\frac{\omega_1 - \omega_2 + \omega_3}{2}) \delta_{\overline{x}_{13}} \big),\\
    \widehat{\mu}_2 &:= \big(\alpha_2 \delta_{\overline{x}_{23}} + (\omega_2 - \alpha_2) \delta_{\overline{x}_{12}} \big) =  \big( (\frac{\omega_2 + \omega_3 - \omega_1}{2}) \delta_{\overline{x}_{23}} + (\frac{\omega_2 - \omega_3 + \omega_1 }{2}) \delta_{\overline{x}_{12}} \big)  ,\\
    \widehat{\mu}_3 &:= \big(\alpha_3 \delta_{\overline{x}_{13}} + (\omega_3 - \alpha_3) \delta_{\overline{x}_{23}} \big) = \big( (\frac{\omega_3 + \omega_1 - \omega_2}{2}) \delta_{\overline{x}_{13}} + (\frac{\omega_3 - \omega_1 + \omega_2}{2}) \delta_{\overline{x}_{23}} \big).
\end{align*}
Observe that at each $\overline{x}_{ij}$, $\widehat{\mu}_i$ and $\widehat{\mu}_j$ put the same mass: it is natural since, otherwise, the learner will choose a class which puts more mass at $\overline{x}_{ij}$. So, this gives a hint about what would be the best adversarial attack. The adversary gathers classes as much as possible and distributes them as uniform as possible.

Let $A_{ij}= A_{ji} := \overline{B}(x_i, \veps) \cap \overline{B}(x_j, \veps)$ and $A_i = \overline{B}(x_i, \veps) \setminus (A_{ij} \cup A_{ik})$. One can observe that since $d(x_i, x_j) \leq 2\veps$ for any $x_i, x_j$ but $\overline{B}(x_1, \veps) \cap \overline{B}(x_2, \veps) \cap \overline{B}(x_3, \veps) = \emptyset$, $\overline{B}(x_i, \veps) = A_{ij} \dot{\cup} A_{ik} \dot{\cup} A_i$ for each $i$. Here $\dot{\cup}$ denotes a disjoint union. Also, since $W_{\infty}(\widetilde{\mu}_i, \omega_i \delta_{x_i}) \leq \veps$, it must be the case that $A_{ij} \cap \spt (\widetilde{\mu}_k) = \emptyset$ if $k \neq i,j$. For each $1 \leq i \leq 3$, construct the following weak partition: 
\begin{equation*}
     \widehat{f}_i(x):=\left\{ \begin{array}{ll}
1 & \textrm{if $x\in A_i$,}\\
\frac{1}{2} & \textrm{if $x\in A_{ij}$,}\\
0 & \textrm{if $x \notin \overline{B}(x_i, \veps)$.}
\end{array} \right. 
\end{equation*}
$\widehat{f}$ is a weak partition since $\overline{B}(x_i, \veps) = A_{ij} \dot{\cup} A_{ik} \dot{\cup} A_i$ and $\overline{B}(x_1, \veps) \cap \overline{B}(x_2, \veps) \cap \overline{B}(x_3, \veps) = \emptyset$. We claim that $(\widehat{f}, \widehat{\mu})$ is a saddle point. Note that a straightforward computation yields $B_{\mu}(\widehat{f}, \widehat{\mu}) = \frac{1}{2}$.

Given $(\widehat{\mu}_1, \widehat{\mu}_2, \widehat{\mu}_3)$, for any $(f_1, f_2, f_3) \in \mathcal{F}$,
\begin{align*}
    B_{\mu}(f, \widehat{\mu}) &= \int_{\mathcal{X}} f_1(x)  d\widehat{\mu}_1(x) + \int_{\mathcal{X}} f_2(x)  d\widehat{\mu}_2(x) + \int_{\mathcal{X}} f_3(x)  d\widehat{\mu}_3(x)\\
    &= (\frac{\omega_1 + \omega_2 - \omega_3}{2})f_1(\overline{x}_{12}) + (\frac{\omega_1 + \omega_3 - \omega_2}{2})f_1(\overline{x}_{13}) + (\frac{\omega_2 + \omega_3 -\omega_1 }{2}) f_2(\overline{x}_{23})\\
    &\quad + (\frac{\omega_1 + \omega_2 - \omega_3}{2}) f_2(\overline{x}_{12}) + (\frac{\omega_1 + \omega_3 - \omega_2}{2}) f_3(\overline{x}_{13}) + (\frac{\omega_2 + \omega_3 -\omega_1}{2}) f_3(\overline{x}_{23})\\
    &= (\frac{\omega_1 + \omega_2 - \omega_3}{2}) (f_1(\overline{x}_{12}) + f_2(\overline{x}_{12})) + (\frac{\omega_1 + \omega_3 - \omega_2}{2}) (f_1(\overline{x}_{13}) + f_3(\overline{x}_{13}))\\
    &\quad + (\frac{\omega_2 + \omega_3 -\omega_1 }{2}) (f_2(\overline{x}_{23}) + f_3(\overline{x}_{23}))\\
    &\leq (\frac{\omega_1 + \omega_2 - \omega_3}{2}) + (\frac{\omega_1 + \omega_3 - \omega_2}{2}) + (\frac{\omega_2 + \omega_3 -\omega_1 }{2})\\
    & = \frac{1}{2},
\end{align*}
where the second to last inequality follows from the fact that $\sum f_i(x) \leq 1$ and the last equality follows from the fact that $\sum \omega_i =1$. Given $(\widehat{f}_1, \widehat{f}_2, \widehat{f}_3)$, on the other hand, for any $(\widetilde{\mu}_1, \widetilde{\mu}_2, \widetilde{\mu}_3)$
\begin{align*}
    B_{\mu}(\widehat{f}, \widetilde{\mu}) &= \int_{\mathcal{X}} \widehat{f}_1(x) d\widetilde{\mu}_1(x) + \int_{\mathcal{X}} \widehat{f}_2(x)  d\widetilde{\mu}_2(x) + \int_{\mathcal{X}}\widehat{f}_3(x)  d\widetilde{\mu}_3(x)\\
    &= \frac{ \widetilde{\mu}_1 (A_{12}) +  \widetilde{\mu}_2 (A_{12})}{2} + \frac{ \widetilde{\mu}_1 (A_{13}) +  \widetilde{\mu}_3 (A_{13})}{2} + \frac{ \widetilde{\mu}_2 (A_{23}) +  \widetilde{\mu}_3 (A_{23})}{2}\\
    &\quad+  \widetilde{\mu}_1 (A_1 ) +   \widetilde{\mu}_2 (A_2 ) +  \widetilde{\mu}_3 (A_3).
\end{align*}
Note that since $W_{\infty}(\widetilde{\mu}_i, \omega_i \delta_{x_i}) \leq \veps$, $\spt (\widetilde{\mu}_i) \cap A_j = \emptyset$ for any $\widetilde{\mu}_i$ and for any $i \neq j$. To minimize the above, the adversary should put $\spt (\widetilde{\mu}_i) \subseteq A_{ij} \cup A_{ik}$ for all $i$. Also, at the minimum, it must be the case that $ \widetilde{\mu}_i(A_{ij}) =  \widetilde{\mu}_j(A_{ij})$, otherwise the adversary would be able decrease the classification power further. Combining all arguments, we can deduce
\begin{equation*}
    B_{\mu}(\widetilde{f}, \widetilde{\mu}) \geq \frac{ \widetilde{\mu}_1 (A_{12}) + \widetilde{\mu}_2 (A_{12})}{2} + \frac{\widetilde{\mu}_1 (A_{13}) + \widetilde{\mu}_3 (A_{13})}{2} + \frac{\widetilde{\mu}_2 (A_{23}) + \widetilde{\mu}_3 (A_{23})}{2}= \frac{1}{2},
\end{equation*}
which verifies the claim. In this case, $B^*_{\mu} = \frac{1}{2}$.

In fact, it is unavoidable to introduce weak partitions $f \in \mathcal{F}$. Let $f=(\mathds{1}_{F_1}, \mathds{1}_{F_2}, \mathds{1}_{F_3})$ be any strong partition, i.e. $F_1 \dot{\cup} F_2 \dot{\cup} F_3 = \cup \overline{B}(x_i, \veps)$. We will show that for any $\widetilde{\mu}$, $(f, \widetilde{\mu})$ cannot be a saddle point. Assume that $\overline{B}(x_1, \veps) \subseteq F_1$. Since $d(x_1, x_2)\leq 2\veps$ and $d(x_1, x_3) \leq 2\veps$, it must be the case that $F_1 \cap \overline{B}(x_2, \veps) \neq \emptyset$ and $F_1 \cap \overline{B}(x_3, \veps) \neq \emptyset$. These facts yield that optimal $\widetilde{\mu}_2$ and $\widetilde{\mu}_3$ for the adversary must satisfy $\spt (\widetilde{\mu}_2) \subseteq F_1 \cap \overline{B}(x_2, \veps)$ and $\spt (\widetilde{\mu}_3) \subseteq F_1 \cap \overline{B}(x_3, \veps)$. This configuration gives a classifying power $\omega_1$ since the learner can only detect class $1$ perfectly and always misclassifies others.

However, given any such $(\widetilde{\mu}_1, \widetilde{\mu}_2, \widetilde{\mu}_3)$, the learner has an incentive to modify a classifying rule. Let $F'_1 := F_1 \setminus (\spt (\widetilde{\mu}_2) \cup \spt (\widetilde{\mu}_3))$, $F'_2 := F_2 \cup \spt (\widetilde{\mu}_2)$ and $F'_3 := F_3 \cup \spt( \widetilde{\mu}_3)$. Then, this classifying rule perfectly classifies. Precisely, there exists a deviation for the learner, $f' = (\mathds{1}_{F'_1}, \mathds{1}_{F'_2}, \mathds{1}_{F'_3})$, such that
\begin{equation*}
    1 = B(f', \widetilde{\mu}) > B(f, \widetilde{\mu}) = \omega_1.
\end{equation*}

Assume that $\overline{B}(x_1, \veps) \not \subseteq F_1$. Since $(F_1, F_2, F_3)$ is a partition, it must be the case that either $F_2 \cap \overline{B}(x_1, \veps) \neq \emptyset$ or $F_3 \cap \overline{B}(x_1, \veps) \neq \emptyset$. Without loss of generality, assume the former case only. The other cases are analogous. $F_2 \cap \overline{B}(x_1, \veps) \neq \emptyset$ yields that an optimal $\widetilde{\mu}_1$ for the adversary must satisfy $\spt (\widetilde{\mu}_1) \subseteq F_2$. Then, a corresponding classifying power is at most $\omega_2 + \omega_3$ since the learner always misclassifies class $1$.

However, given any such $(\widetilde{\mu}_1, \widetilde{\mu}_2, \widetilde{\mu}_3)$, the learner has an incentive to modify a classifying rule again. Let $F'_1 := F_1 \cup \spt (\widetilde{\mu}_1)$, $F'_2 := F_2 \setminus \spt (\widetilde{\mu}_1)$ and $F'_3 := F_3$. Similar as above, letting $f' = (\mathds{1}_{F'_1}, \mathds{1}_{F'_2}, \mathds{1}_{F'_3})$, such that
\begin{equation*}
    1 = B(f', \widetilde{\mu}) > \omega_2 + \omega_3 \geq B(f, \widetilde{\mu}).
\end{equation*}
Therefore, any strong partition $f=(\mathds{1}_{F_1}, \mathds{1}_{F_2}, \mathds{1}_{F_3})$ cannot sustain a saddle point in this case. 

We want to emphasize that the same reasoning still holds for other cases. In other words, even this simple discrete measures, it is necessary to extend strong partition to weak partition in order to achieve the minimax value.

\textbf{Case 4 - (ii)} $\omega_1 \geq \omega_2 + \omega_3$. In this case, no matter how the adversary perturbs the distribution, there will always be an excess mass from class $1$ that won't be matched to other classes. Motivated by this observation, let $\kappa = \omega_1 - (\omega_2 + \omega_3) \geq 0$ and consider the following measures $(\widehat{\mu}_1, \widehat{\mu}_2, \widehat{\mu}_3)$:
\begin{align*}
     \widehat{\mu}_1 &= \omega_2 \delta_{\overline{x}_{12}} + \omega_3 \delta_{\overline{x}_{13}} + \kappa \delta_{x_1},\\
     \widehat{\mu}_2 &= \omega_2 \delta_{\overline{x}_{12}},\\
     \widehat{\mu}_3 &= \omega_3 \delta_{\overline{x}_{13}}.
\end{align*}
Consider $(\widehat{f}_1,\widehat{f}_2, \widehat{f}_3)=(1,0,0)$. We claim that $(\widehat{f}, \widehat{\mu}) = ( (\widehat{f}_1, \widehat{f}_2, \widehat{f}_3), (\widehat{\mu}_1, \widehat{\mu}_2, \widehat{\mu}_3) )$ is a saddle point. Note that a straightforward computation yields $B_{\mu}(\widehat{f}, \widehat{\mu}) = \omega_1$.

For any $(f_1, f_2, f_3) \in \mathcal{F}$,
\begin{align*}
    B_{\mu}(f, \widehat{\mu})&= \int_{\mathcal{X}} f_1(x)  d\widehat{\mu}_1(x) + \int_{\mathcal{X}} f_2(x)  d\widehat{\mu}_2(x) + \int_{\mathcal{X}} f_3(x)  d\widehat{\mu}_3(x)\\
    &= \omega_2 f_1(\overline{x}_{12}) + \omega_3 f_1(\overline{x}_{13}) + \kappa f_1(x_1) + \omega_2 f_2(\overline{x}_{12}) + \omega_3 f_3(\overline{x}_{13})\\
    &= \omega_2 (f_1(\overline{x}_{12}) + f_2(\overline{x}_{12})) + \omega_3 (f_1(\overline{x}_{13}) + f_3(\overline{x}_{13})) + \kappa f_1(x_1)\\
    & \leq \omega_2 + \omega_3 + \kappa\\
    & = \omega_1.
\end{align*}
On the other hand, for any feasible $(\widetilde{\mu}_1, \widetilde{\mu}_2, \widetilde{\mu}_3)$,
\begin{equation*}
    B_{\mu}(\widehat{f}, \widetilde{\mu}) = \int_{\mathcal{X}} \widehat{f}_1(x) d\widetilde{\mu}_1(x) + \int_{\mathcal{X}} \widehat{f}_2(x)  d\widetilde{\mu}_2(x) \int_{\mathcal{X}} \widehat{f}_3(x)  d\widetilde{\mu}_3(x) = \omega_1.
\end{equation*}
The claim follows. In this case, $B^*_{\mu} = \omega_1$. Here, $\omega_1 \geq \frac{1}{2}$, since $\omega_1 \geq \omega_2 + \omega_3$ and $\sum \omega_i=1$. In the case that $\omega_1=\omega_2+\omega_3$, both \textbf{Case 4 -(i)} and \textbf{Case 4 -(ii)} provide $B^*_{\mu}=\frac{1}{2}$, which shows the consistency. 

We now show that the adversary has no incentive to use the point $\overline{x}_{23}$, in contrast to what  happens in \textbf{Case 4 -(i)}. Fix a small $\eta > 0$, and suppose that the adversary moves $\eta$ mass from each of $\omega_2 \delta_{x_2}$ and $\omega_3 \delta_{x_3}$ to the point $\overline{x}_{23}$, respectively. Construct corresponding measures:
\begin{align*}
     \widetilde{\mu}_1 &= (\omega_2 - \eta)  \delta_{\overline{x}_{12}} + (\omega_3 - \eta) \delta_{\overline{x}_{13}} + \kappa' \delta_{x_1},\\
     \widetilde{\mu}_2 &= \eta \delta_{\overline{x}_{23}} + (\omega_2 - \eta)  \delta_{\overline{x}_{12}},\\
     \widetilde{\mu}_3 &=  (\omega_3 - \eta) \delta_{\overline{x}_{13}} + \eta \delta_{\overline{x}_{23}}
\end{align*}
where $\kappa' = \omega_1 - (\omega_2 + \omega_3 - 2\eta) = \kappa + 2\eta$. We show that $\widetilde{\mu}$ can not be a solution to the adversarial problem by showing that the learner can select a strategy $\widetilde{f}$ for which
\[ B_\mu(\widetilde{f}, \widetilde{\mu}) > \omega_1.\]
Indeed, we can select $\widetilde{f} := (\mathds{1}_{\overline{B}(x_1,\veps)}, 0, \mathds{1}_{\X \setminus {\overline{B}(x_1,\veps)}  } )$. It follows that
\begin{align*}
   B_\mu(\widetilde{f} , \widetilde{\mu}) & =  \int_{\mathcal{X}} \widetilde{f}_1(x)  d\widetilde{\mu}_1(x) + \int_{\mathcal{X}} \widetilde{f}_2(x)  d\widetilde{\mu}_2(x) + \int_{\mathcal{X}} \widetilde{f}_3(x)  d\widetilde{\mu}_3(x) 
   \\& = (\omega_2 - \eta) + (\omega_3 -\eta) +\kappa'  + \eta = \omega_1 +\ \eta >\omega_1.
\end{align*}
Notice that while the geometry of points $x_1, x_2, x_3$ in \textbf{case 4 -(i)} and \textbf{case 4 -(ii)} is the same, the geometries of the corresponding optimal adversarial attacks are determined by the full distribution $\mu$ and not just by the geometry of its support. In fact, the optimal adversarial attacks $\widetilde{\mu}$ and the optimal barycenter measure $\lambda$ depend on not only the geometry of the support of $\mu$ but also the magnitudes of its marginals, $\mu_i$'s.

\begin{figure}[h]
	\includegraphics[scale=0.35]{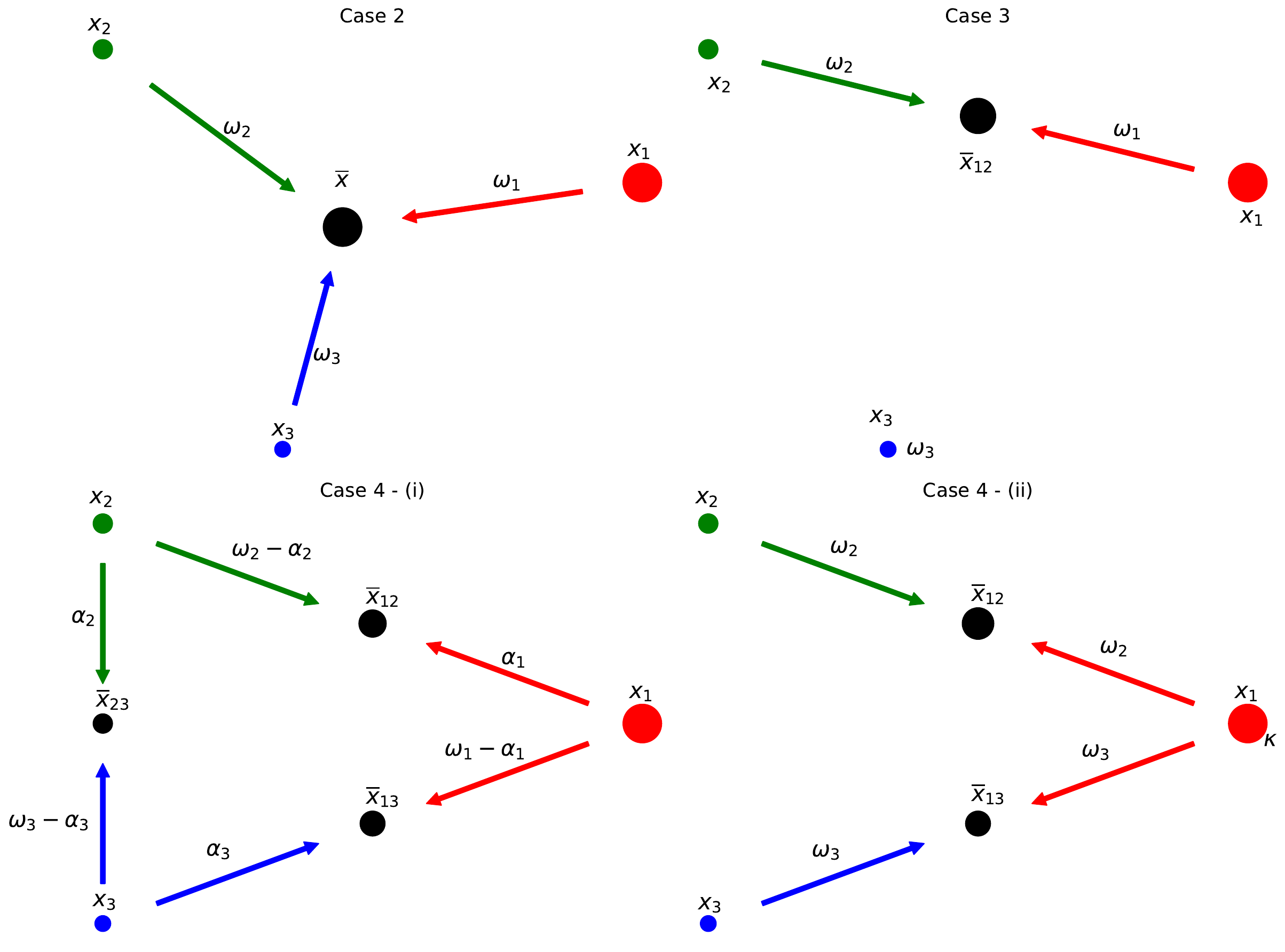}
	\caption{Illustrations of the adversarial attacks in all cases from section \ref{ex : toy_example}. Weights on arrows indicate the amount of mass the adversary moves to a perturbed point. $\overline{x}$'s are the support of $\lambda$ in \eqref{eq:generalized_barycenter}. One observes that the support of $\lambda$ depends on both the geometry of data distributions and their magnitudes.}
	\label{plot: toy_example}
\end{figure}

\subsection{Numerical Experiments}\label{subsec : numerical expeirment}

\begin{figure}[h]
	\includegraphics[scale=0.35]{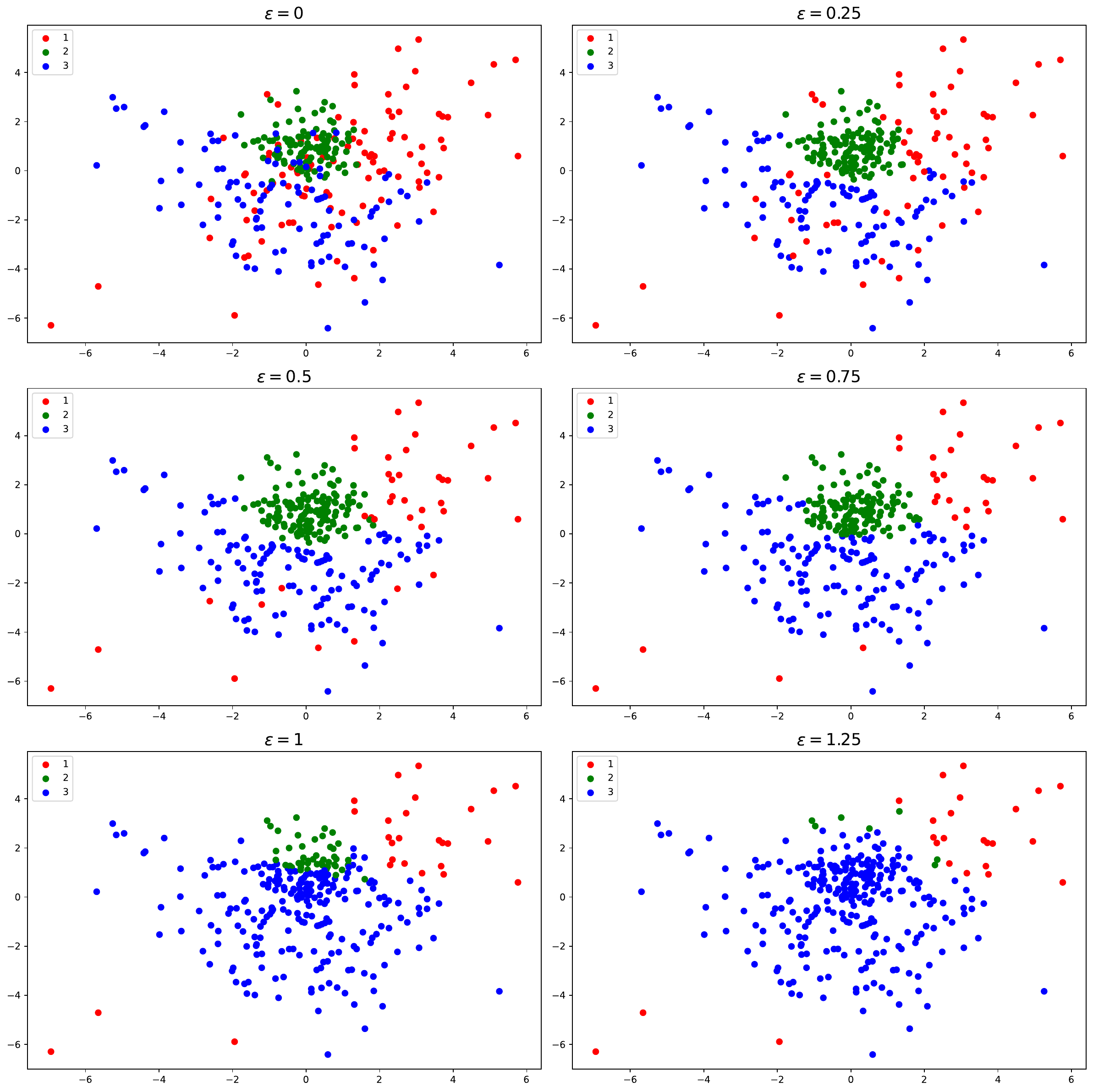}
	\caption{Three Gaussians in $\mathbb{R}^2$. One can observe that as $\veps$ grows the robust classifying rule becomes simpler, as expected.}
	\label{plot: Gaussian_example}
\end{figure}

\begin{figure}[h]
	\includegraphics[scale=0.3]{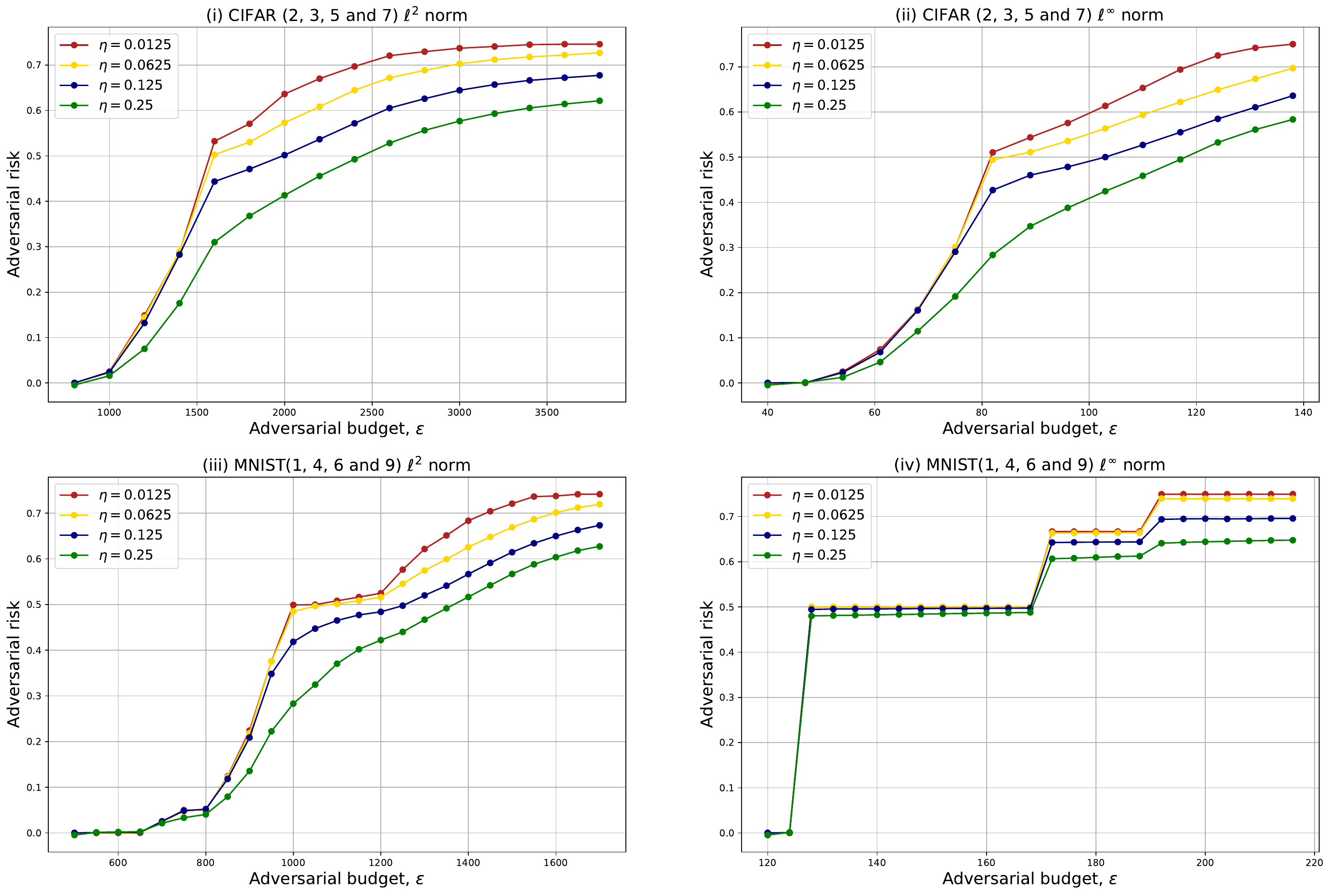}
	\caption{Adversarial risks \eqref{Robust problem:Intro} computed using the multimarginal Sinkhorn algorithm. $\eta$ is the entropic regularization parameter of the Sinkhorn algorithm. The maximum adversarial risk in all cases is $0.75$ because we consider $4$ classes and an equal number of points in each class. Due to the entropic penalty, the multimarginal Sinkhorn algorithm always gives an upper bound for the optimal classification power $B_\mu^*$, hence gives a lower bound for the adversarial risk $R^*_\mu$.}
	\label{plot: cifar_mnist}
\end{figure}

In this section we illustrate our theoretical results numerically. We obtain robust classifiers for synthetic data sets and compute optimal adversarial risks for two popular real data sets: MNIST and CIFAR. 

From the perspective of numeric, our aim is to solve the  MOT problem \eqref{Robust problem:Alternative_form} and its dual for an empirical measure $\mu$ whose support consists of $n$ data points. We use Sinkhorn algorithm for concreteness. Introduced in \cite{cuturi2013sinkhorn}, Sinkhorn algorithm has been one of the central algorithmic tools in computational optimal transport in the past decade. This algorithm, originally introduced in the context of standard ($2$-marginal) optimal transport problems, was extended to MOTs in \cite{benamou2015iterative, benamou2019generalized}. Works that study the computational complexity of generic MOT problems include: \cite{lin2022complexity, tupitsa2020multimarginal, haasler2021multimarginal, carlier2021linear}. In particular, \cite{lin2022complexity} and \cite{tupitsa2020multimarginal} prove the complexity of MOT Sinkhorn algorithm to be $\widetilde{O}(K^3 n^K \epsilon^{ -2})$ and $\widetilde{O}(K^3 n^{K+1} \epsilon^{ -1})$, respectively, where $\epsilon$ is the error tolerance. 

In our first illustration, we consider a data set $(x_1, y_1), \dots(x_n, y_n)$ in $\R^2 \times \{1,2,3\} $ obtained by sampling $y_i$ uniformly from $\{1,2,3 \}$ and then $x_i$ from a certain Gaussian distribution with parameters depending on the outcome of $y_i$. We consider the cost $c=c_\veps$ from \textbf{Example} \ref{ex : infty_OT cost} with $d$ the Euclidean distance in $\R^2$ and different values of $\veps$. In Figure \ref{plot: Gaussian_example} we show the labels assigned to the data by the (approximate) robust classifier, which we computed using \textbf{Corollary} \ref{cor:OptimalPair} for the dual potentials $\phi_j$ generated by the MOT Sinkhorn algorithm.

In our second illustration, we use the mutlimarginal version of Sinkhorn algorithm to compute the adversarial risk $R^*_\mu$ (i.e. the optimal value of \eqref{Robust problem:Intro}) for $\mu$ an empirical measure supported on a subset of either the CIFAR or MNIST data sets. In both cases we consider samples belonging to one of four possible classes in order to decrease the computational complexity of the problem. We use the cost $c$ from \textbf{Example} \ref{ex : infty_OT cost} for different values of $\veps$ and two choices of $d$: the Euclidean distance $\ell^2$ and the $\ell^\infty$ distance. The results are shown in Figure \ref{plot: cifar_mnist}. We can observe that for the CIFAR data set the two distance functions behave similarly: while not the same, the plots exhibit a similar qualitative behavior. For the MNIST data set, on the other hand, the situation is markedly different: in contrast to the plot for the $\ell^2$ distance, the adversarial risk with $\ell^{\infty}$ distance varies dramatically as $\veps$ grows. This observation is consistent with the findings in \cite{Pydi2021AdversarialRV} for the binary case. 

We emphasize that Figure \ref{plot: cifar_mnist} only provides approximations of the true adversarial risk $R^*_\mu$. Indeed, recall that $R^*_{\mu} = 1 - B^*_{\mu}$. Approximating $B^*_{\mu}$ using the MOT Sinkhorn algorithm will always produce an upper bound for $B^*_{\mu}$ since the regularization term effectively restricts the solution space of \eqref{Robust problem:Alternative_form}. Thus, the multimarginal Sinkhorn algorithm always yields a lower bound for the true $R^*_{\mu}$. Of course, one can always compute a tighter lower bound by reducing the regularization parameter $\eta$ at the expense of increasing the computational burden.

As way of conclusion for this section we provide pointers to the literature discussing the computational complexity of the Wasserstein barycenter problem; Wasserstein barycenter problems are specific instances in the MOT family. On the one hand, \cite{altschuler2022wasserstein} prove certain computational hardness of the barycenter problem in the dimension of the feature space (here $\X$). On the other hand, \cite{altschuler2021wasserstein} present an algorithm that can get an approximate solution of the optimal barycenter in polynomial time for a fixed dimension of the feature space. While our MOT is not the standard barycenter problem, it is still a generalized version thereof, and thus, it is reasonable to expect that the structure of our problem can be used in the design of algorithms that perform better than off-the-shelf MOT solvers. We leave this task for future work.

\section{Conclusions and future directions}
\label{sec:Conclusions}

In this paper we have discussed a series of equivalent formulations of adversarial problems in the context of multiclass classification. These formulations take the form of problems in optimal transport, specifically, multimarginal optimal transport and (generalized) Wasserstein barycenter problem. Besides providing a novel connection between apparently unrelated fields, we have also discussed a series of theoretical and computational implications emanating from these equivalences. In what follows we briefly expand this discussion, while at the same time provide a few perspectives on future work.

First, it is of interest to design scalable algorithms for solving the MOT problem \eqref{Robust problem:Alternative_form}. In general, MOT problems are not scalable in the number of marginals of the problem. However, this may not necessarily be an issue for our MOT problem, since it possesses a special structure that, as we discussed throughout section \ref{sec:GenBar}, allows us to interpret the desired MOT problem as a generalized barycenter problem; barycenter problems, at least in their standard version, are known to scale much better than general MOT problems. Tailor-specific algorithms for our MOT problem can take advantage of the favourable geometric structure of a given data set. Indeed, if a data set is such that there is only a small number of classes (much smaller than $K$) that interact with each other at the scale implicitly specified by the cost $c$ (think of \textbf{Example} \ref{ex : infty_OT cost}), then the effective size of problem \eqref{Robust problem:Alternative_form} will be considerably smaller than the size of the worst case setting —see the reformulation \eqref{eq:multimarginal_decomposed}.

Second, it would be of interest to use \eqref{Robust problem:Intro} to help in the training of robust classifiers within specific families of models. Notice that \eqref{Robust problem:Intro} is model free from the perspective of the learner, but in applications practitioners may be interested in solving a problem like:
\begin{align*}
  \inf_{f \in \mathcal{G}} \sup_{\widetilde{\mu} \in \mathcal{P}(\Z) }  \left\{ R(f, \widetilde{\mu}) - C(\mu, \widetilde{\mu}) \right\},
\end{align*}
which differs from \eqref{Robust problem:Intro} in the family of classifiers $\mathcal{G}$, which may be strictly smaller than $\mathcal{F}$; for example, $\mathcal{G}$ could be a family of neural networks, kernel-based classifiers, or other popular (parametric) models. There are two ways in which problem \eqref{Robust problem:Intro} is still meaningful for the above model-specific problem: 1) the optimal $\widetilde{\mu}^*$ computed from the problem \eqref{Robust problem:Intro} can be used as a way to generate adversarial examples that could be used during training of the desired model; 2) the optimal value of \eqref{Robust problem:Intro} can serve as a benchmark for robust training within \textit{any} family of models.

Third, more theoretical understanding of optimal dual potentials and robust classifiers is required. As stated in \textbf{Corollary} \ref{cor:OptimalPair}, an optimal robust classifier can be obtained from solutions to \eqref{eqn:DualMOT}, or, equivalently, from \eqref{eq:mot_decomposed_dual} and \eqref{eq:mot_dual}. However, unless $c$ is continuous, even the existence of Borel measurable dual potentials is not guaranteed, and hence neither is the existence of optimal robust classifiers. At this stage, it is thus necessary to \textit{assume} that the classifier intorduced in \textbf{Corollary} \ref{cor:OptimalPair} is Borel measurable. This measurability issue, i.e., that a robust classifier may not be Borel measurable, has been discussed not only in the adversarial training community \cite{pydi2021the, awasthi2021existence, awasthi2021extended, frank2022existence, frank2022consistency}, but also more generally in the \textit{distributional robust optimization} community, e.g., \cite{MR3959085}. In general, at this point only the existence of \textit{universally} measurable robust classifier $f^*$ can be guaranteed. Whether there exist Borel measurable robust classifiers for discontinuous costs (like the one in Example \eqref{ex : infty_OT cost}) is a question that we hope to explore in future work.

Finally, it is of interest to investigate the geometric content that profiles like the ones presented in Figure \ref{plot: cifar_mnist} carry about a specific data set. As illustrated in Figure \ref{plot: cifar_mnist}, these curves are specific signatures (adversarial signatures) of a given data distribution, and thus, they may be potentially used to capture similarities and discrepancies between different data sets.

The above are just but a few directions currently under investigation that emanate from this work.

\vskip 0.2in
\bibliography{references.bib}
\bibliographystyle{amsalpha}
\end{document}